\documentclass{svjour3}
\usepackage[utf8]{inputenc} 
\usepackage[T1]{fontenc}    
\usepackage{hyperref}       
\usepackage{url}            
\usepackage{booktabs}       
\usepackage{amsfonts}       
\usepackage{nicefrac}       
\usepackage{cite}
\usepackage{microtype}    

\usepackage{setspace}
\usepackage{tcolorbox}
\usepackage{amssymb}
\usepackage[noend]{algorithmic}
\usepackage{amsmath}
\usepackage{graphicx}
\usepackage{multirow}
\usepackage{color}
\usepackage{float}
\usepackage{subfigure}
\usepackage{algorithm}
\usepackage{algorithmic}
\usepackage{pifont}

\usepackage{amsfonts}

\usepackage{amsthm}
\usepackage{xparse}
\NewDocumentCommand{\var}{O{s} m O{}}{%
  \ensuremath{#1_{#2}^{#3}}
}

\usepackage{amssymb,bm,amsmath,thmtools,thm-restate}

\newcommand{\featureSpace}{\mathcal{X}}

\newcommand{\trainingSample}{\mathcal{S}}

\newcommand{\entropyProbMeasure}{H_{\probMeasure}}
\newcommand{\entropySample}{H_{\trainingSample}}
\newcommand{\mutualInformationProbMeasure}{I_{\probMeasure}}
\newcommand{\mutualInformationSample}{I_{\trainingSample}}
\newcommand{\X}{\mathcal{X}}
\newcommand{\Y}{\mathcal{Y}}
\newcommand{\sign}{\text{sign}}
\newcommand{\R}{\mathbb{R}}

\newcommand{\riskMinimiser}{f^*}
\newcommand{\riskMinimiserProb}{f^*_{\probMeasure}}
\newcommand{\empiricalRiskMinimiser}{\riskMinimiser_{\trainingSample}}
\newcommand{\entropyPhi}{\phi}

\title{Margin Maximization as Lossless Maximal Compression}

\author{Nikolaos Nikolaou \and Henry Reeve \and\\ Gavin Brown}

\institute{N. Nikolaou \at
              Department of Physics \& Astronomy, UCL\\
              Gower Street, London, WC1E 6BT, UK\\
              \email{n.nikolaou@ucl.ac.uk}           
           \and
           H. Reeve \at
              School of Computer Science, University of Birmingham\\
              Edgbaston, Birmingham, B15 2TT, UK\\
               \email{h.w.j.reeve@bham.ac.uk}
           \and
           G. Brown \at
              School of Computer Science, University of Manchester\\
              Kilburn Building, Oxford Road, Manchester, M13 9PL, UK\\
              \email{gavin.brown@manchester.ac.uk}
}

\date{Received: date / Accepted: date}

\begin{document}

\maketitle

\begin{abstract}
  The ultimate goal of a supervised learning algorithm is to produce models constructed on the training data that can \emph{generalize} well to new examples. In classification, \emph{functional margin maximization} -- correctly classifying as many training examples as possible with maximal confidence -- has been known to construct models with good generalization guarantees. This work gives an information-theoretic interpretation of a margin maximizing model on a noiseless training dataset as one that achieves \emph{lossless maximal compression} of said dataset -- i.e. extracts from the features all the useful information for predicting the label and no more. The connection offers new insights on generalization in supervised machine learning, showing margin maximization as a special case (that of classification) of a more general principle and explains the success and potential limitations of popular learning algorithms like gradient boosting. We support our observations with theoretical arguments and empirical evidence and identify interesting directions for future work.
\end{abstract}

\section{Introduction}

The goal of a supervised learning algorithm is to construct a model on the training set that can \emph{generalize} well on new data. Yet, generalization is an elusive property, involving intractable quantities to be approximated or bound --like the \emph{generalization error}-- or notions with multiple definitions --like that of \emph{model complexity}. As a result there are many different theoretical routes to generalization, leading to often apparently `contradictory' conclusions with one another or with empirical evidence~\cite{zhang2016understanding}. For instance, why are certain learning algorithms that explore overparameterized or non-parametric model families so good at producing models that can generalize well, even without explicit regularization~\cite{buhlmann2007boosting, zhang2016understanding, kawaguchi2017generalization}? A unified language for comparing the complexity of models trained on a given dataset can help us identify good model selection and algorithmic practices that guide the learning algorithm towards models that are complex enough to not underfit yet also maximally resistant to overfitting.

In this paper we make a step towards this direction by bridging two --until now disconnected-- theoretical paths to generalization in the case of classification, namely \emph{information theory}~\cite{shannon1948mathematical} inspired by recent advances on the \emph{information bottleneck principle}~\cite{tishby2000information, tishby2015deep, shwartz2017opening} and \emph{margin theory}~\cite{vapnik1982estimation, schapire1998boosting}. From an information-theoretic perspective, we would like our learning algorithm to learn a model that contains all the information from the features necessary for describing the target (we call this property \emph{losslessness}) and no more information beyond that (we call this property \emph{maximal compression}). Margin theory suggests constructing a model that can correctly classify as many training examples as possible with as high confidence as possible (i.e. one that maximizes the quantity known as the \emph{functional margin} over the training set). We prove that in the case of classification of on noiseless (i.e. unambiguously labelled) datasets, functional margin maximization is equivalent to lossless maximal compression in the information-theoretic sense. The existence of margin-based bounds on the generalization error implies that margin maximization is beneficial for achieving good generalization and therefore so is lossless maximal compression.

Our experiments on gradient boosting, a method that maximizes the training margins, show empirically that on noiseless data, margin maximization amounts to lossless maximal compression and that maximally compressed models on average exhibit the highest generalization capability (as estimated by the test error). We identify interesting similarities between the way training progresses in Deep Neural Networks (DNNs) and in gradient boosting and gain useful insights on the training of gradient boosting algorithms. All findings persist across a wide range of datasets \& hyperparameter configurations.

To our knowledge, there is no prior work establishing the connection between functional margin maximization and lossless model compression in the information-theoretic sense. Both margin theory and information theory have been individually connected to generalization and have been used to explain resistance to overfitting. The idea that functional margin maximization promotes good generalization can be traced back to~\cite{vapnik1982estimation}. It has been used in the theoretical analysis of Boosting algorithms~\cite{schapire1998boosting}, with recent work using it to explain good generalization in DNNs~\cite{sokolic2017generalization, dziugaite2017computing, neyshabur2017exploring, wei2018margin}. The related notion of \emph{geometric margin maximization}\footnote{The geometric margin of a classifier is the distance of the closest example in the training set to the decision boundary. For linear models the geometric margin is a rescaling of the functional margin, therefore a model maximizing the one also maximizes the other.} has been used to justify good generalization in Support Vector Machines (SVMs)~\cite{Cortes1995}. The idea that a learned representation of a dataset that generalizes well is one that extracts from the features all the useful information for predicting the target and no more, is captured in information theoretic terms under the \emph{information bottleneck principle}~\cite{tishby2000information}. Recent work has offered insights into the good generalization capabilities of DNNs, utilizing these ideas~\cite{tishby2015deep, shwartz2017opening}. More recently, bounds on the generalization error of a learning algorithm in terms of the mutual information between its input and outputhave been established~\cite{xu2017information, asadi2018chaining}.

\newcommand{\numSamples}{n}
\newcommand{\unknownDistributionOnXY}{\mathbb{P}}
\newcommand{\probMeasure}{P}
\newcommand{\risk}{\mathcal{R}}
\newcommand{\empiricalRisk}{\risk_{\trainingSample}}

\newcommand{\E}{\mathbb{E}}
\newcommand{\empiricalMeasure}{\hat{\unknownDistributionOnXY}_{\numSamples}}
\newcommand{\one}{\mathbf{1}}

\section{Background}
\subsection{Binary Classification}
A \emph{classification algorithm}, receives as input a \emph{training dataset} $\trainingSample$ consisted of $\numSamples$ pairs $({\bf{x}}^{i}, y^{i})$ of feature vectors ${\bf{x}}^{i} \in \mathcal{X}$ and corresponding \emph{class labels} $y^{i} \in \mathcal{Y}$. The training set is drawn i.i.d. from some unknown probability measure $\unknownDistributionOnXY$ on $\X\times \Y$. We shall focus on \emph{binary classification}, where $\mathcal{Y} = \{-1,1\}$. In this setting, we consider w.l.o.g. the output of the learning algorithm (\emph{model}) as a function $f:\X \rightarrow [-1,1]$ that allows us to predict the label on unseen examples drawn from $\unknownDistributionOnXY$ with feature vector ${\bf{x}}$ as $\hat{y} = \sign(f({\bf{x}}))$. Given any probability measure $\probMeasure$ on $\X\times \Y$ and any function $f:\X \rightarrow \Y$ we let $\risk_{\probMeasure}(f)$ denote the probability of making an error with respect to distribution $\probMeasure$,
\begin{align}\label{riskDefnEqn}
\risk_{\probMeasure}(f) = \probMeasure\left[ \sign(f({X})) \neq Y \right].
\end{align}
Ideally the learning algorithm will output the model with the lowest possible risk w.r.t. the unknown distribution $\unknownDistributionOnXY$, i.e. a function $f$ that minimizes the \emph{true classification risk} $\risk_{\unknownDistributionOnXY}(f)$. However, since we do not have direct access to the unknown distribution $\unknownDistributionOnXY$ we must estimate $\unknownDistributionOnXY$ with the empirical measure $\empiricalMeasure$ defined for each set $A \subset \X \times \Y$ in terms of the training data $\trainingSample= \left\lbrace (\bf{x}^i,y^i)\right\rbrace_{i \in [\numSamples]}$  by
\begin{align}\label{empiricalMeasureDefn}
\empiricalMeasure\left(A\right) =\frac{1}{\numSamples} \cdot \sum_{i \in [n]}\one\left\lbrace (\bf{x}^i,y^i) \in A \right\rbrace.
\end{align}
The empirical risk $\empiricalRisk(f):=\risk_{\empiricalMeasure}(f)$ of a function $f:\X \rightarrow [-1,1]$ is given by
\begin{align*}
\empiricalRisk(f) = \empiricalMeasure\left[ \sign(f({X})) \neq Y \right]=\frac{1}{\numSamples} \cdot \sum_{i \in [n]}\one\left\lbrace \sign\left(f(\bf{x}^i)\right)\neq y^i \right\rbrace.
\end{align*}
In what follows we shall refer to a general \emph{finitely supported} measure $\probMeasure$ on $\X\times \Y$. The motivating example here is the empirical measure $\empiricalMeasure$ which is supported on the finite set $\trainingSample$, the training dataset.

\newcommand{\A}{\mathcal{A}}
\newcommand{\B}{\mathcal{B}}
\newcommand{\ASupp}{\A_{\probMeasure}}
\newcommand{\BSupp}{\B_{\probMeasure}}
\newcommand{\XSupp}{\X_{\probMeasure}}

\subsection{Information theory}
We now present some basic definitions and properties from information theory~\cite{shannon1948mathematical}. Let $A$ \& $B$ be random variables (RVs), with alphabets $\A$ \& $\B$ with probability distribution measure $\probMeasure$. We shall assume that $\probMeasure$ is finitely supported and there exist finite subsets $\ASupp \subset \A$ and $\BSupp \subset \B$ such that $\probMeasure\left[A \in \ASupp\right]=\probMeasure\left[B \in \BSupp\right]=1$.

The \emph{entropy} of a RV $A$, measures the amount of uncertainty associated with its value when only its distribution is known. It is defined by
\begin{equation}
\label{eq:entropy}
\entropyProbMeasure(A) = -\sum_{a \in \ASupp}\probMeasure(A=a)\log \left({\probMeasure(A=a)}\right).
\end{equation}

The amount of information shared by RVs $A$ \& $B$ is their \emph{mutual information}, defined as
\begin{equation}
\label{eq:mutual_info}
\mutualInformationProbMeasure(A;B) = \sum_{a \in \ASupp, b \in \BSupp}\probMeasure(A=a,B=b)\log{\frac{\probMeasure(A=a,B=b)}{\probMeasure(A=a)\probMeasure(B=b)}} = \mutualInformationProbMeasure(B;A).
\end{equation}

In terms of information theory, the chain rule of probability takes the form
\begin{equation}
\label{eq:chain_rule_it}
\mutualInformationProbMeasure(A;B) = \entropyProbMeasure(A) - \entropyProbMeasure(A|B) \overset{\mathrm{\mutualInformationProbMeasure(A;B) = \mutualInformationProbMeasure(B;A)}}{=} \entropyProbMeasure(B) - \entropyProbMeasure(B|A),
\end{equation}
where $\entropyProbMeasure(A|B)$ is the \emph{conditional entropy} of $A$ given $B$, given by
\begin{equation}
\label{eq:cond_entropy}
\entropyProbMeasure(A|B) = -\sum_{a \in \ASupp, b \in \BSupp}\probMeasure(A=a,B=b)\log{\probMeasure(A=a|B=b)},
\end{equation}
which measures the uncertainty of the value of RV $A$ given the value of RV $B$.
From Eq.~(\ref{eq:chain_rule_it}), it is clear that $\mutualInformationProbMeasure(A;B)$ measures the decrease in uncertainty about either the value of RV $A$ or the value of RV $B$, when the value of the other RV is known.

If $B$ is a deterministic transformation of $A$ then there is no uncertainty remaining about the value of $B$ given the value of $A$, so we have $\entropyProbMeasure(B|A) = 0 {\iff} \mutualInformationProbMeasure(A;B) = \entropyProbMeasure(B) \leq \entropyProbMeasure(A)$. Finally, if $G(A)$ is an invertible transformation of RV $A $, we have $\mutualInformationProbMeasure(G(A);B) = \mutualInformationProbMeasure(A;B)$ as the value of $A$ grants us perfect knowledge of the value of $G(A)$ and vice-versa.

In what follows we shall be particularly interested in the \emph{empirical entropy} $\entropySample:=H_{\empiricalMeasure}$ and the \emph{empirical mutual information} $\mutualInformationSample:=I_{\empiricalMeasure}$ where $\empiricalMeasure$ is the empirical measure of Eq.~(\ref{empiricalMeasureDefn}).

\subsubsection{The information bottleneck principle\\}
Suppose we wish to learn a compressed representation $F = f(X)$ from the original features $X$ that is useful for predicting a target variable $Y$. Treating $X$, $Y$ \& $F$ as RVs, the \emph{Information Bottleneck} principle~\cite{tishby2000information}, offers a way to select a representation $F$, by trading-off the information the learned representation $F$ captures from $X$ regarding the target variable $Y$, i.e. $\mutualInformationSample(F;Y)$ --the higher, the better for predicting $Y$-- and the total information it captures from $X$, i.e. $\mutualInformationSample(F;X)$ --the lower, the higher the degree compression. We thus look for a representation $F^{*}$ such that
\begin{equation}
\label{eq:ib}
F^{*} = \arg\min_{F} \{\mutualInformationSample(F;X) -\beta \mutualInformationSample(F;Y)\},
\end{equation}
where $\beta$ is a Lagrange multiplier that controls the aforementioned tradeoff.

Recently, this principle has been used to explain good generalization in DNNs~\cite{tishby2015deep, shwartz2017opening} and the use of training set estimates of Eq.~(\ref{eq:ib}) or variants of it as objective functions for training DNNs --and other models-- have been growing in popularity~\cite{alemi2016deep,strouse2017deterministic,simeone2018brief}. The reasoning is that compression controls for the complexity of the learned representation, thus promoting good generalization~\cite{shamir2008learning}.

In this work we draw inspiration from the above line of research and our findings further reinforce the role of information compression in promoting generalization. \emph{We regard the trained model's outputs as the `representation' $F$ and explore the properties of a model $F$ that minimizes $\mutualInformationSample(F;X)$ subject to maximizing $\mutualInformationSample(F;Y)$ on a given training set $\trainingSample$}. We shall call such an intuitively `ideal' model $F$ a \emph{lossless maximal compressor} of the training set $\trainingSample$.

\subsection{Margin maximization}
The (normalized) \emph{functional margin\footnote{Also known as the \emph{hypothesis margin}, or --in the case of ensembles-- the \emph{voting margin}.}} of a training example $({\bf{x}}^{i},y^{i})$ under a model $f$ is defined as
\begin{equation*}
\label{eq:margin}
m^{i} = y^{i} f({\bf{x}}^{i}) \in [-1, 1].
\end{equation*}
It is a combined measure of confidence and correctness of the classification of the example under $f$. Its sign encodes whether the example is correctly classified ($y^{i} f({\bf{x}}^{i})>0$) or misclassified ($y^{i} f({\bf{x}}^{i})<0$), while the magnitude of the margin (i.e. the magnitude of the score $f({\bf{x}}^{i})$) measures the \emph{confidence of the model in its prediction} (the higher, the more confident).





Maximizing the margins over the training set has been connected to good generalization~\cite{vapnik1982estimation,schapire1998boosting}. An upper bound to the generalization error $P_{\unknownDistributionOnXY}(yf({\bf{x}}) \leq 0)$ of an AdaBoost classifier $f$, based on its minimum margin over the training set, is proven in \cite{schapire1998boosting}. Tighter generalization bounds, dependent not only on the minimum margin but on the entire distribution of the training margins have been derived (e.g. \emph{Emargin bound}~\cite{wang2011refined}, \emph{k-th margin bound}~\cite{gao2013doubt}). Beyond boosting, such bounds hold for voting classifiers in general and recently similar bounds have been derived for DNNs~\cite{sokolic2017generalization, dziugaite2017computing, neyshabur2017exploring, wei2018margin}.

In this work, we will \emph{establish an equivalence between models that maximize the margins on a noiseless training set $\trainingSample$ and lossless maximal compressors of $\trainingSample$}. We will verify our observations empirically, using boosting, a method that explicitly minimizes a monotonically decreasing loss function of the margin (i.e. maximizes training examples' margins)\footnote{Adaboost approximately maximizes the average margin \& actually minimizes the margins' variance~\cite{shen2010dual}.}. As we will see, boosting drives learning towards a lossless maximal compressor of the noiseless training dataset. It achieves the lowest generalization error (estimated by the average test set error) once lossless maximal compression has been achieved.

\section{Lossless maximal compression \& margin maximisation}
\subsection{An information-theoretic view of datasets \& models}
\label{ssec:def}

We will now define properties that capture the relationship between the information content of the model's output ($F$) and the information present in the features \& the target ($X$ \& $Y$, respectively) as measured on the \emph{training dataset} $\trainingSample$. In Figure~$\ref{fig:FXY_Venn}$  we provide a visual summary of these properties and their possible combinations. In Table~$\ref{tab:FXY_Table}$  we summarize the information-theoretic equalities and inequalities that hold under each scenario. Proofs not directly following the statement of a lemma or theorem can be found in the Supplementary Material.

Any function $f:\featureSpace \mapsto [-1, 1]$, can be considered as a \emph{model} of the training dataset $\trainingSample$. Typically, the model constructed by a learning algorithm is a member of some given \emph{model family}. In this work we impose no restriction on the model space, i.e. $f \in \Phi$, where $\Phi$ is the set of all models. Being a \emph{deterministic} transformation of $X$, $F = f(X)$ cannot contain more information than $X$. So,
$\entropyProbMeasure(F|X)=0 \iff \mutualInformationProbMeasure(X;F) = \entropyProbMeasure(F) \leq \entropyProbMeasure(X)$.

\begin{definition}[Noiselessness]
\label{def:noiselessness}
A probability distribution $\probMeasure$ is noiseless if and only if $\entropyProbMeasure(Y|X)=0$.\\
Otherwise, $\probMeasure$ is noisy and $\entropyProbMeasure(Y|X)>0$. We shall say that a dataset $\trainingSample$ is noiseless (respectively, noisy) if the corresponding empirical measure $\empiricalMeasure$ is noiseless (respectively, noisy).
\end{definition}

Under this information-theoretic perspective, a noiseless distribution $\probMeasure$, hence a noiseless dataset $\trainingSample$, is one in which \emph{the features $X$, contain all information to perfectly describe the target $Y$}. 

Given a distribution $\probMeasure$ on $\X\times \Y$ we let $\riskMinimiserProb:\X \rightarrow\Y$ denote a minimizer of the risk i.e.
\begin{align*}
\riskMinimiserProb \in \arg \min_{f \in \Phi} \left\lbrace \risk_{\probMeasure}\left(f\right) \right\rbrace.
\end{align*}
In particular, when $\probMeasure$ is the underlying distribution $\unknownDistributionOnXY$ then $\sign \circ \riskMinimiser_{\unknownDistributionOnXY}$ is the Bayes classifier. When $\probMeasure$ is the empirical measure $\empiricalMeasure$ then $\sign \circ \empiricalRiskMinimiser$ is the empirical risk minimiser where $\empiricalRiskMinimiser:=\riskMinimiser_{\empiricalMeasure}$.

\begin{lemma}\label{noiselessIFFZeroRiskLemma}
A dataset $\trainingSample$ is noiseless if and only if $\empiricalRisk(\empiricalRiskMinimiser)  = 0$. 
\end{lemma}

In other words, a noiseless training dataset $\trainingSample$ is one in which \emph{no datapoints with the same feature vector $\mathbf{x}$ have different labels $y$}. For such a dataset\footnote{Also known as a \emph{unambiguously labelled} or \emph{consistent} training dataset in the literature.}, there exists a model that can achieve zero empirical risk (training error), i.e. that can perfectly classify the training data. In other words, there exists some \emph{deterministic mapping} from the features $X$ to the target $Y$.

We shall now introduce properties that make a model $f$ useful for the purposes of capturing relevant and ignoring redundant information from a training set $\trainingSample$.

\begin{definition}[Losslessness]
\label{def:losslessness}
A model $f$ is lossless on the dataset $\trainingSample$ if and only if  $\mutualInformationSample(F;Y) = \mutualInformationSample(Y;X)$.
Otherwise, the model is lossy on $\trainingSample$ and $\mutualInformationSample(F;Y) < \mutualInformationSample(Y;X)$.
\end{definition}

A lossless\footnote{Often the term "lossless encoding" of some r.v. $X$ in the literature characterizes an encoding $f(X)$ that allows us to recover the original value of $X$ from it. In our case, because of the supervised nature of the learning task, it shall mean that $f(X)$ allows us to recover the original value of the target $Y$ from it. Not necessarily the value of the feature vector.} model $f$ on a dataset $\trainingSample$ is one that \emph{captures all the information in features $X$ that is relevant for describing the target $Y$}. We can equivalently state that the r.v. $F$ is a \emph{sufficient statistic} of the empirical distribution $\empiricalMeasure$ of the training data. 

\begin{lemma}\label{losslessIFFARiskMinimiserUpToInvertibleTransformationLemma}
Suppose dataset $\trainingSample$ is noiseless. A model $f$ is lossless on  $\trainingSample$ if and only if there exists an invertible transformation $g:[-1,1]\rightarrow \R$ such $\empiricalRisk(g \circ f) =\empiricalRisk(\empiricalRiskMinimiser)$. 
\end{lemma}

Lemma 2 means that if a model $f$ is lossless on a training set $\trainingSample$, its \emph{output can be used to describe the target $Y$ with the only source of training error being the irreducible class overlap in the training set}.





\begin{definition}[Maximal Compression]
\label{def:max_comp}
A model $f:\X\rightarrow [-1,1]$ is a maximal compressor of  the dataset $\trainingSample$ if and only if $\mutualInformationSample(F;X) = \mutualInformationSample(F;Y)$.
Otherwise, the model is undercompressed on $\trainingSample$ and $\mutualInformationSample(F;X) > \mutualInformationSample(F;Y)$. 
\end{definition}

A model $f$ that is a maximal compressor of a training dataset $\trainingSample$ is \emph{one that only captures from the features $X$ information relevant for describing the target $Y$}. It does not necessarily capture \emph{all} that information; this special case, merits a definition of its own given below.

\begin{definition}[Lossless Maximal Compression - (LMC)]
\label{def:lossless_max_comp}
A model $f$ is a lossless maximal compressor (LMC) of a training dataset $\trainingSample$ if and only if it is lossless on $\trainingSample$ and a maximal compressor on $\trainingSample$. 
\end{definition}

\begin{proposition}
\label{thm:lmc}
A model $f$ is an LMC of a training dataset $\trainingSample$, if and only if it satisfies
\begin{equation*}
\mutualInformationSample(F;X) = \mutualInformationSample(F;Y) = \mutualInformationSample(Y;X).
\end{equation*}
\end{proposition}

\begin{proof}
Follows straightforwardly from Definition~\ref{def:losslessness} \& Definition~\ref{def:max_comp}.\end{proof}

A model $f$ that is an LMC of a training dataset $\trainingSample$ is \emph{one that only captures from the features $X$ all the information relevant for describing the target $Y$}. \emph{From an information-theoretic perspective, an LMC of $\trainingSample$ is the optimal classification model that can be constructed from $\trainingSample$}.

We have defined the notion of a noiseless / noisy training dataset $\trainingSample$ and those of a lossy / lossless $f$ on $\trainingSample$ and of an undercompressed / maximally compressed model $f$ on $\trainingSample$. In Figure~$\ref{fig:FXY_Venn}$ we provide a visual summary of these properties and their possible combinations in the form of entropy Venn diagrams. In Table~$\ref{tab:FXY_Table}$ we summarize the relationships among the various information-theoretic quantities involved that hold under each scenario.

In the next subsections, we shall use the properties we defined here to obtain a better understanding of what types of models the information-theoretically optimal model, the LMC corresponds to for a noiseless dataset $\trainingSample$ and for a noisy dataset $\trainingSample$.

\begin{table}[b]
 \label{tab:FXY_Table}
 \caption{Relationships among various information-theoretic quantities under the possible scenarios (see Figure \ref{fig:FXY_Venn}) relating the feature joint RV $X$ and the target RV $Y$ in a training dataset $\trainingSample$ with a model $f$ whose output is distributed as a RV $F$.}
 \small
\centering
 \begin{tabular}{|c||c|} 
 \hline
 \multicolumn{2}{|c|}{\textbf{Noiseless Dataset}} \\  \hline\hline
 \textbf{Lossy} & \multirow{2}{*}{$\mutualInformationSample(F;Y) < \mutualInformationSample(X;Y) = \entropySample(Y) < \mutualInformationSample(X;F) = \entropySample(F) \leq \entropySample(X)$} \\
 \textbf{Undercompressed} & \\ \hline 
 \textbf{Lossy} &  \multirow{2}{*}{$\mutualInformationSample(F;Y) < \mutualInformationSample(X;Y) = \entropySample(Y) = \mutualInformationSample(X;F) = \entropySample(F) \leq \entropySample(X)$} \\ 
 \textbf{Maximally Compressed} & \\ \hline 
 \textbf{Lossless} & \multirow{2}{*}{$\mutualInformationSample(F;Y) = \mutualInformationSample(X;Y) = \entropySample(Y) < \mutualInformationSample(X;F) = \entropySample(F) \leq \entropySample(X)$} \\
 \textbf{Undercompressed} & \\ \hline 
 \textbf{Lossless} & \multirow{2}{*}{$\mutualInformationSample(F;Y) = \mutualInformationSample(X;Y) = \entropySample(Y) = \mutualInformationSample(X;F) = \entropySample(F) \leq \entropySample(X)$} \\
 \textbf{Maximally Compressed} & \\ \hline 
 \multicolumn{2}{|c|}{\textbf{Noisy Dataset}} \\  \hline\hline
 \textbf{Lossy} & $\mutualInformationSample(F;Y) < \mutualInformationSample(X;Y) < \entropySample(Y)$ \\ 
 \textbf{Undercompressed} & $\mutualInformationSample(X;Y) < \mutualInformationSample(F;X) = \entropySample(F) \leq \entropySample(X)$ \\ \hline
 \textbf{Lossy} & $\mutualInformationSample(F;Y) < \mutualInformationSample(X;Y) < \entropySample(Y)$ \\ 
 \textbf{Maximally Compressed} & $\mutualInformationSample(X;Y) = \mutualInformationSample(F;X) = \entropySample(F) \leq \entropySample(X)$ \\ \hline
\textbf{Lossless} & $\mutualInformationSample(F;Y) = \mutualInformationSample(X;Y) < \entropySample(Y)$ \\ 
 \textbf{Undercompressed} & $\mutualInformationSample(X;Y) < \mutualInformationSample(F;X) = \entropySample(F) \leq \entropySample(X)$ \\ \hline
\textbf{Lossless} & $\mutualInformationSample(F;Y) = \mutualInformationSample(X;Y) < \entropySample(Y)$ \\ 
 \textbf{Maximally Compressed} & $\mutualInformationSample(X;Y) = \mutualInformationSample(F;X) = \entropySample(F) \leq \entropySample(X)$ \\ \hline
 \end{tabular}
\end{table}

\begin{figure}[H]\label{fig:FXY_Venn}
\centering
\subfigure{\includegraphics[width=0.7\textwidth]{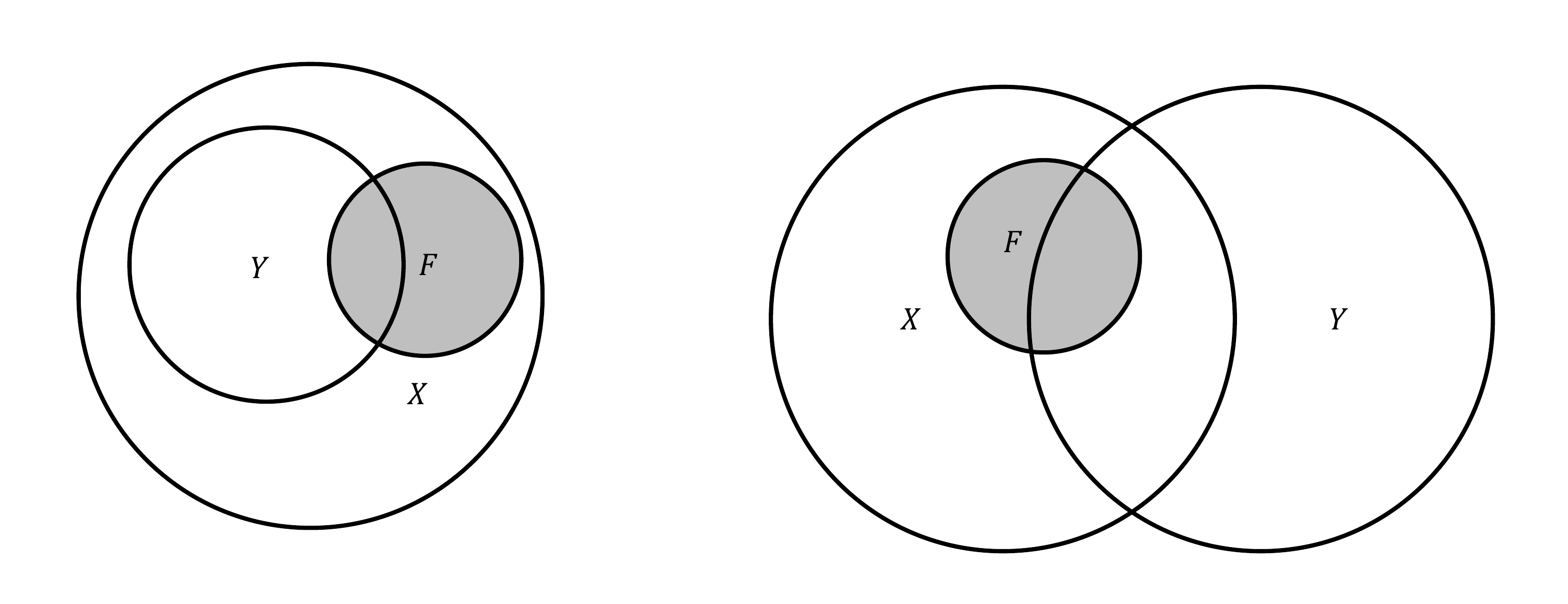}}
\subfigure{\includegraphics[width=0.7\textwidth]{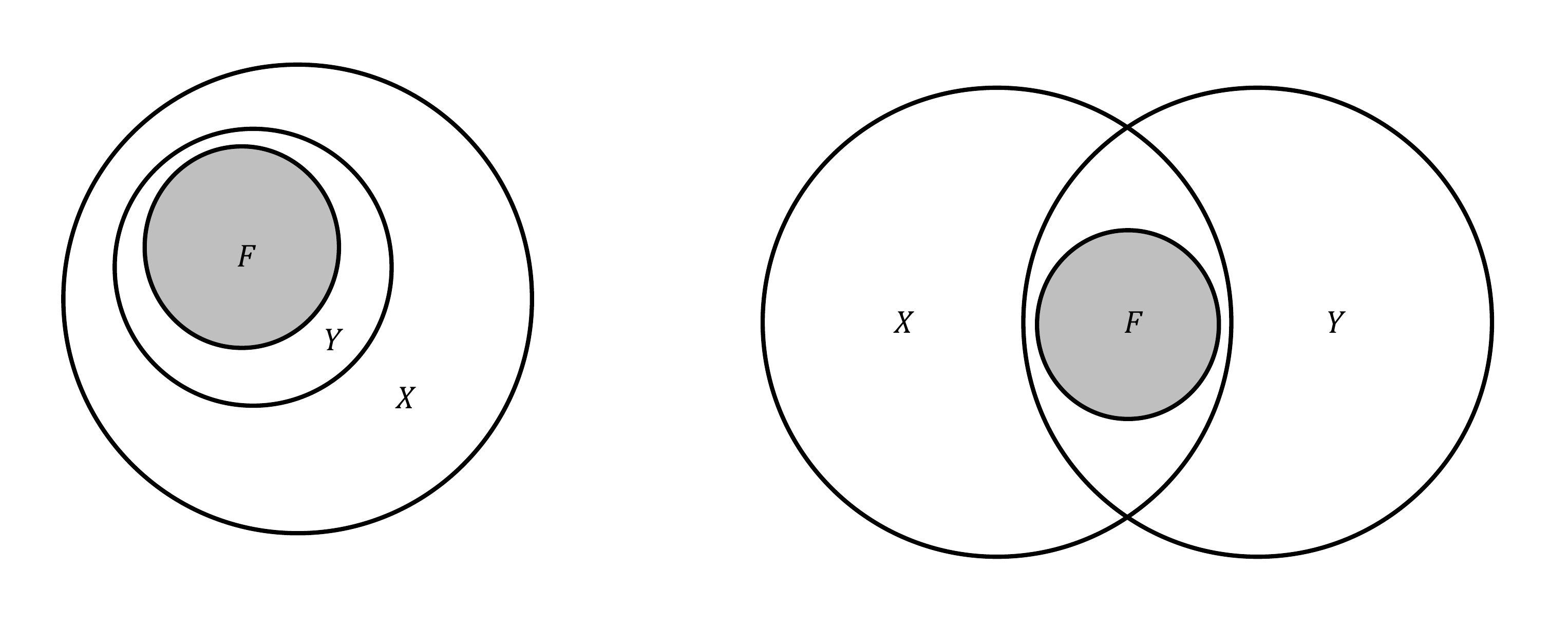}}
\subfigure{\includegraphics[width=0.7\textwidth]{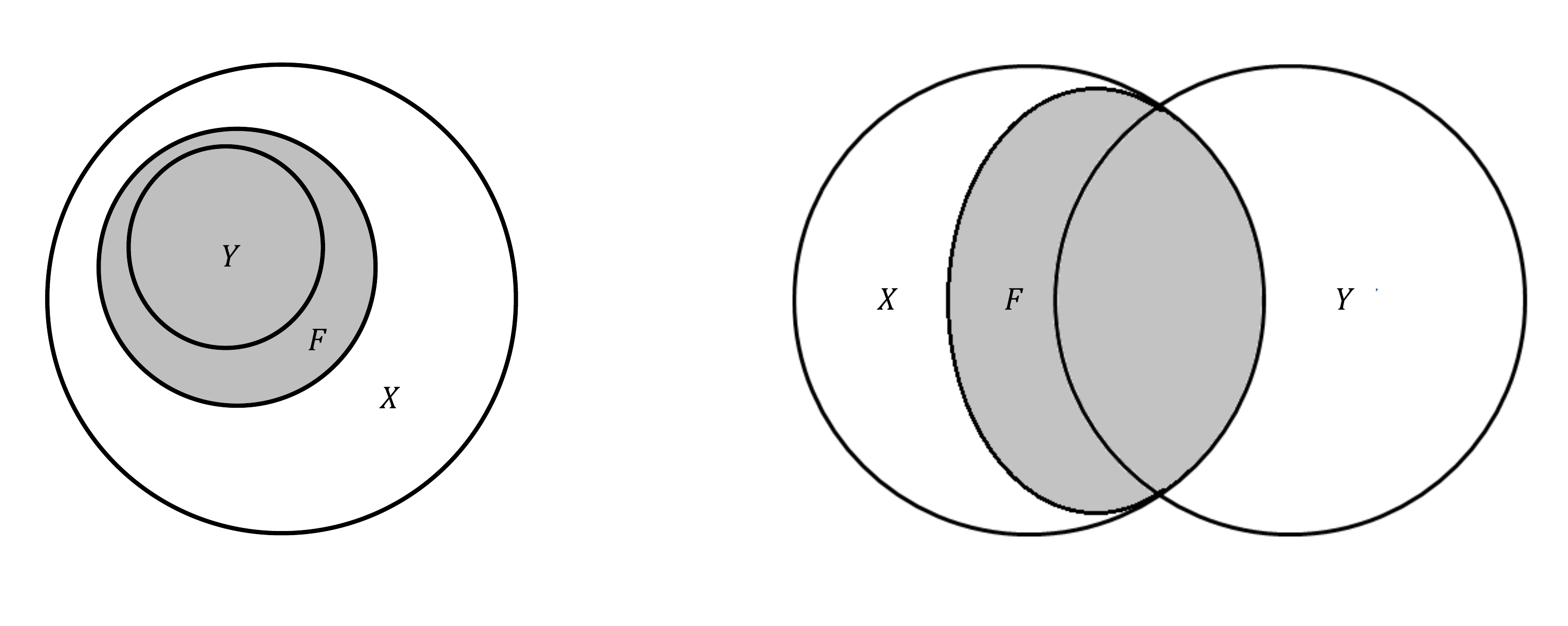}}
\subfigure{\includegraphics[width=0.7\textwidth]{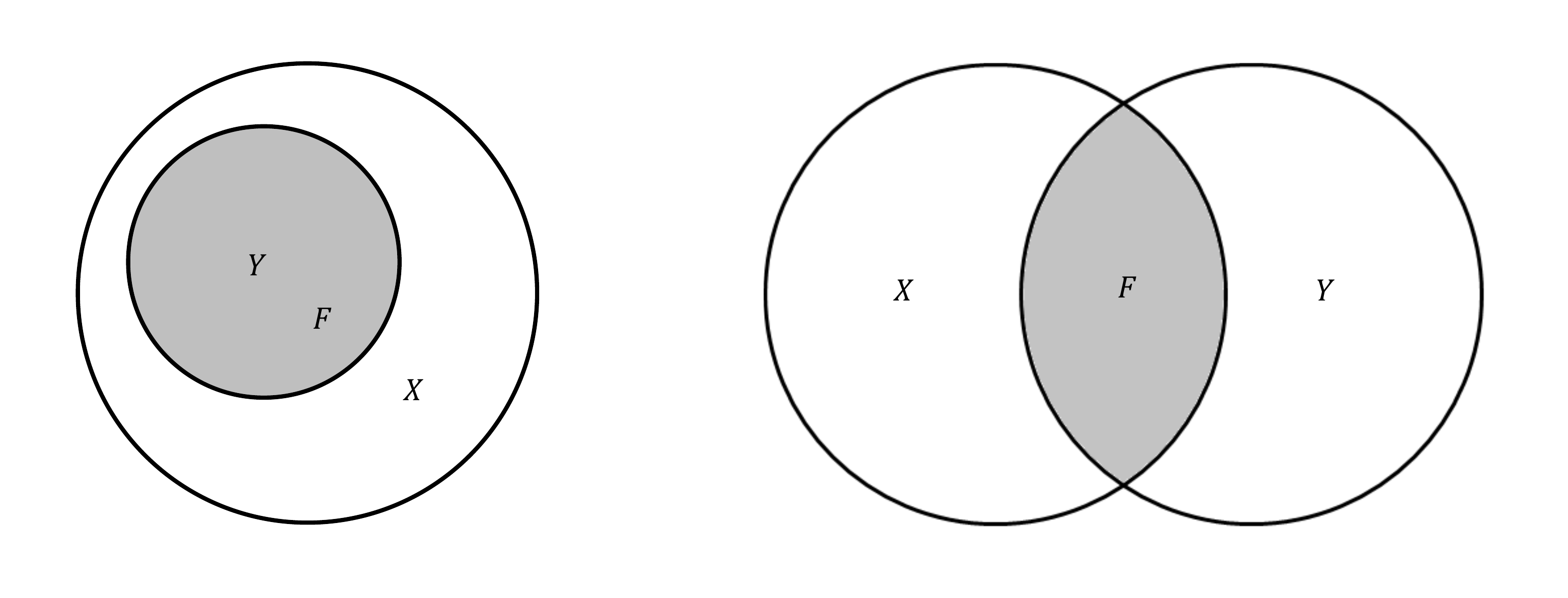}}
\caption{Venn diagrams capturing the possible relationships between the information content of the feature joint RV $X$ and the target RV $Y$ in a training dataset $\trainingSample$ with a RV $F$ representing the output of a model $f$ (shaded). [Left] $\trainingSample$ is a noiseless dataset; $X$ contains all information to perfectly describe $Y$. [Right] $\trainingSample$ is a noisy dataset; $X$ does not contain all information to perfectly describe $Y$. [From Top to Bottom] (i) A lossy, undercompressed model $f$. (ii) A lossy, maximally compressed model $f$. (iii) A lossless, undercompressed model $f$. (iv) A lossless, maximally compressed model $f$. Table \ref{tab:FXY_Table} shows the equalities \& inequalities involving the various underlying information-theoretic terms.}
\end{figure}




\newpage

\subsection{Noiseless data: Equivalence of lossless maximal compression \& margin maximisation}

Let us first focus on the special case of a noiseless dataset $\trainingSample$, i.e. a dataset that does not contain any datapoints with the same feature vector but different labels. We will then discuss the noisy case where ambiguously labelled datapoints can be present in the dataset. 

The noiseless case merits a special discussion for several reasons: (i) It is the typical case studied in the literature and as such it allows us to connect our observations to existing work. (ii) It allows us to establish an equivalence between information theoretic lossless maximal compression and margin maximization. (iii) It is a very common case in practice since in large dimensional datasets, encountering datapoints that have the same feature vector but different class labels are typically expected to be rare\footnote{This is because encountering datapoints that have the same feature vector in high dimensional feature spaces is typically expected to be rare in the first place.}. 

We will now show the equivalence between lossless maximal compression and margin maximisation on a noiseless dataset $\trainingSample$.

\begin{theorem}
 \label{thm:margin_max_lmc} Suppose $\trainingSample$ is noiseless and finitely supported. A model $f$ is an LMC with respect to $\trainingSample$ if and only if there exists some invertible transformation $g:[-1,1]\rightarrow\R$ such that $g \circ f$ is a margin maximizer with respect to $\trainingSample$. 
\end{theorem}

Under Theorem \ref{thm:margin_max_lmc} a classification model that maximizes the training margins on a given noiseless dataset\footnote{A margin maximizer on a noiseless dataset is a model that correctly classifies all training examples, with maximal confidence. Obviously, if $y^{i} f(\mathbf{x}^{i}) = 1, \forall (\mathbf{x}^{i}, y^{i})\in \trainingSample$, then both the average and the minimal margin of $f$ on $\trainingSample$ are equal to $1$ (maximal) and the variance of the margin distribution of $f$ on $\trainingSample$ is $0$ (minimal).} is one that captures all the information present in the features of that dataset relevant for predicting the target label and no more. Conversely, since an LMC is a margin maximizer, it offers the same guarantees on the generalization error as the latter. Note that Theorem \ref{thm:margin_max_lmc} captures a relationship between a noiseless dataset $\trainingSample$ and a model $f$, regardless of the underlying learning algorithm that produced it (i.e. the model family it explores or the optimization method used to explore it).

From Lemma $1$ \& Lemma $2$, we have that a lossless model $f$ on a noiseless training dataset $\trainingSample$ is one whose output can be used to classify every training example to the correct class (i.e. $\trainingSample$ is \emph{separable} by $f$). The success of algorithms that generate models that can \emph{interpolate}\footnote{An \emph{interpolating classifier} is one that can \emph{perfectly separate} the data, i.e. \emph{achieve zero training error}. In our terminology it is a \emph{lossless model on a noiseless dataset}, as such it falls within the case examined here.} the data, yet, despite exploring overparameterized model spaces, are resistant to overfitting (e.g. gradient boosting, random forests, SVMs and DNNs) has recently attracted considerable research interest~\cite{wyner2015explaining, belkin2018reconciling, hastie2019surprises}.

Our work connects these findings to information theory and margin theory: we posit that models generated by such methods are typically not simply lossless, but actually LMCs, hence margin maximizers and their good generalization follows via the margin-based generalization bounds. Algorithms such as the aforementioned, have mechanisms for promoting both losslessness (interpolation, in a noiseless dataset), guaranteeing the model produced will not underfit (afforded by overparameterizing the model space) and maximal compression (afforded via explicit or implicit margin maximization) which produces a model from that space that is maximally resistant to overfitting.

\subsection{Noisy data: The equivalence collapses}
Let us now discuss the case of a noisy dataset $\trainingSample$ and how it differs from the case of noiseless data.


In the noiseless case, any model that correctly classifies every training datapoint in $\trainingSample$ (i.e. achieves zero training error) is a lossless model on $\trainingSample$. In the noisy case this observation is no longer relevant, as there exist at least 2 datapoints which are noisy, i.e. have the same feature vector $\mathbf{x}$, but different labels $y$. It is no longer the case that there exists a model $f$ that can perfectly separate the data. 

We can also rephrase the observation stated above as follows: Any model $f$ that yields the \emph{minimal achievable training error} on a noiseless training dataset $\trainingSample$ is a lossless model on $\trainingSample$. As we will see from Lemma~\ref{losslessIFFCapturesCCProbLemma}, this condition is no longer sufficient for $f$ to be lossless on a noisy training dataset $\trainingSample$. 

A model that minimizes the training error on a noisy dataset $\trainingSample$ will be one that classifies all points in $\trainingSample$ with the same feature vector $\mathbf{x}$ to the majority class among them. Furthermore, a margin maximizer~\footnote{We remind the reader that we refer to minimizers of the average (equivalently: total) margin over the training examples with this term.} $f$ on $\trainingSample$ is a model that minimizes the training error while also assigning maximal absolute score to its predictions (i.e. $|f(\mathbf{x})| = 1, \forall \mathbf{x}$). It is easy to see that -- unlike in the case of a noiseless training dataset $\trainingSample$ where a margin maximizer was an LMC-- in the noisy case, a margin maximizer cannot even be a lossless model. This is a direct consequence of Lemma~\ref{losslessIFFCapturesCCProbLemma}, the proof of which can be found in Section A of the Supplementary Material.


\newcommand{\Z}{\mathcal{Z}}
\newcommand{\Prob}{\mathbb{P}}
\newcommand{\entropyFunction}{\phi}

\begin{lemma}\label{losslessIFFCapturesCCProbLemma}
Suppose that $X$ and $Y$ are discrete random variables taking values in $\X$ and $\Y=\{-1,+1\}$, respectively. Suppose that $f:\X \rightarrow \Z$ and let $F=f(X)$. Then $I(X;Y)=I(F;Y)$ if and only if the map $x\mapsto \Prob(Y=1|X=x)$ is constant on all sets of the form $f^{-1}(z)\subseteq \X$ for some $z \in \Z$.
\end{lemma}

In simpler terms, Lemma~\ref{losslessIFFCapturesCCProbLemma} tells us that if for two feature vectors $\mathbf{x_1}$ and $\mathbf{x_2}$ a model $f$ satisfies $f(\mathbf{x_1}) =f(\mathbf{x_2})$, then it also has to be the case that $\Prob(Y=1|X=\mathbf{x_1}) = \Prob(Y=1|X=\mathbf{x_2})$ for $f$ to be lossless (and inversely). Therefore, a margin maximizer, i.e. a model assigning maximal (i.e. \emph{the same}) score both to noiseless positive examples (unambiguously labelled positive examples, for which $\Prob(Y=1|X=\mathbf{x})= 1$) and to noisy popsitive examples (ambiguously labelled examples, i.e. ones with $0.5<\Prob(Y=1|X=\mathbf{x})<1$) violates the condition of losslessness.

We therefore see that a margin maximizing model $f$ of a noisy training dataset $\trainingSample$ cannot be an LMC of $\trainingSample$ as it is not even lossless on $\trainingSample$. Furthermore, as margin maximizers are themselves training error minimizers, this implies that not all training error minimizers of $\trainingSample$ are LMCs (or even lossless) on $\trainingSample$ either. These observations are summarized in Table~\ref{tab:noisy_vs_noiseless}.

A lossless model (one satisfying $I_{S}(F;Y) = I_{S}(X;Y)$) is one that captures all the information present in the features $X$ relevant for predicting the target $Y$. In the case of a noisy training dataset, this information includes the uncertainty introduced by the ambiguous labelling of a feature vector $\mathbf{x}$, i.e. $\Prob(Y=1|X=\mathbf{x})$. So a lossless model should assign different scores $f(\mathbf{x})$ to feature vectors $f(\mathbf{x})$ which have different values of $\Prob(Y=1|X=\mathbf{x})$.

Moreover, $f$ is an LMC (has the minimum $I_{S}(F;X)$ that allows $I_{S}(F;Y) = I_{S}(X;Y)$) iff it is lossless while using the fewest values $f(\mathbf{x})$ possible to encode the empirical $\Prob(Y=1|X=\mathbf{x})$, i.e. have as many distinct values for $f(\mathbf{x})$ as there are distinct values of $\Prob(Y=1|X=\mathbf{x})$.




The above discussion provides an intuition into the limitations of margin maximization approaches in the presence of label noise. The sub-optimality of boosting (a margin maximization approach) in the presence of label noise has been observed in earlier studies~\cite{kalai2003boosting, servedio2003smooth, bootkrajang2013boosting} and here we provide an information-theoretic justification for this phenomenon. Simply put, the strategy of maximizing the margins on a noisy dataset is not producing a lossless maximally compressing model on that dataset. In fact, the resulting margin maximizing model, is not even going to be lossless on the training dataset as it will fail to capture the uncertainty over the labels. When the training data are noisy, we should instead aim to produce models whose scores $f(\mathbf{x})$ capture the underlying empirical distribution $\Prob(Y=1|X=\mathbf{x})$ (lossless). Ideally, we should aim for strategies producing models whose scores $f(\mathbf{x})$ are in $1-1$ correspondence to $\Prob(Y=1|X=\mathbf{x})$ (LMCs).

\begin{table}[h]
\centering
\caption{A brief comparison of loslessness and lossless maximal compression (LMC) and how it relates to training error minimization and margin maximization depending on whether the training dataset $\trainingSample$ is noisy or noiseless.}\label{tab:noisy_vs_noiseless}
\label{tab:my-table}
\begin{tabular}{c|c|c|}
\cline{2-3}
& Noiseless training dataset $\trainingSample$ & Noisy training dataset $\trainingSample$\\ \hline
\multicolumn{1}{|c|}{Losslessness} & \begin{tabular}[c]{@{}c@{}}All\\training error minimizers\\are lossless\end{tabular} & \begin{tabular}[c]{@{}c@{}}Some\\training error minimizers\\are lossless\\ \\No\\margin maximizer\\is lossless\end{tabular} \\ \hline
\multicolumn{1}{|c|}{LMC} & \begin{tabular}[c]{@{}c@{}}All\\margin maximizers\\are LMCs\end{tabular} & \begin{tabular}[c]{@{}c@{}}Some\\ training error minimizers\\are LMCs\end{tabular} \\ \hline
\end{tabular}
\end{table}

\section{Empirical Evidence}

\subsection{Experimental Setup}
Boosting, a method that explicitly maximizes the margins of the training examples\footnote{Gradient boosting is a family of ensemble learning methods that construct an additive model by adding on each round the component minimizing some monotonically decreasing loss function of the margin. It can be viewed as minimizing said loss by performing gradient descent on the space of components (base learners).}, can be shown empirically to also converge to LMC models on noiseless datasets. After lossless maximal compression is achieved, so is the minimal generalization error, as estimated by the error on the test set. To demonstrate this, we plot the \emph{trajectory} of the boosting ensemble on the \emph{entropy-normalized information plane}, $\mutualInformationSample(F;Y)/\entropySample(Y)$ vs. $\mutualInformationSample(F;X)/\entropySample(X)$. For each boosting round $t$, $F = F_t$ denotes the RV of which the ensemble's outputs are realizations. 

The experiments were carried out on binary classification tasks on both real-world UCI datasets and artificial data (dataset descriptions in the Supplementary Material). Qualitatively, the results are similar for all datasets (see Figure \ref{fig:trajectories_real} as well as Section C of the Supplementary Material). The boosting ensemble consisted of a maximum of $T=100$ decision trees (i.e. rounds of boosting) of maximal depth $6$. No shrinkage of the updates or subsampling of the examples was performed (both are techniques to counter overfitting), and the exponential loss function was used (i.e. the loss minimized by AdaBoost). We performed no hyperparameter optimization. 
Plotting trajectories on the information plane follows~\cite{tishby2015deep, shwartz2017opening}. All information-theoretic quantities were estimated on the training data by first discretizing the features \& model outputs in $b = 100$ equal-sized bins\footnote{Note that by discretizing the features we might convert an originally noiseless dataset into a noisy one. In the experiments included in this paper this did not happen for any dataset for the numbers of discretization bins chosen. So all results shown are on noiseless datasets.}, then using maximum likelihood estimators. The joint RV $X$ was then constructed by the discretized features $X_1, X_2, \dots, X_d$ as $X = \sum_{i=1}^{d} X_{i}b^{i-1}$. We plot average results across $100$ runs with different train-test splits ($50\%$--$50\%$) on the same original data. We also visualize the trajectories obtained by some random individual runs to showcase that although they can vary significantly from one another, they all follow the same general pattern. All datasets \& code used in the experiments can be downloaded at \url{https://github.com/nnikolaou/margin_maximization_LMC}.


\subsection{Results \& Analysis}

Let us first introduce some characteristic points on the information plane:\\
\textbf{Lossless maximal compression (LMC) point:} A red star on the information plane denotes the point of lossless maximal compression -- the optimal feasible point a model $f$ can occupy on the plane on a given dataset -- on which $\mutualInformationSample(F;Y)$ is the maximal achievable while $\mutualInformationSample(F;X)$ is minimal. On this point, $\mutualInformationSample(F;Y) = \mutualInformationSample(F;X) = \mutualInformationSample(X;Y)$ and for a noiseless dataset, $\mutualInformationSample(X;Y)/\entropySample(Y) = 1$.\\ 
\textbf{Average margin maximization point:} With a  hollow green circle on the information plane, we denote the model (round of boosting) under which the average (equiv. total) margin is first minimized.\\
\textbf{Training error minimization point:} With a full black dot on the information plane, we denote the model (round of boosting) under which the training error is first minimized (losslessness is achieved). At this point $\mutualInformationSample(F;Y)$ has reached its maximum, so for a noiseless dataset, $\mutualInformationSample(X;Y)/\entropySample(Y) = 1$.\\
\textbf{Test error minimization point:} With a magenta square on the information plane, we denote the model (round of boosting) under which the test error (proxy for generalization error) is first minimized.

\begin{figure}
\centering
\subfigure{\includegraphics[width=0.49\textwidth]{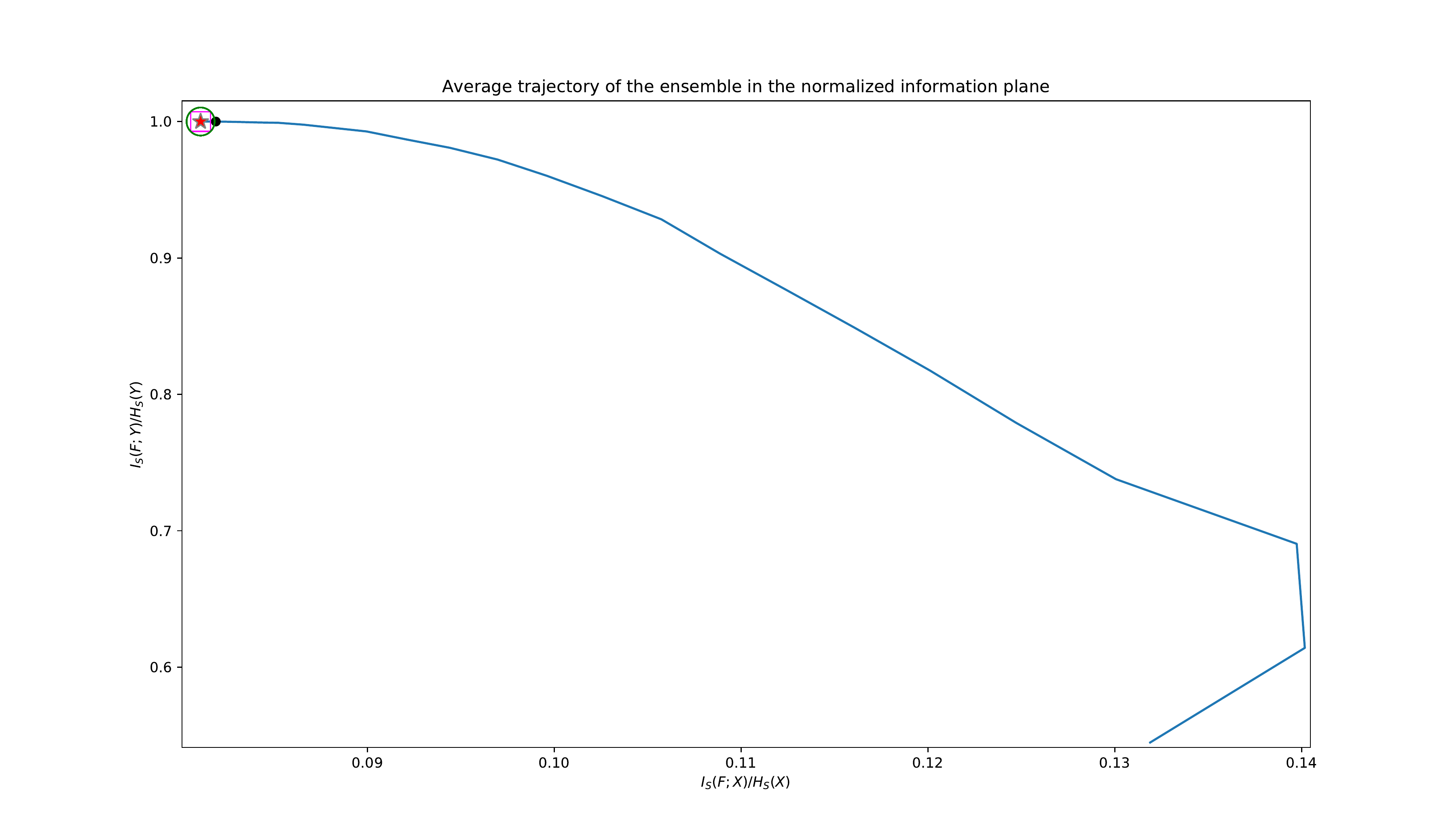}}
\subfigure{\includegraphics[width=0.49\textwidth]{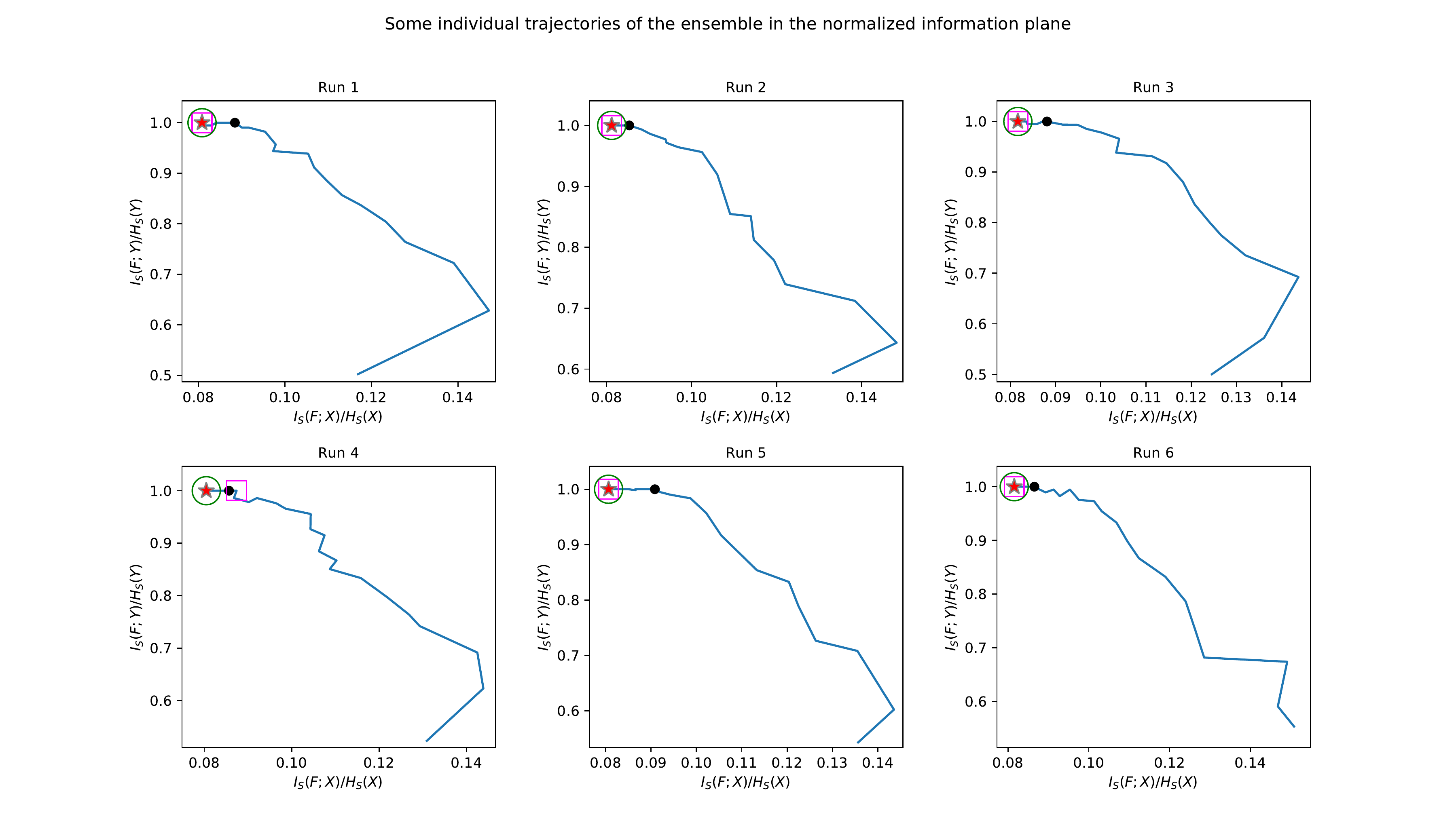}}
\subfigure{\includegraphics[width=0.49\textwidth]{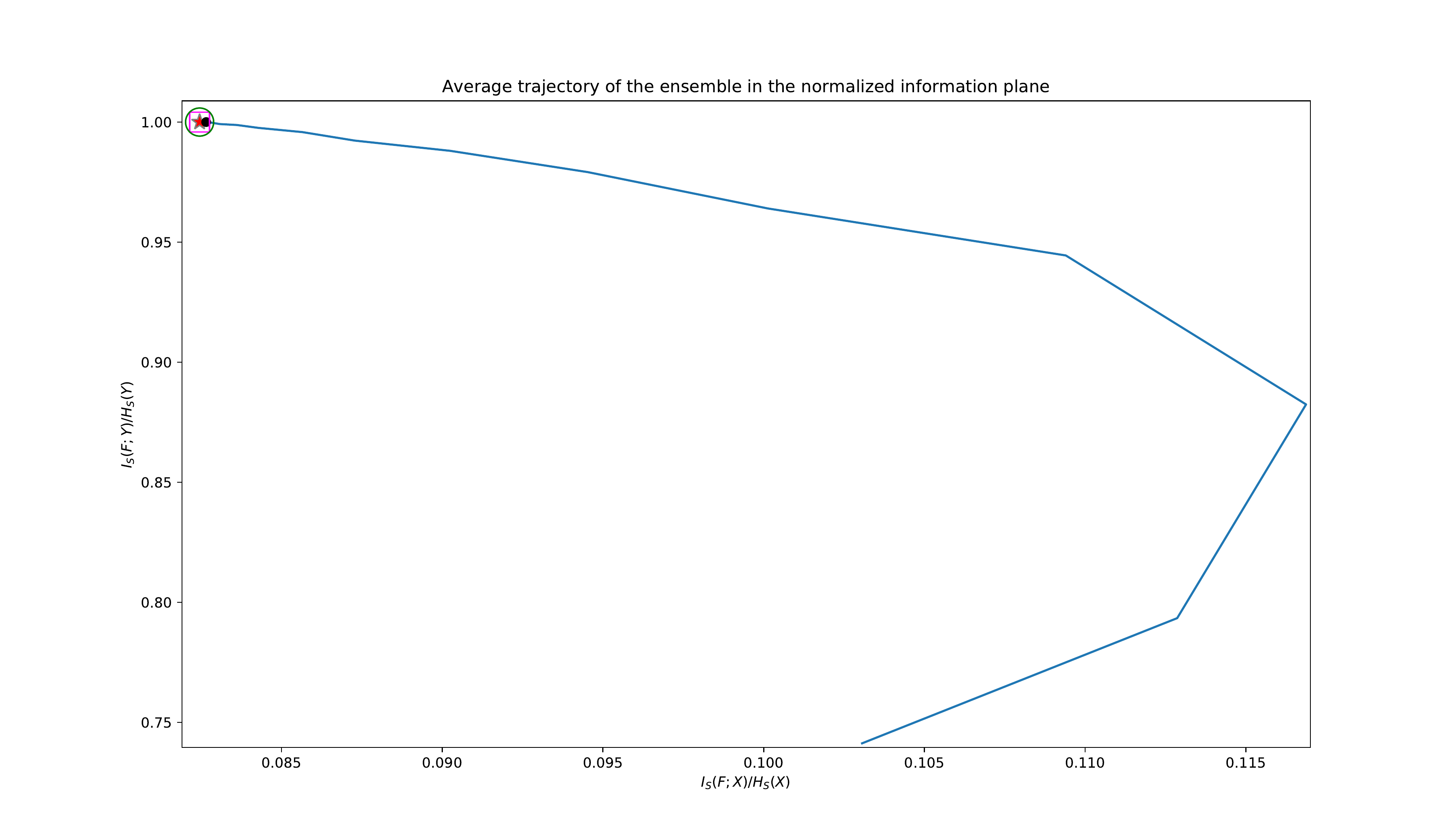}}
\subfigure{\includegraphics[width=0.49\textwidth]{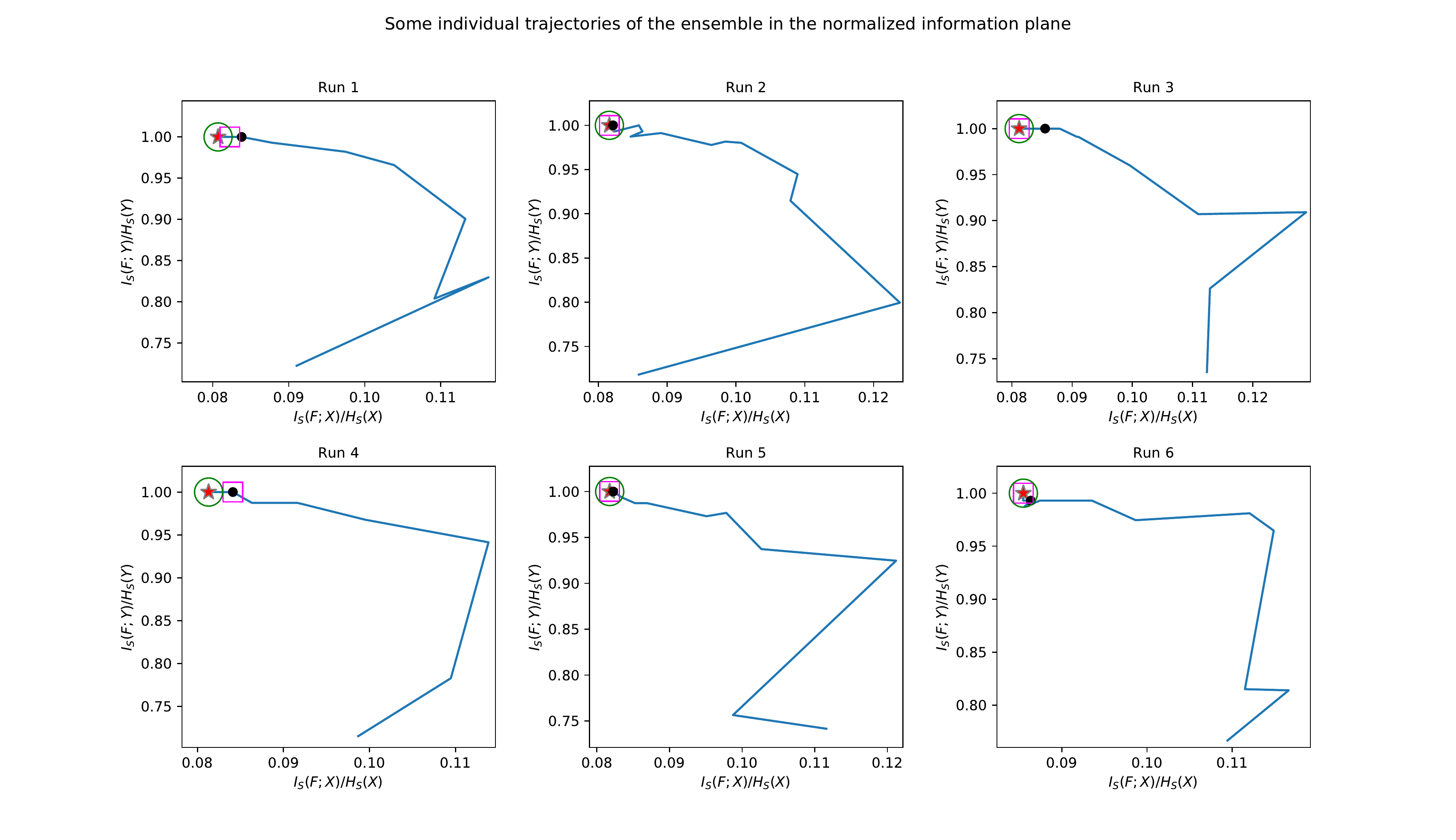}}
\subfigure{\includegraphics[width=0.49\textwidth]{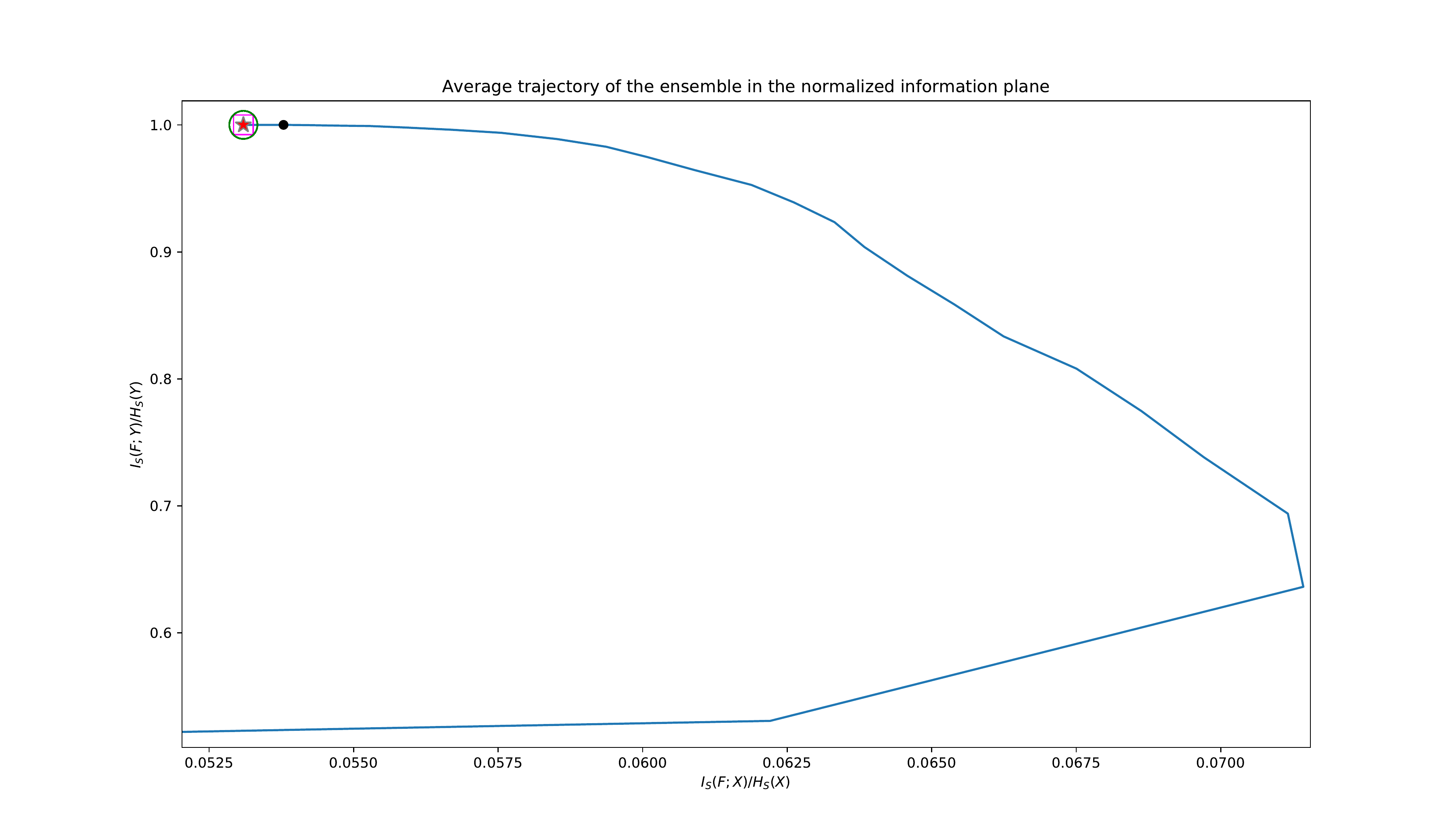}}
\subfigure{\includegraphics[width=0.49\textwidth]{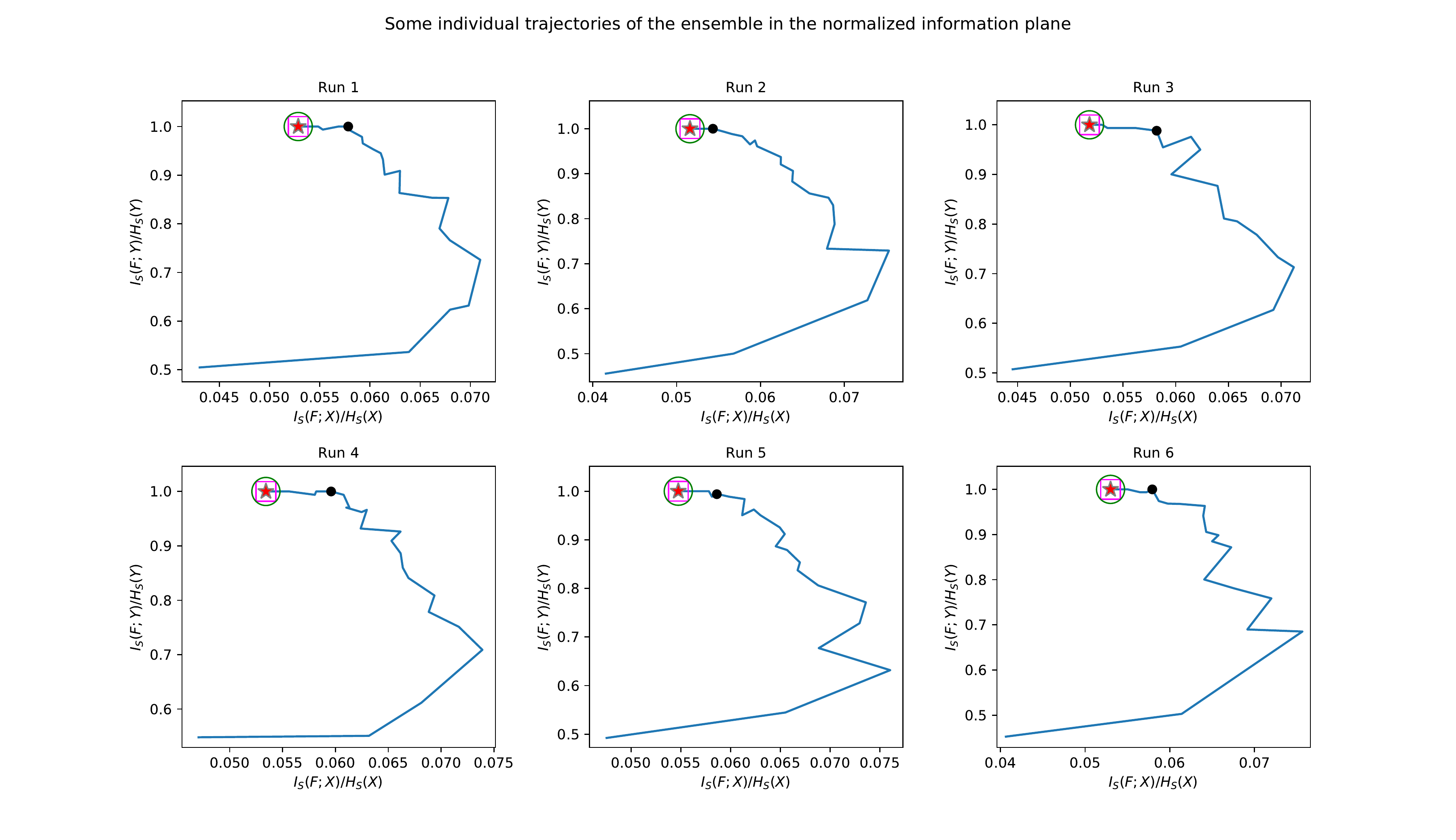}}
\subfigure{\includegraphics[width=0.49\textwidth]{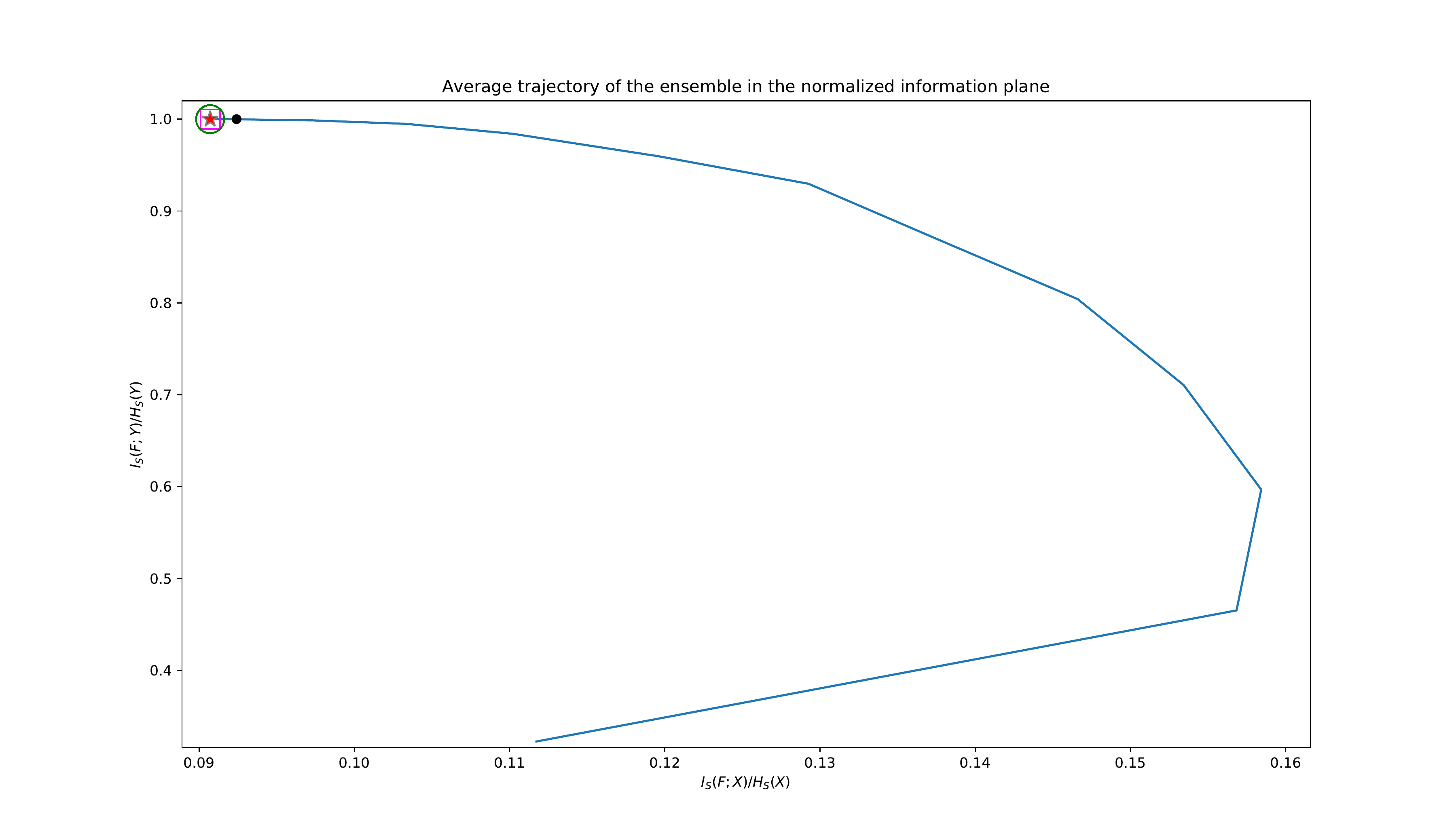}}
\subfigure{\includegraphics[width=0.49\textwidth]{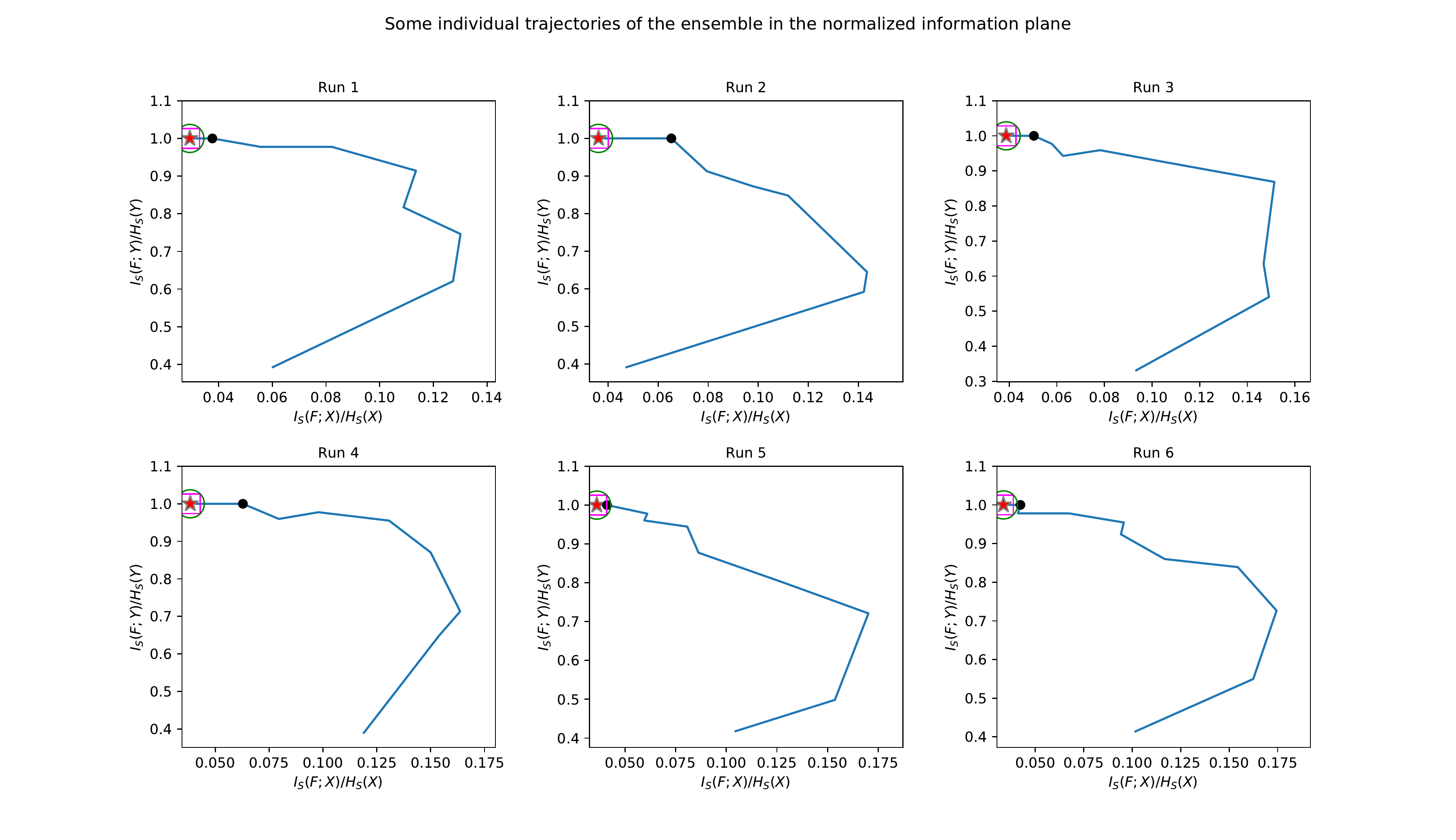}} 
\caption{Trajectory of the boosting ensemble on the normalized information plane as the rounds of boosting progress. We highlight the point on which the training error is first minimized (full black circle), the point on which the test error is first minimized (magenta square), the point on which the margins are first maximized (hollow green circle) and the lossless maximal compression point (red star). From TOP to BOTTOM (dataset): \emph{waveform}, \emph{krvskp}, \emph{musk2}, \emph{credit}; LEFT: Average trajectory across 100 runs; RIGHT: Some random individual trajectories. Notice that on all datasets, boosting traces a trajectory that leads to the LMC point, the latter coinciding with the margin maximization point and --on average-- with the  test error minimization point. The results are qualitatively consistent across individual runs: trajectories can vary significantly yet the aforementioned observations hold.}\label{fig:trajectories_real}
\end{figure}

\begin{figure}
\centering
\subfigure{\includegraphics[width=0.49\textwidth]{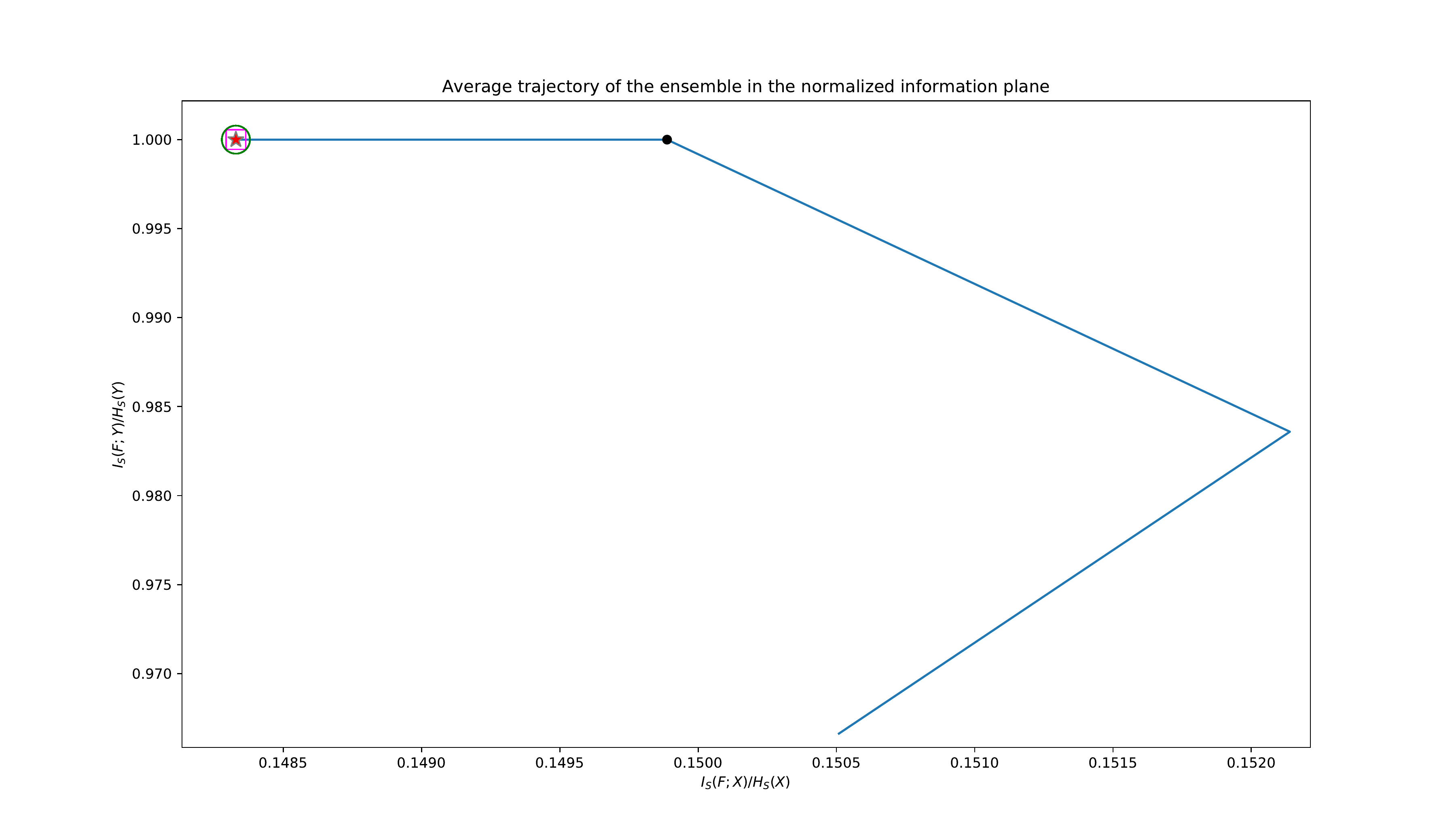}}
\subfigure{\includegraphics[width=0.49\textwidth]{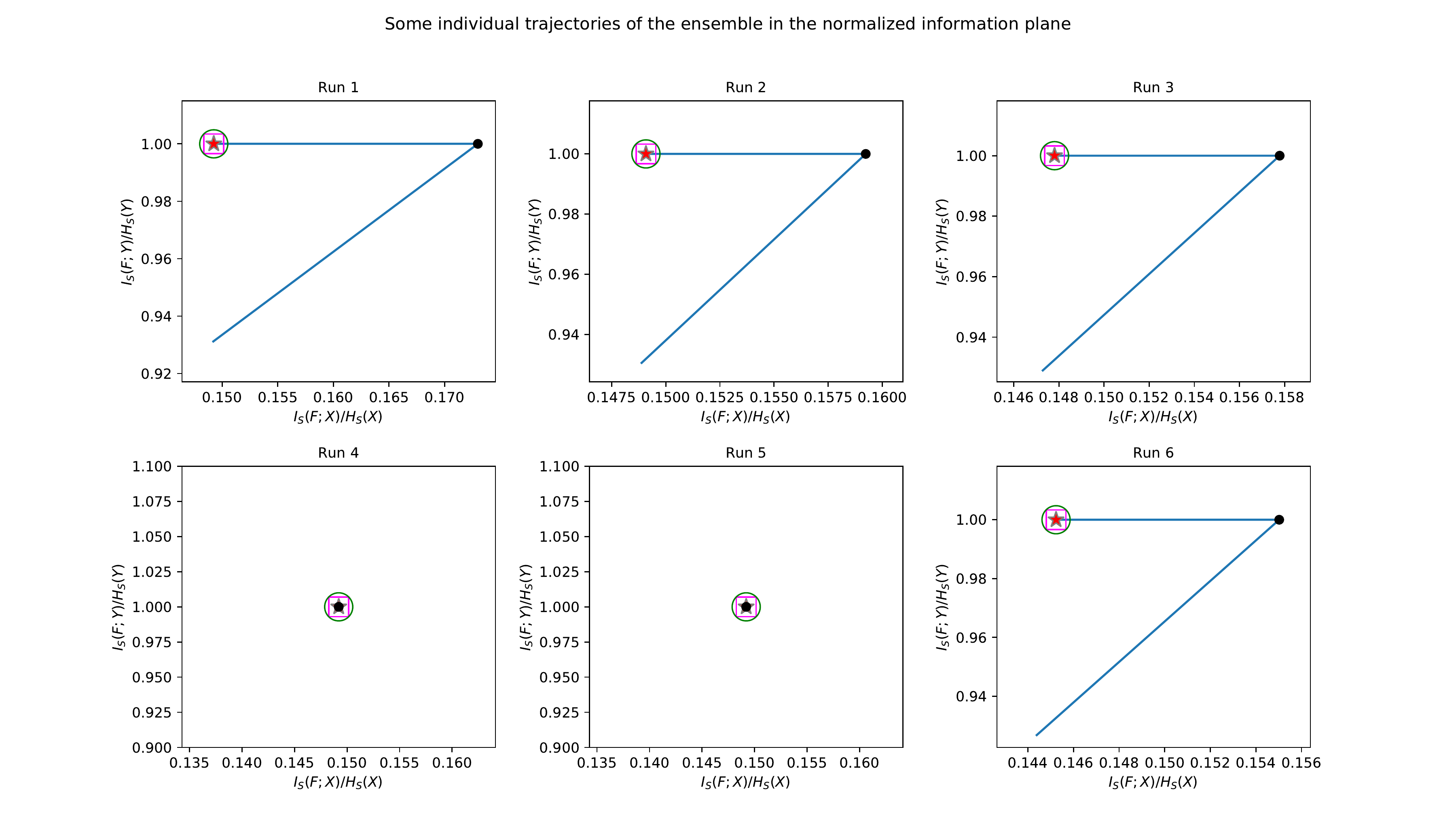}}
\subfigure{\includegraphics[width=0.49\textwidth]{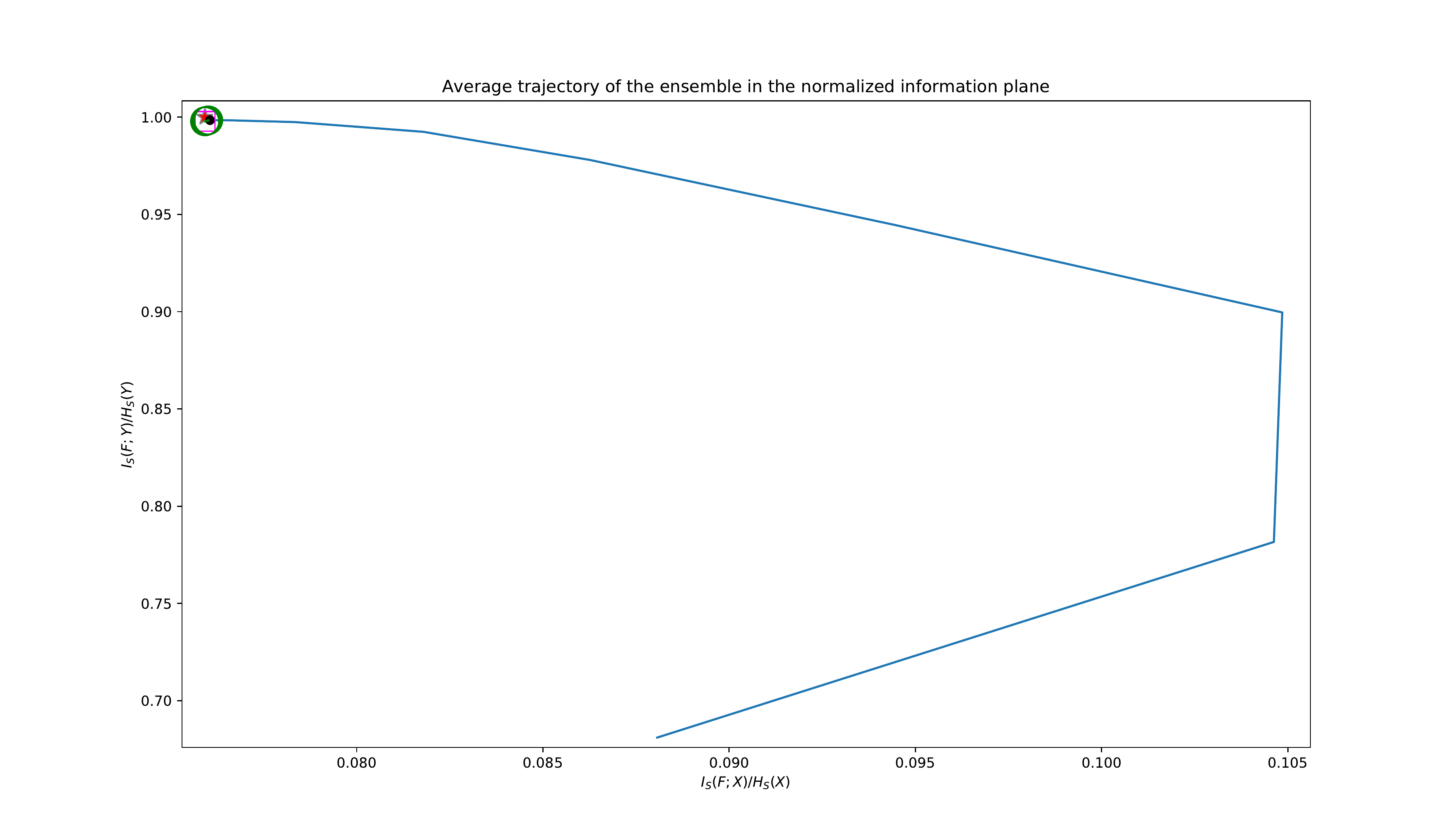}}
\subfigure{\includegraphics[width=0.49\textwidth]{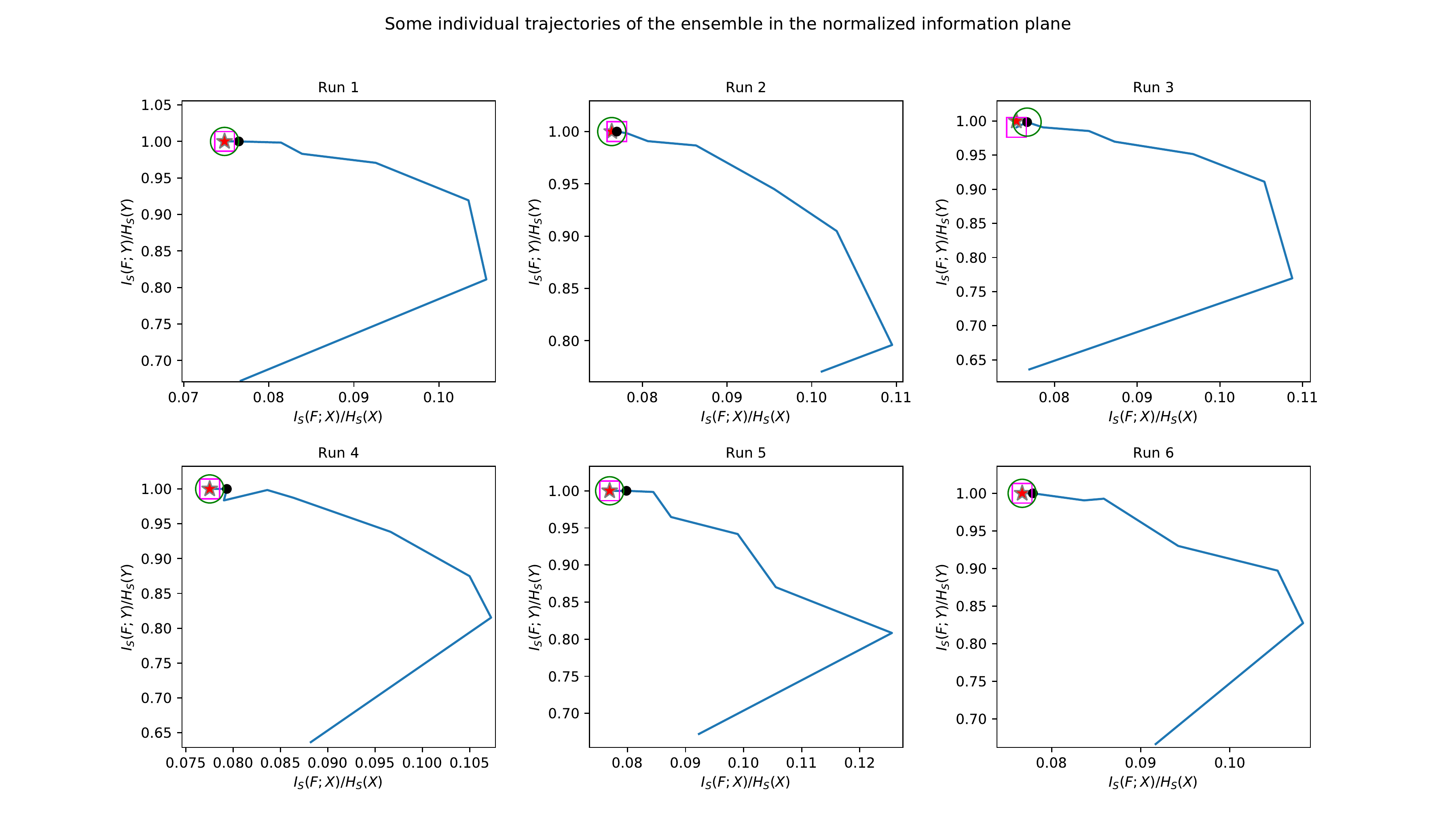}}
\subfigure{\includegraphics[width=0.49\textwidth]{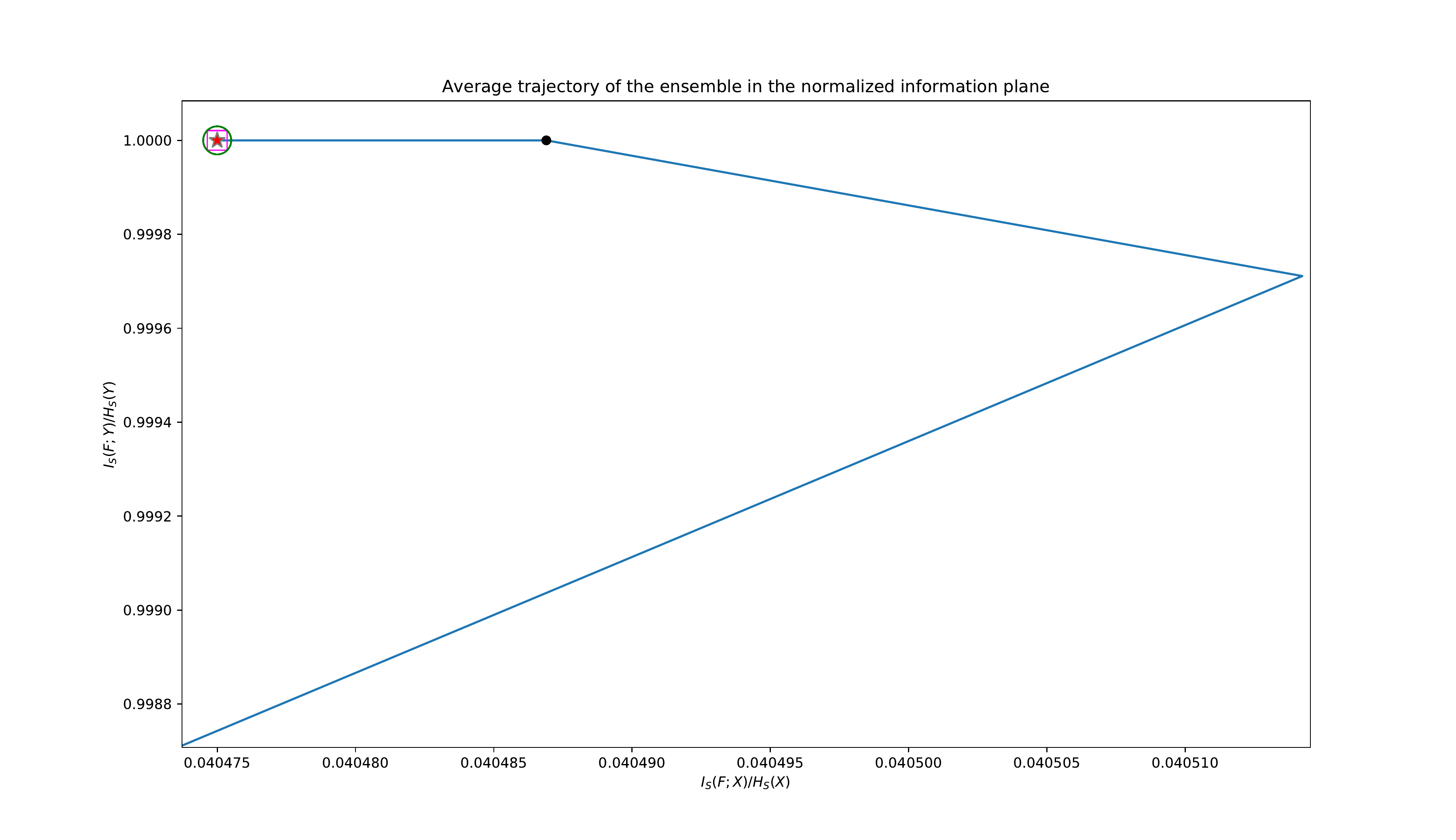}}
\subfigure{\includegraphics[width=0.49\textwidth]{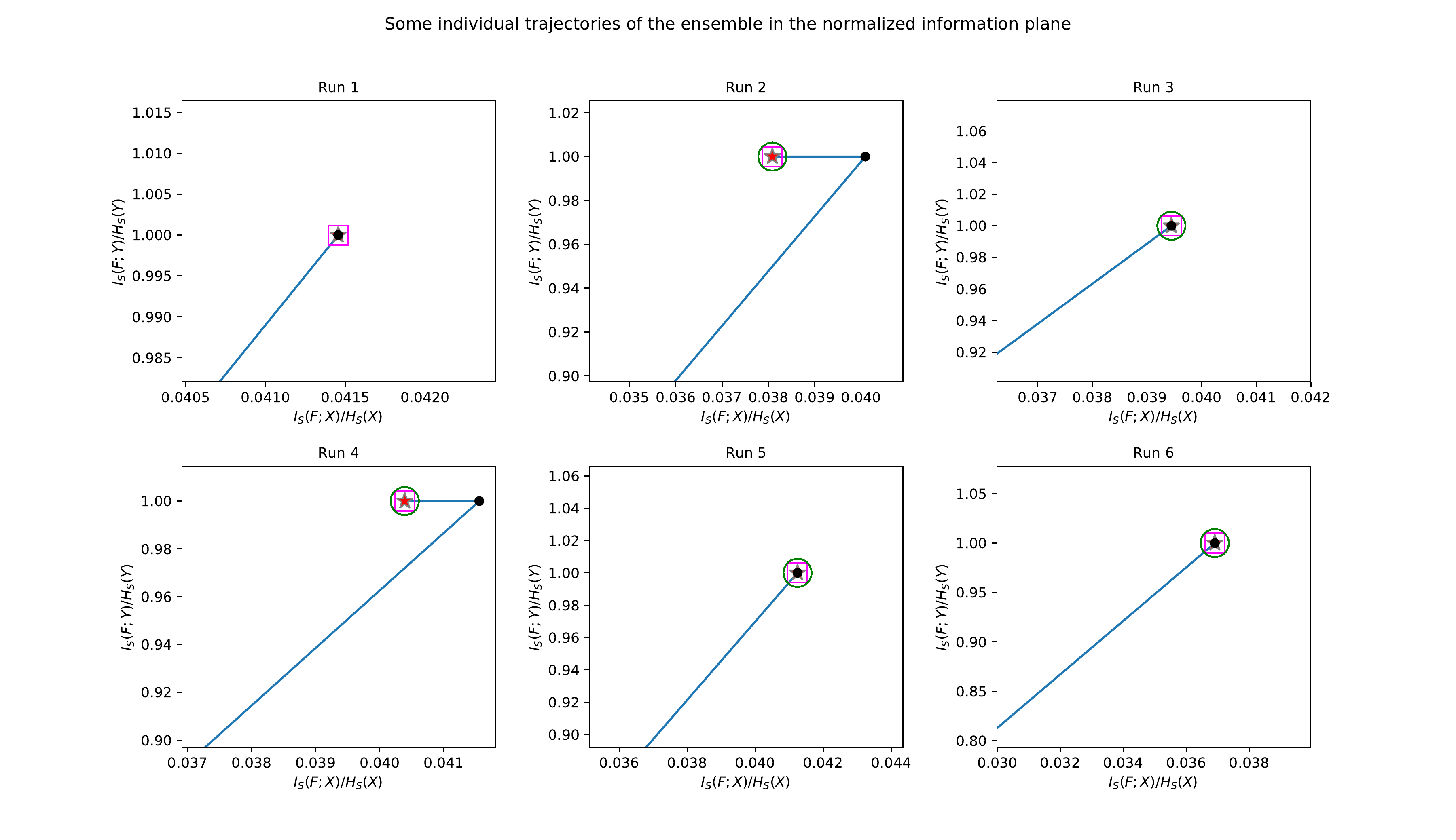}}
\subfigure{\includegraphics[width=0.49\textwidth]{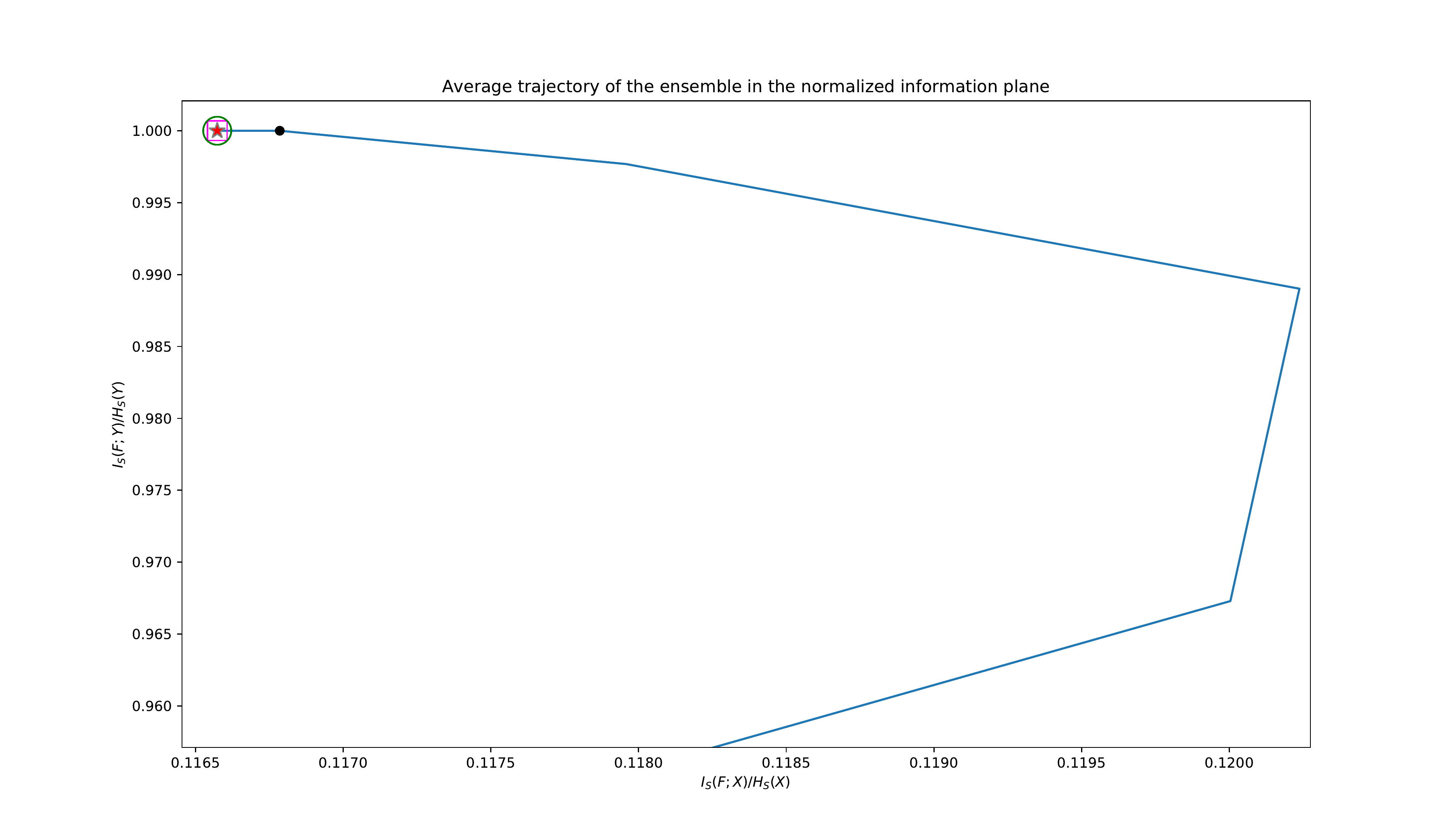}}
\subfigure{\includegraphics[width=0.49\textwidth]{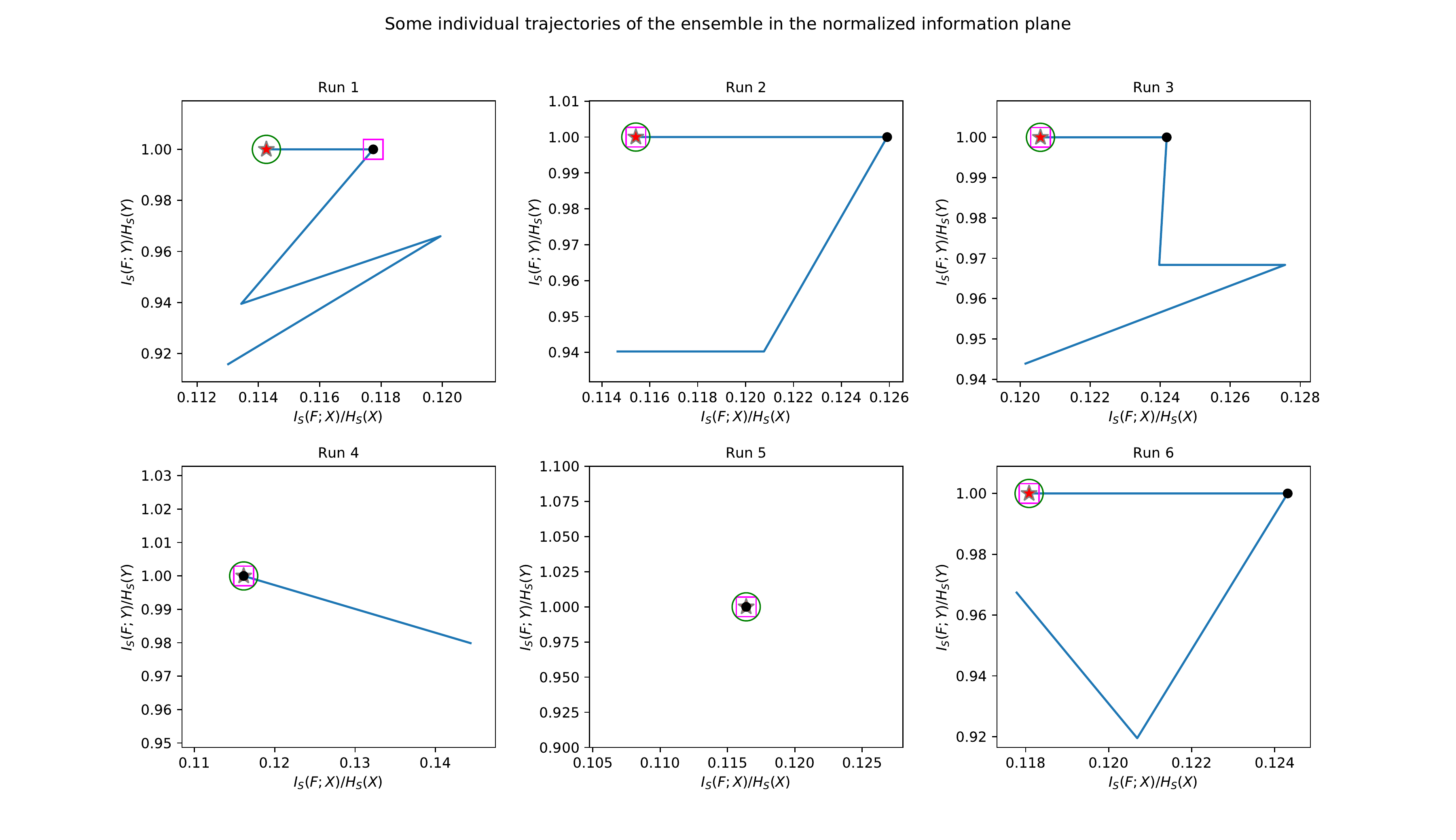}} 
\caption{Trajectory of the boosting ensemble on the normalized information plane as the rounds of boosting progress. We highlight the point on which the training error is first minimized (full black circle), the point on which the test error is first minimized (magenta square), the point on which the margins are first maximized (hollow green circle) and the lossless maximal compression point (red star). From TOP to BOTTOM (dataset): \emph{sonar}, \emph{splice}, \emph{semeion}, \emph{wdbc}; LEFT: Average trajectory across 100 runs; RIGHT: Some random individual trajectories.}\label{fig:trajectories_real2}
\end{figure}

\begin{figure}
\centering
\subfigure{\includegraphics[width=0.49\textwidth]{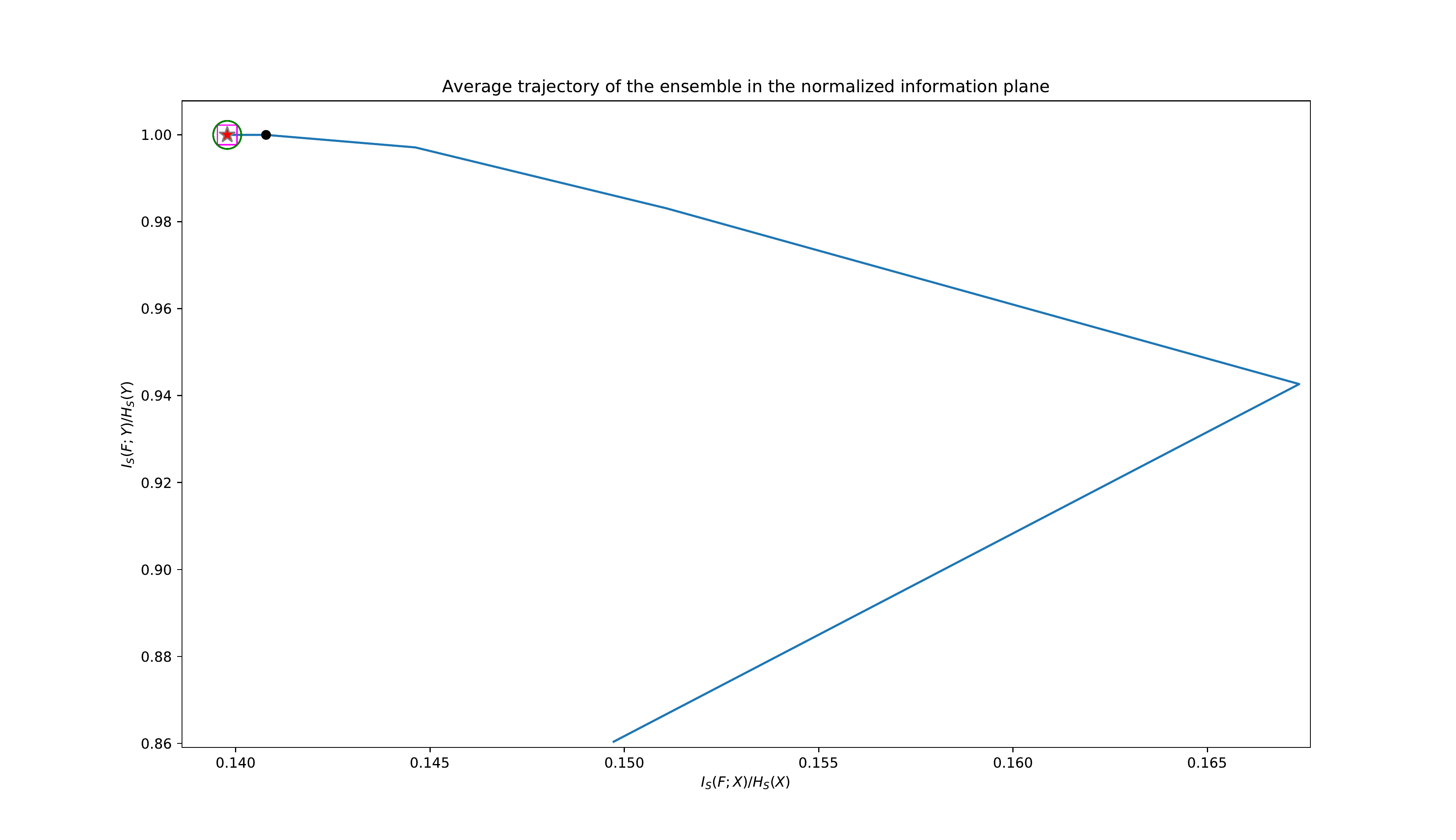}}
\subfigure{\includegraphics[width=0.49\textwidth]{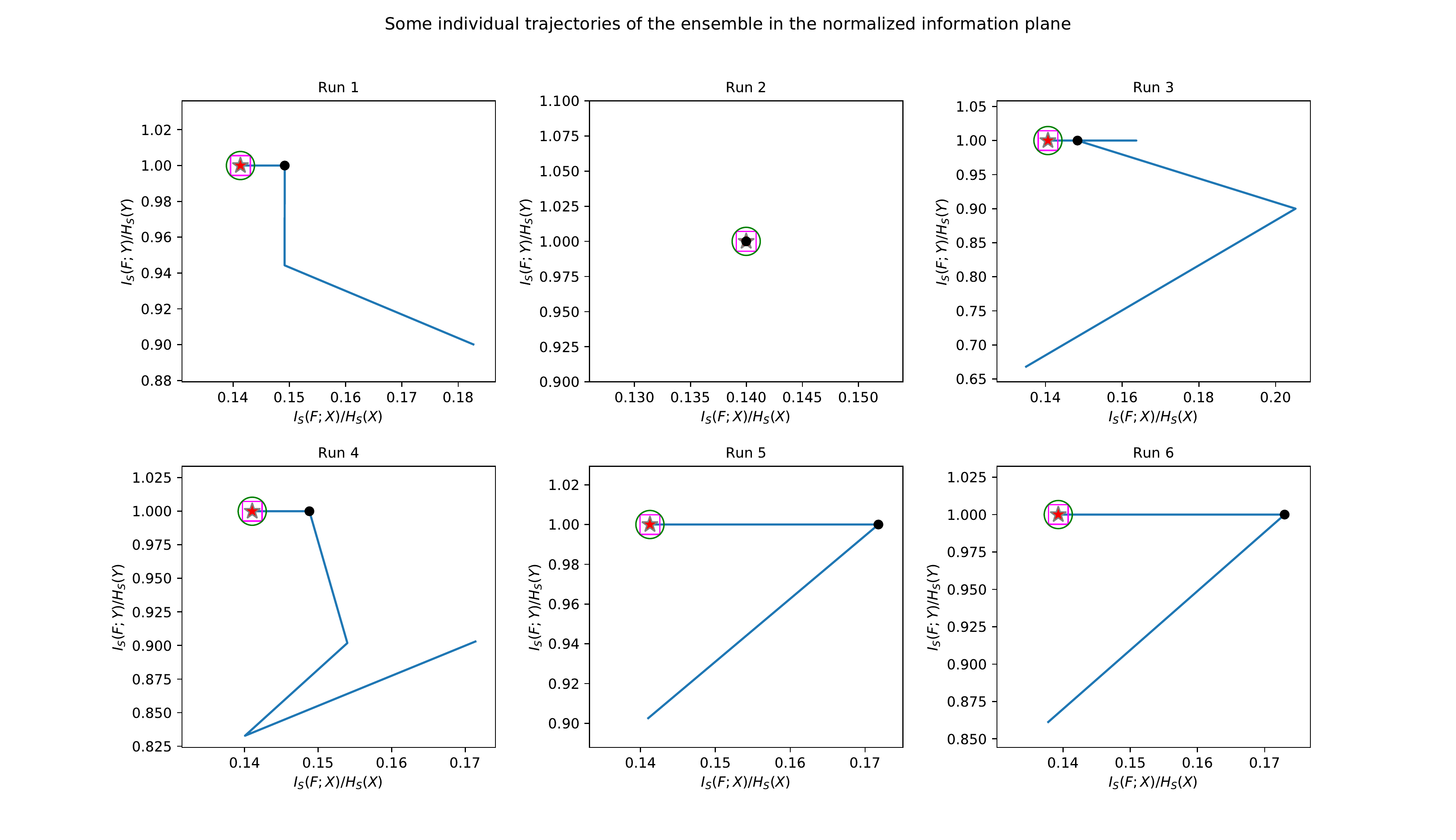}}
\subfigure{\includegraphics[width=0.49\textwidth]{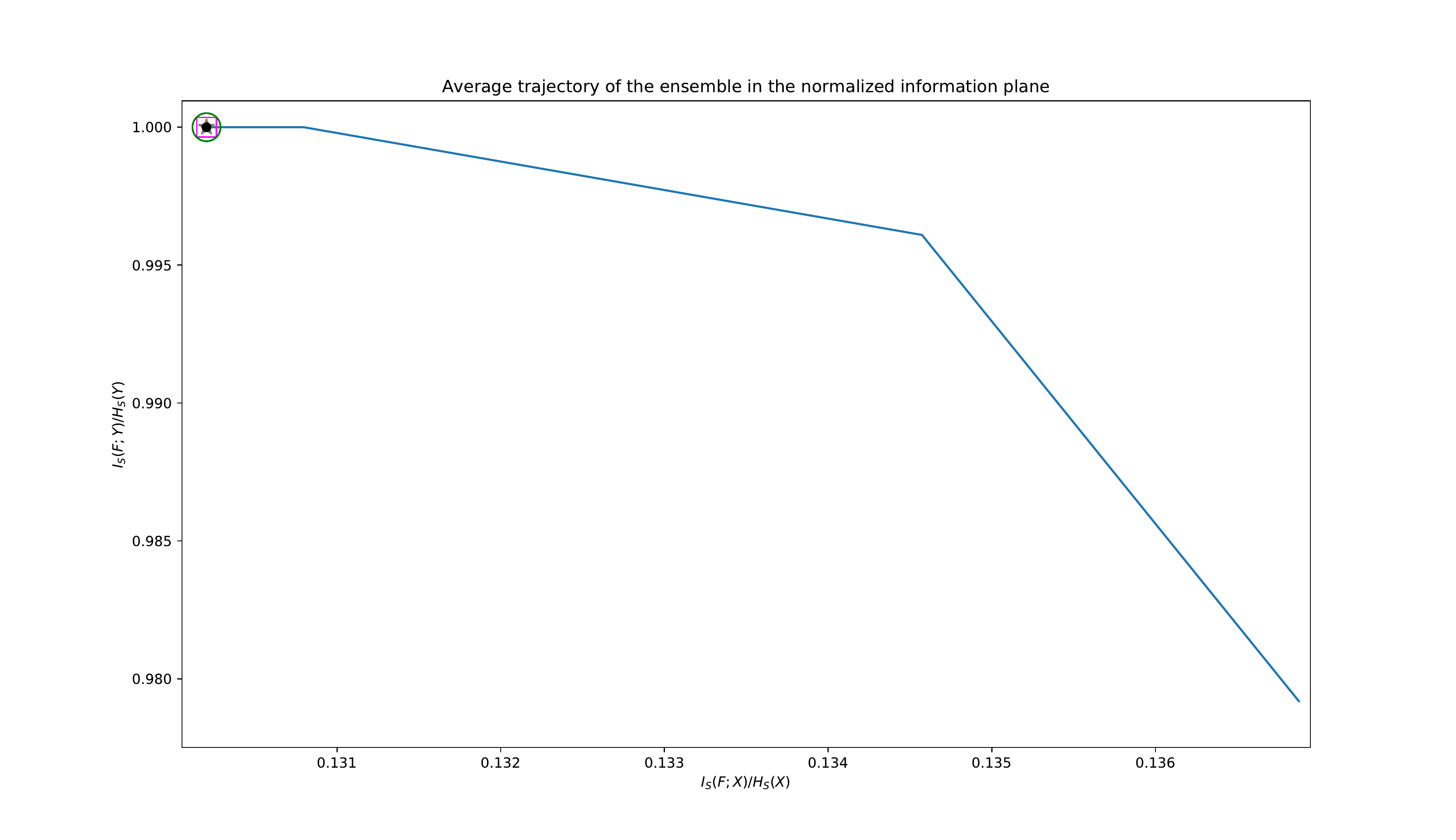}}
\subfigure{\includegraphics[width=0.49\textwidth]{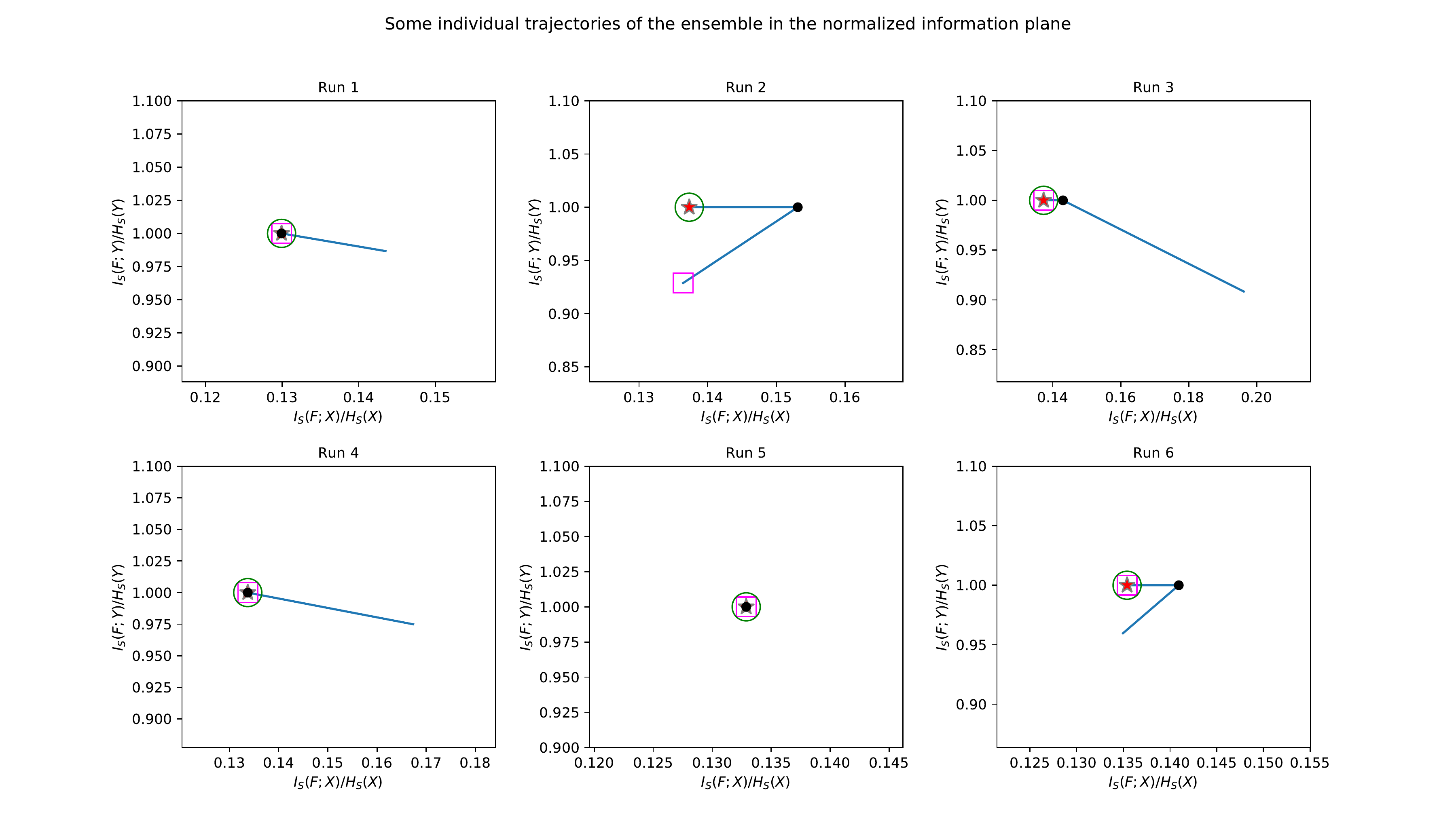}}
\subfigure{\includegraphics[width=0.49\textwidth]{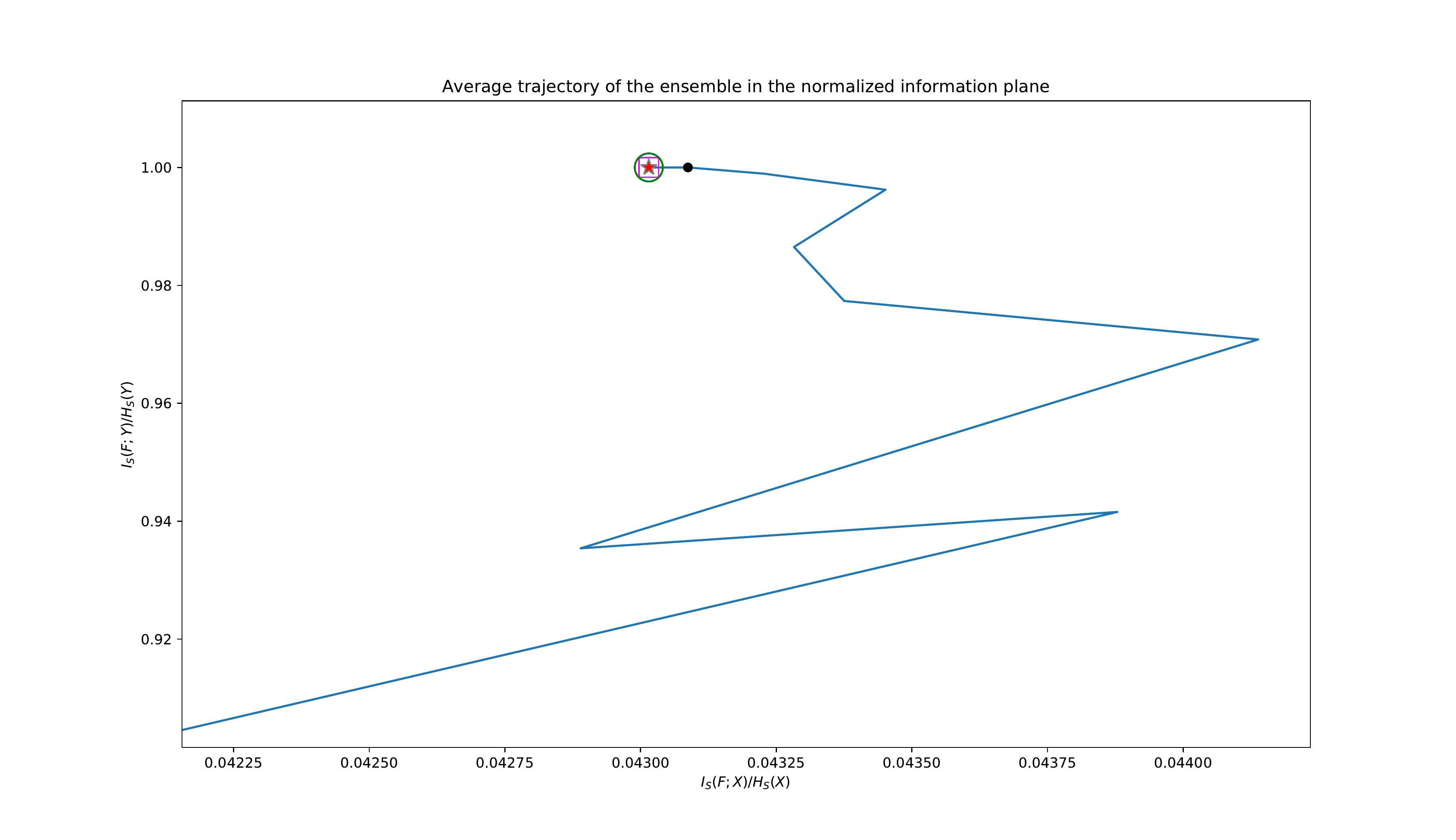}}
\subfigure{\includegraphics[width=0.49\textwidth]{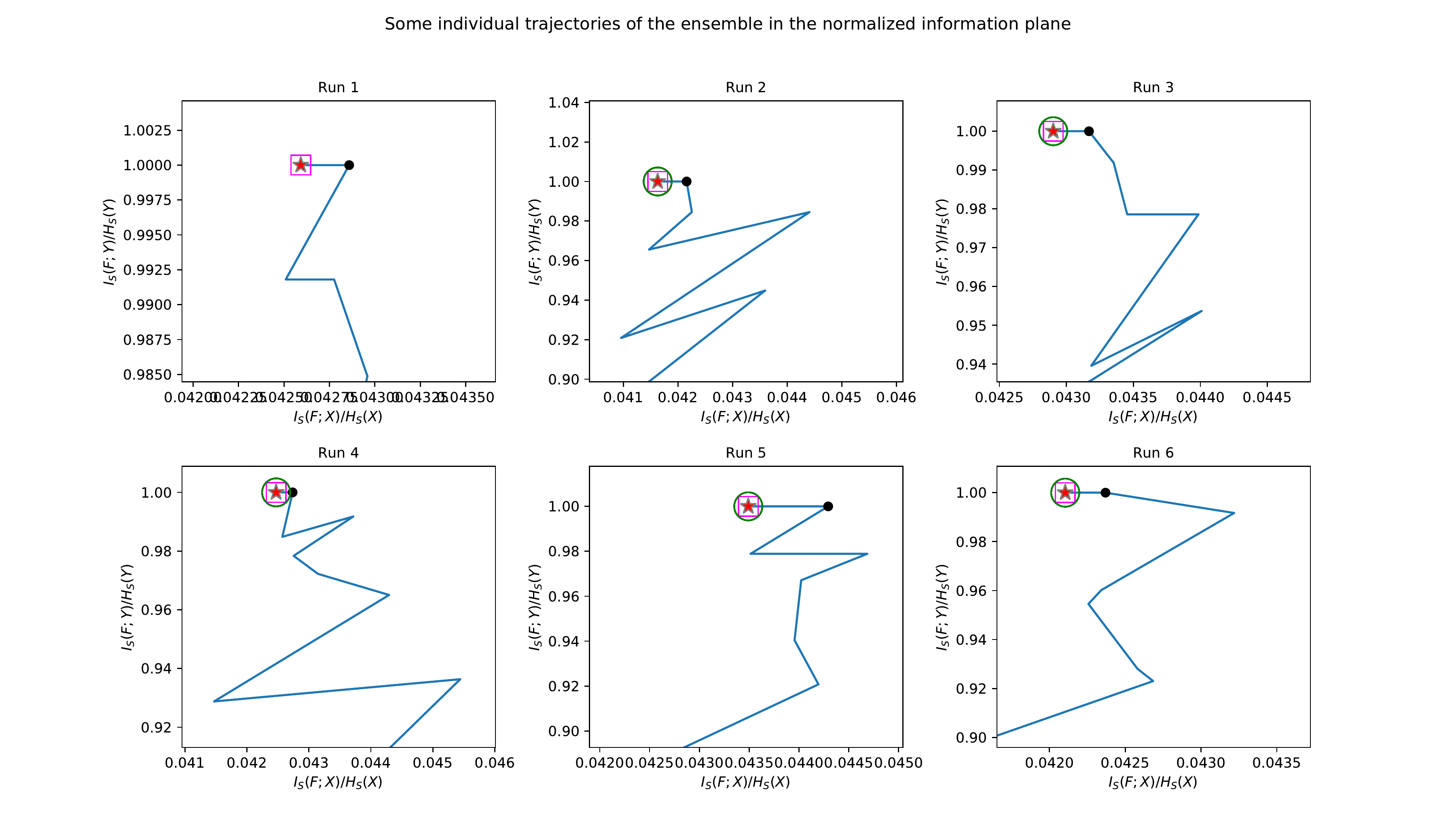}}
\subfigure{\includegraphics[width=0.49\textwidth]{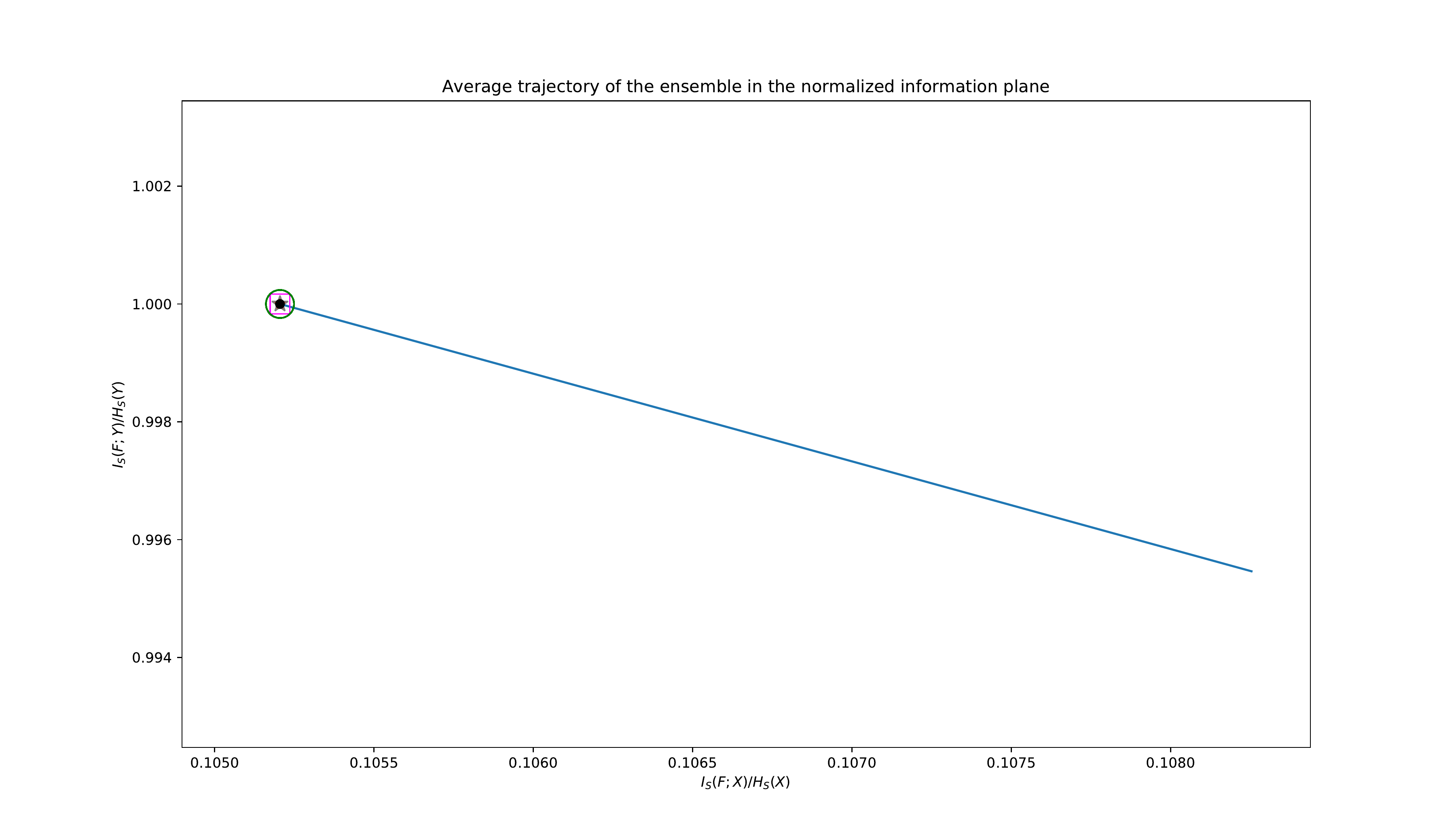}}
\subfigure{\includegraphics[width=0.49\textwidth]{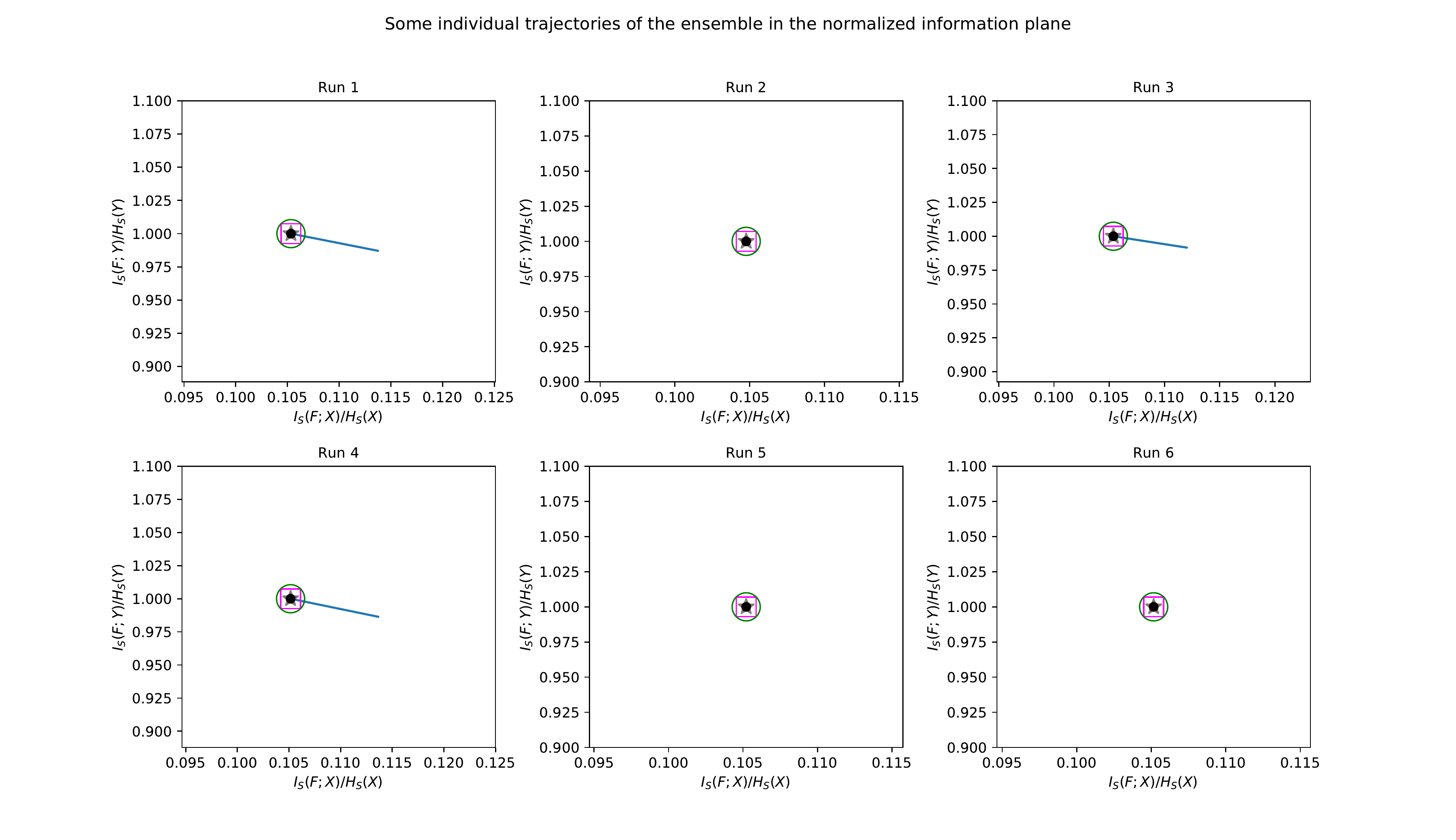}} 
\caption{Trajectory of the boosting ensemble on the normalized information plane as the rounds of boosting progress. We highlight the point on which the training error is first minimized (full black circle), the point on which the test error is first minimized (magenta square), the point on which the margins are first maximized (hollow green circle) and the lossless maximal compression point (red star). From TOP to BOTTOM (dataset): \emph{heart}, \emph{congress}, \emph{landsat}, \emph{mushroom}; LEFT: Average trajectory across 100 runs; RIGHT: Some random individual trajectories.}\label{fig:trajectories_real3}
\end{figure}

\begin{figure}
\centering
\subfigure{\includegraphics[width=0.49\textwidth]{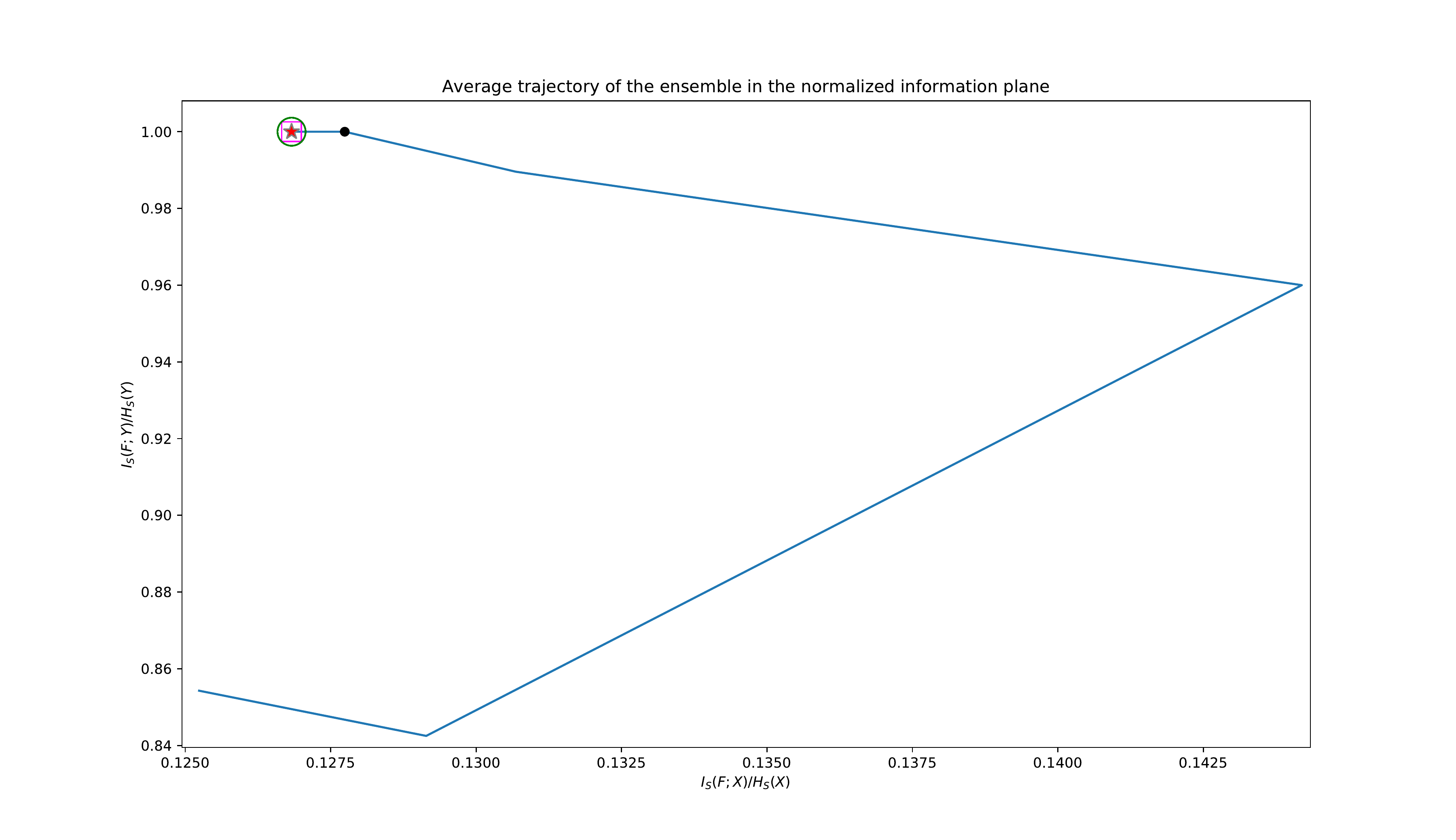}}
\subfigure{\includegraphics[width=0.49\textwidth]{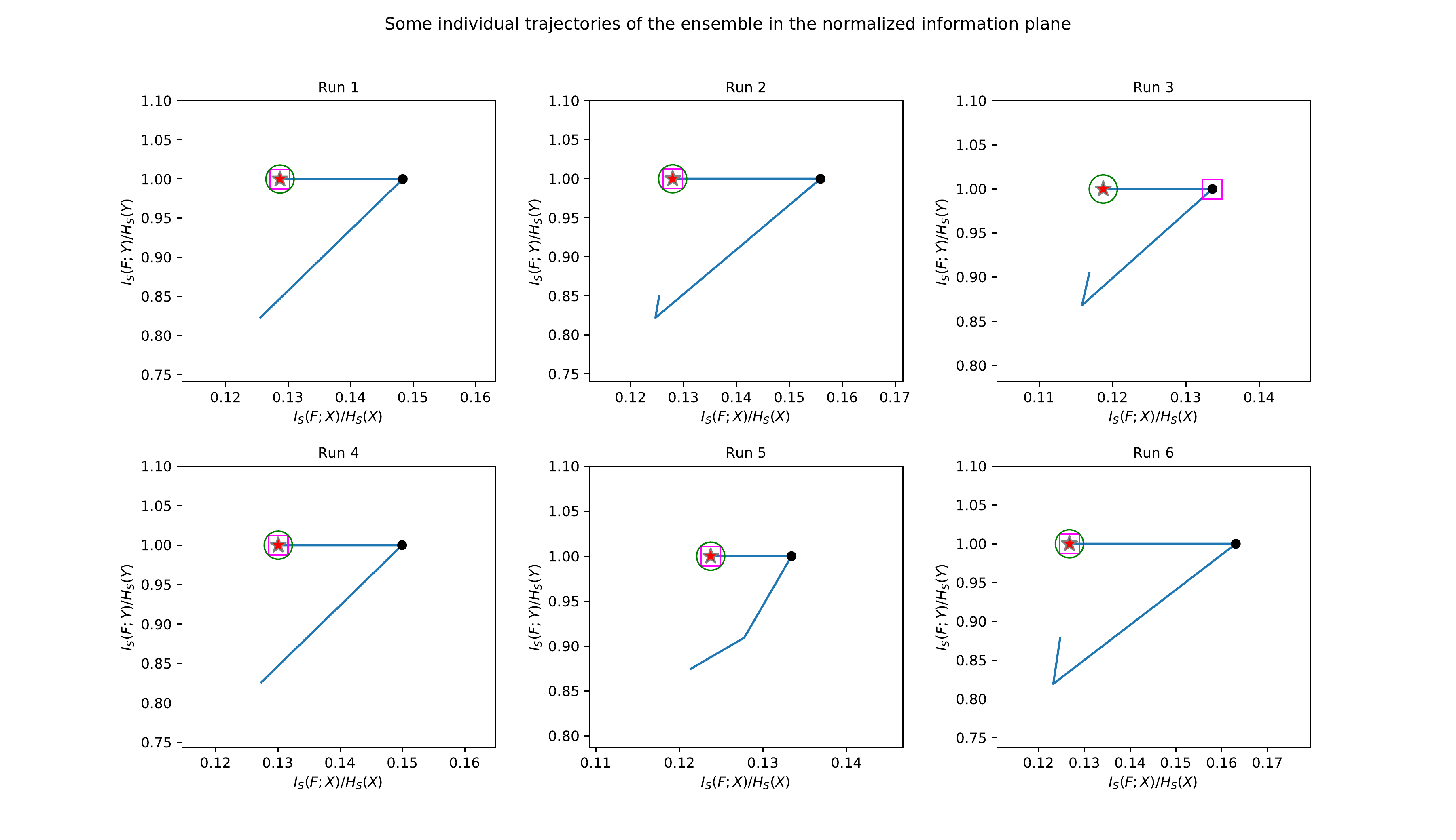}}
\subfigure{\includegraphics[width=0.49\textwidth]{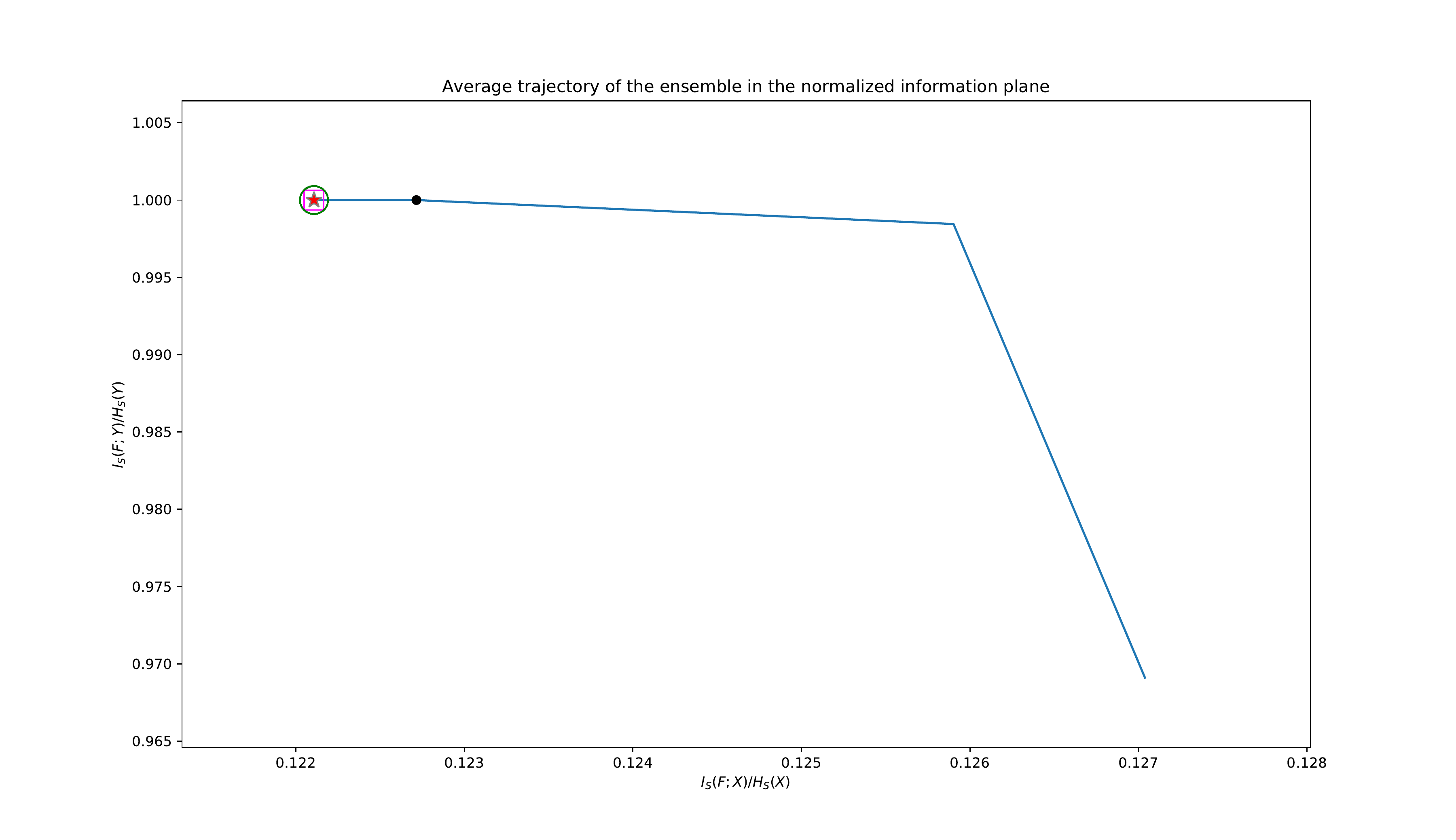}}
\subfigure{\includegraphics[width=0.49\textwidth]{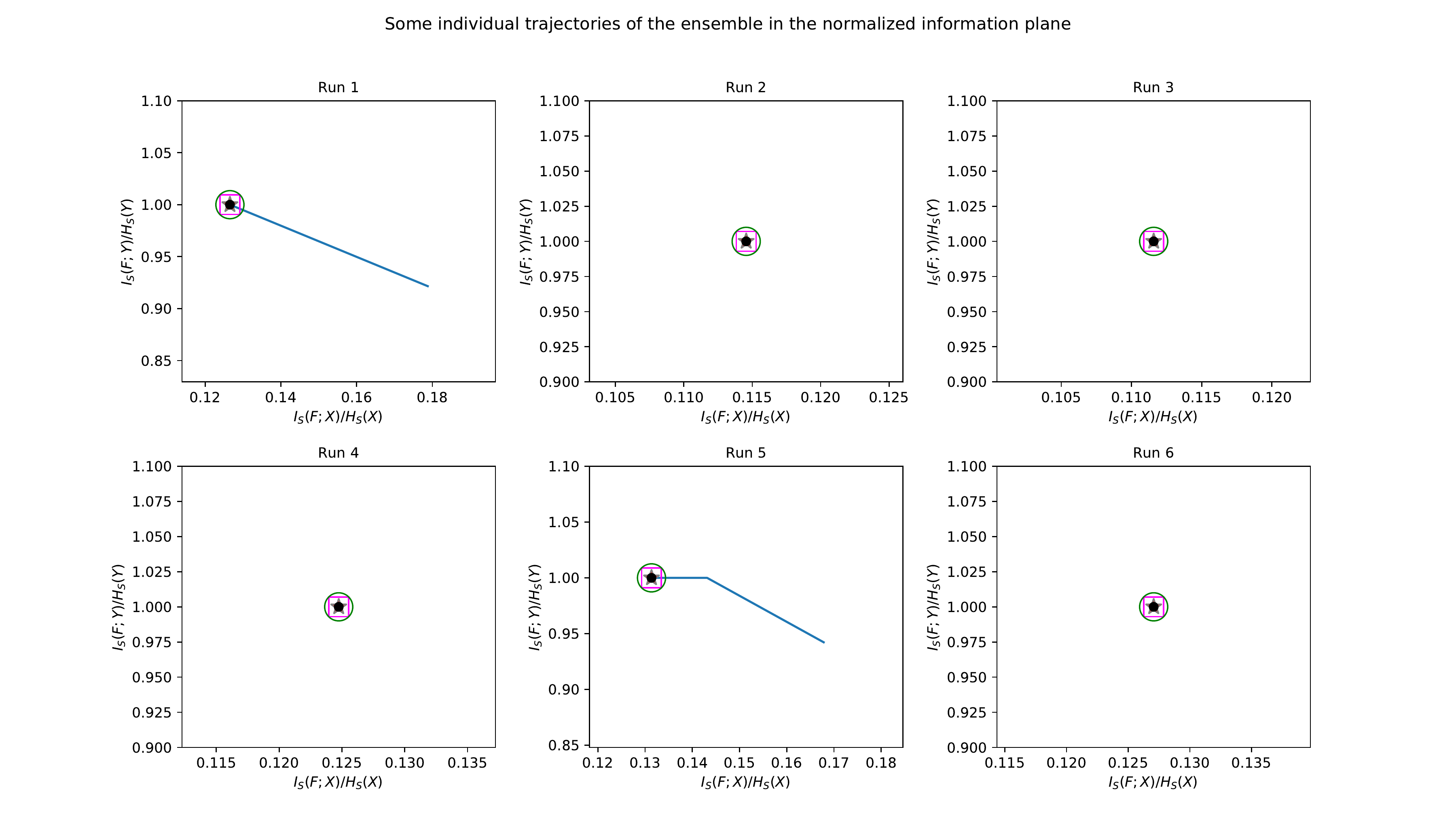}}
\caption{Trajectory of the boosting ensemble on the normalized information plane as the rounds of boosting progress. We highlight the point on which the training error is first minimized (full black circle), the point on which the test error is first minimized (magenta square), the point on which the margins are first maximized (hollow green circle) and the lossless maximal compression point (red star). From TOP to BOTTOM (dataset): \emph{ionosphere}, \emph{parkinsons}; LEFT: Average trajectory across 100 runs; RIGHT: Some random individual trajectories.}\label{fig:trajectories_real4}
\end{figure}

Let us now summarize our observations from Figure \ref{fig:trajectories_real} \& the figures of Section C of the Supplementary Material:\\
\textbf{Boosting leads to lossless maximal compression:}
In all datasets, the boosting ensemble traces a trajectory on the information plane that leads to the LMC point and once it reaches it in never escapes.\\
\textbf{Lossless maximal compression coincides with margin maximization:} In all datasets the image on the information plane of the models that minimize the margin coincides with the LMC point.\\
\textbf{Lossless maximal compression coincides with maximal generalization:} The point of the ensemble's trajectory corresponding to the minimal test error coincides --on average-- with the LMC point on the information plane (and so does the margin maximization point). In other words, LMCs correspond to the models exhibiting --on average-- the best generalization behaviour.\\
\textbf{Average trajectory shape:} After the training error is minimized, the test error can be further decreased by training for more rounds. This is a known result in boosting, explained via margin theory. Here we give an information-theoretic interpretation. Training until training error minimization, amounts to achieving losslessness. Subsequent rounds result in travelling along the line of maximal $\mutualInformationSample(F;X)$ on the information plane, towards the LMC point. This compresses the model $f$ (relieves it of remaining information from $X$ irrelevant for predicting $Y$), decreasing its effective complexity\footnote{Holds for average trajectories. Single runs include steps that both increase $\mutualInformationSample(F;X)$ \& decrease $\mutualInformationSample(F;Y)$.}.\\ 
\textbf{Training in boosting consists of 2 (typically distinct) phases:} The results suggest the presence of 2 distinct phases during training under gradient boosting. A similar behaviour was observed in~\cite{shwartz2017opening} for the trajectories of the representations learned by DNNs. Following the terminology of~\cite{shwartz2017opening}, these are the \emph{empirical risk minimization (ERM) phase}, when $\mutualInformationSample(F;Y)$ increases (the model better fits the training data) but typically so does $\mutualInformationSample(F;X)$ (the model uses more information from $X$) and the \emph{compression phase}, when $\mutualInformationSample(F;X)$ decreases (the model uses increasingly less information from $X$, reducing its effective complexity), without decreasing $\mutualInformationSample(F;Y)$. We can view the ERM phase as decreasing the bias of the model while not decreasing its variance and the compression phase as reducing variance while not increasing bias. The ERM phase is usually much shorter than the compression phase, as is the case with DNNs~\cite{shwartz2017opening}. Although typically we do observe these 2 phases as distinct in the average trajectories, they need not be, as was also observed in subsequent studies in DNNs~\cite{xu2018training}. Trajectories of individual runs, are not as smooth as the average trend; we can even observe steps that increase both bias \& variance. However, the 2 phases still appear to be distinct: once losslessness is achieved (ERM phase terminates), it is maintained and pure compression begins.\\ 
\textbf{Early stopping does not improve generalization in gradient boosting:} As long as losslessness can be achieved, additional boosting rounds do not hurt generalization. Once the model reaches the LMC point on the information plane, it never escapes it. Subsequent iterations neither increase the training nor the test error. This suggests that early stopping with boosting is unnecessary for improving generalization and agrees with recent observations~\cite{wyner2015explaining}. General margin losses minimized via stochastic gradient descent (SGD) also exhibit similar behaviour~\cite{soudry2017implicit}.\\
\textbf{Consistency across datasets, hyperparameter \& discretization settings:}
The aforementioned observations hold across different datasets and hyperparameter settings. Section C of the Supplementary Material contains more results supporting this claim. They also hold if we change the number of bins used to discretize the features (provided the dataset remains noiseless) or the scores (provided they are$\geq 2$).\\
\textbf{Margin maximization as a built-in regularization mechanism:} Additional regularization techniques like subsampling or shrinkage are not the main reason why boosting regularizes. Their contribution is small compared to the algorithms' built-in regularization mechanism: margin maximization, which as we saw amounts to lossless maximal compression of the training dataset. This is another similarity shared with DNNs trained with SGD that achieve good generalization by tracing a similar trajectory on the information plane~\cite{shwartz2017opening}, and additional regularization control (e.g. dropout or batch normalization) is beneficial, but not the main contributor to their good generalization~\cite{zhang2016understanding, shwartz2017opening, kawaguchi2017generalization}.

\section{Discussion}
We characterized from an information theoretic perspective, models trained on a given training set w.r.t. the information they capture from it. Under this light, we identified an ideal model trained on a given dataset as its lossless maximal compressor (LMC): one capturing all the information from the features relevant for predicting the target and no more. We then established that an LMC is --in the case of classification-- equivalent to a margin maximizer of the dataset (provided it is noiseless, i.e. consistently labelled). The existence of margin-based bounds on the generalization error implies that margin maximization, hence lossless maximal compression, is beneficial to generalization.

Our experiments on gradient boosting, demonstrate that indeed, margin maximization amounts to lossless maximal compression on noiseless data. The evolution of the model constructed by boosting, traces a trajectory on the information plane that leads to the point of lossless maximal compression which also coincides with the point of margin maximization and the point on average exhibiting the best generalization. In agreement with recent studies on boosting, we observe that early stopping is unnecessary for improving generalization~\cite{wyner2015explaining} and identify interesting similarities between how training progresses in DNNs and in gradient boosting in terms of the trajectory they trace on the information plane~\cite{shwartz2017opening}. All observations persist across a wide range of datasets and hyperparameter configurations.

This work gives an information-theoretic interpretation of margin maximization and provides us with a principled way to define model complexity for the purposes of generalization, thus shedding more light on the success of methods like gradient boosting. It also opens various directions for future work. For instance, exploring how these concepts can be applied in model selection or to inform learning algorithm design to more efficiently traverse the information plane to reach the LMC point. It would also be of interest to identify the analogue of the LMC in learning tasks other than classification, like ranking or regression.

\paragraph{\textbf{\bf{Acknowledgements}}} This project was partially supported by the EPSRC LAMBDA [EP/N035127/1] \& Anyscale Apps [EP/L000725/1] project grants and the EU Horizon 2020 research \& innovation programme [grant No 758892, ExoAI]. NN acknowledges the support of the EPSRC Doctoral Prize Fellowship at the University of Manchester and the NVIDIA Corporation's GPU grant. The authors thank Konstantinos Sechidis, Konstantinos Papangelou \& Ingo Waldmann for their useful comments and suggestions.

\bibliographystyle{spmpsci}
\bibliography{bibliography_MMLMC}

\newpage
\section{Supplementary Material}
\subsection*{A. Proofs}

In this section we shall prove Lemma \ref{noiselessIFFZeroRiskLemma}, Lemma \ref{losslessIFFARiskMinimiserUpToInvertibleTransformationLemma} and Theorem \ref{thm:margin_max_lmc}. Rather than proving these results directly we shall instead prove generalisations to an arbitrary finitely supported probability measure $\probMeasure$ (Lemma \ref{losslessIFFARiskMinimiserUpToInvertibleTransformationLemmaP}, Lemma \ref{noiselessIFFZeroRiskLemmaP} and Theorem \ref{thm:margin_max_lmcP}). The sample based (i.e. dataset based) results used in the main paper correspond to the special case in which the probability measure $\probMeasure$ is the empirical measure $\empiricalMeasure$.

\begin{definition}
\label{def:noiselessnessP}
A finitely supported probability distribution $\probMeasure$ is noiseless if and only if $\entropyProbMeasure(Y|X)=0$. 
\end{definition}

Definition \ref{def:noiselessnessP} generalises Definition \ref{def:noiselessness} which corresponds to the special case in which $\probMeasure$ is the empirical measure $\empiricalMeasure$. The proofs of the following results require the following elementary lemma.

\begin{lemma}\label{entropyPhiPropertiesLemma} Let the function  $\entropyPhi:[0,1]\rightarrow \R$ defined by $\entropyPhi(0)=0$ and $\entropyPhi(z) = z\cdot \log(1/z)$ for $z \in (0,1]$. Then we have $\entropyPhi(z) \geq 0$ for all $z \in [0,1]$ with equality if and only if $z \in \{0,1\}$.
\end{lemma}

We shall now prove Lemma \ref{noiselessIFFZeroRiskLemmaP} which generalises Lemma \ref{noiselessIFFZeroRiskLemma}.

\begin{lemma}\label{noiselessIFFZeroRiskLemmaP}
A finitely supported probability distribution $\probMeasure$ is noiseless if and only if $\risk_{\probMeasure}(\riskMinimiserProb)  = 0$. 
\end{lemma}
\begin{proof} We can write out the conditional entropy $\entropyProbMeasure(Y|X)$ in terms of $\entropyPhi$ as follows,
\begin{align*}
\entropyProbMeasure(Y|X) &= \sum_{x\in \XSupp } \probMeasure(X=x)\left(- \sum_{y \in \Y}\probMeasure(Y=y|X=x)\log\left(\probMeasure(Y=y|X=x) \right)\right)\\
&= \sum_{x\in \XSupp } \probMeasure(X=x)\left( \sum_{y \in \Y}\entropyPhi\left(\probMeasure(Y=y|X=x) \right)\right).
\end{align*}
Since $\entropyPhi$ is non-negative it follows that $\entropyProbMeasure(Y|X)= 0$ if and only if for each $x \in \XSupp$ we have $\entropyPhi\left(\probMeasure(Y=y|X=x) \right)=0$ which is the case if and only if $\probMeasure(Y=y|X=x) \in \{0,1\}$.

Now suppose  that $\entropyProbMeasure(Y|X)= 0$ so for each $x \in \XSupp$, $y \in \Y$, $\probMeasure(Y=y|X=x) \in \{0,1\}$. Then we can define $\riskMinimiserProb:\X \rightarrow [-1,1]$ so that $\probMeasure\left(Y=\riskMinimiserProb(x) |X=x\right)=1$, so $\probMeasure\left( \sign(\riskMinimiserProb(x))\neq Y|X=x\right)=0$. Thus for each $x \in \XSupp$ if 
\begin{align*}
\risk_{\probMeasure}\left(\riskMinimiserProb\right) &= \probMeasure\left[ \sign(\riskMinimiserProb(X))\neq Y\right] = \sum_{x \in \XSupp} \probMeasure\left(X=x\right) \cdot \probMeasure\left( \sign(\riskMinimiserProb(x))\neq Y|X=x\right)=0.
\end{align*}
One the other hand, if $\risk_{\probMeasure}\left(\riskMinimiserProb\right)= 0$ then we must have $\probMeasure\left(Y=y|X=x\right)=1$ if $y = \sign\left(\riskMinimiserProb(x)\right)$ and $\probMeasure\left(Y=y|X=x\right)=0$ otherwise. Thus, $\probMeasure(Y=y|X=x) \in \{0,1\}$ for each $x \in \XSupp$, $y \in \Y$ and so $\entropyProbMeasure(Y|X)= 0$.
\end{proof}

Definition \ref{def:losslessnessP} generalises Definition \ref{def:losslessness}.

\begin{definition}\label{def:losslessnessP}
A model $f$ is lossless with respect to $\probMeasure$ if and only if  $\mutualInformationProbMeasure(F;Y) = \mutualInformationProbMeasure(Y;X)$.
\end{definition}

Lemma \ref{losslessIFFARiskMinimiserUpToInvertibleTransformationLemmaP} generalises Lemma \ref{losslessIFFARiskMinimiserUpToInvertibleTransformationLemma}.

\begin{lemma}\label{losslessIFFARiskMinimiserUpToInvertibleTransformationLemmaP}
Suppose $\probMeasure$ is noiseless. A model $f$ is lossless with respect to $\probMeasure$ if and only if there exists an invertible transformation $g:[-1,1]\rightarrow \R$ such  $\risk_{\probMeasure}(g \circ f) =\risk_{\probMeasure}(\riskMinimiserProb)$. 
\end{lemma}
\begin{proof} The model $f:\X \rightarrow [-1,1]$ is lossless if and only if $\mutualInformationProbMeasure(F;Y) = \mutualInformationProbMeasure(Y;X)$, which is the case if and only if $\entropyProbMeasure(Y|F) = \entropyProbMeasure(Y|X)=0$, where we have used Eq.~(\ref{eq:chain_rule_it}) and the assumption that $\probMeasure$ is noiseless. Moreover, as in the proof of Lemma \ref{noiselessIFFZeroRiskLemma} we have
\begin{align*}
\entropyProbMeasure(Y|X) &= \sum_{s\in f(\XSupp) } \probMeasure(f(X)=s)\left( \sum_{y \in \Y}\entropyPhi\left(\probMeasure(Y=y|f(X)=s) \right)\right).
\end{align*}
Using the fact that $\entropyPhi(z)\geq 0$ on $[0,1]$ with equality only at $\{0,1\}$ we infer that $\entropyProbMeasure(Y|X)=0$ if and only if for each $y \in \Y$, $s \in f(\XSupp)$ we have  $\probMeasure(Y=y|f(X)=s) \in \{0,1\}$.

Now if $\risk_{\probMeasure}(g \circ f) =\risk_{\probMeasure}(\riskMinimiserProb)$ for some invertible transformation $g$ then 
\begin{align*}
\sum_{s\in f(\XSupp) } \probMeasure(f(X)=s) \cdot \probMeasure\left(Y \neq \sign( g(s))|f(X)=s\right)&=\probMeasure\left( \sign( g \circ f(X))\neq Y\right)\\&=\risk_{\probMeasure}(g \circ f) =\risk_{\probMeasure}(\riskMinimiserProb)=0.
\end{align*}
This implies that for each $s \in f(\XSupp)$ for $y= \sign( g(s))$, $\probMeasure(Y=y|f(X)=s)=1$ and for $y\neq \sign( g(s))$, $\probMeasure(Y=y|f(X)=s)=0$, so in general $\probMeasure(Y=y|f(X)=s) \in \{0,1\}$, $\entropyProbMeasure(Y|X)=0$ and $f$ is lossless.

Conversely, if $f$ is lossless then for each $s \in f(\XSupp)$ we can take
\begin{align*}
g(s) = (2\cdot \probMeasure(Y=1|f(X)=s) -1) \left(\frac{s+2}{3} \right),
\end{align*}
and extend $g$ on $[-1,1]\backslash f(\XSupp)$ arbitrarily to form a bijection. Since $f$ is lossless, for each $s \in f(\XSupp)$  and $y \in \Y$ we have $\probMeasure(Y=y|f(X)=s) \in \{0,1\}$. If for some $s \in f(\XSupp)$ we have $ \probMeasure(Y=1|f(X)=s)=1$ then $g(s)>0$ and so $\probMeasure\left(Y\neq \sign(g(s))|f(X)=s\right)=0$. Similarly, for $s \in f(\XSupp)$ with $ \probMeasure(Y=-1|f(X)=s)=1$ we have $g(s)<0$ and so again  $\probMeasure\left(Y\neq \sign(g(s))|f(X)=s\right)=0$. Hence, in general $\risk_{\probMeasure}\left( g\circ f\right) = 0 = \risk_{\probMeasure}\left(\riskMinimiserProb\right)$.

\end{proof}

Definitions \ref{def:max_compP} and \ref{def:losslessnessP} generalise Definitions \ref{def:max_comp} and \ref{def:losslessness}, respectively.

\begin{definition}\label{def:max_compP}
A model $f$ is a maximal compressor of a distribution $\probMeasure$ if and only if $\mutualInformationProbMeasure(F;X) = \mutualInformationProbMeasure(F;Y)$. 
\end{definition}

\begin{definition}\label{def:lossless_max_compP}
A model $f$ is a lossless maximal compressor (LMC) of a training dataset $\trainingSample$ if and only if it is lossless on $\trainingSample$ and a maximal compressor on $\trainingSample$. 
\end{definition}

Proposition \ref{thm:lmcP} generalises Proposition \ref{thm:lmc}.

\begin{proposition}
\label{thm:lmcP}
A model $f$ is an LMC of a finitely supported probability distribution  $\probMeasure$, if and only if it satisfies
\begin{equation*}
\mutualInformationProbMeasure(F;X) = \mutualInformationProbMeasure(F;Y) = \mutualInformationProbMeasure(Y;X).
\end{equation*}
\end{proposition}
\begin{proof}
Follows straightforwardly from Definition~\ref{def:losslessnessP} \& Definition~\ref{def:max_compP}.\end{proof}

Finally we shall prove Theorem \ref{thm:margin_max_lmcP} which generalises Theorem \ref{thm:margin_max_lmc}.

\begin{theorem}
 \label{thm:margin_max_lmcP} Suppose $\probMeasure$ is noiseless and finitely supported. A model $f$ is an LMC with respect to $\probMeasure$ if and only if there exists some invertible transformation $g:[-1,1]\rightarrow\R$ such that $g \circ f$ is a margin maximizer with respect to $\probMeasure$. 
\end{theorem}
\begin{proof} As we saw in the proof of Lemma \ref{noiselessIFFZeroRiskLemma}, the fact that $\probMeasure$ is noiseless implies that for each $x \in \XSupp$ and $y \in \Y$ we have $\probMeasure\left(Y=y|X=x\right) \in \{0,1\}$. We form partition partition $\XSupp = \X_+\cup \X_-$ so that for $x \in \X_+$, $\probMeasure\left(Y=1|X=x\right)=1$ and for $x \in \X_-$, $\probMeasure\left(Y=-1|X=x\right)=1$.

Now suppose that for some invertible transformation $g: [-1,1]\rightarrow \R$, $g \circ f$ is a margin maximizer with respect to $\probMeasure$. Hence, if $g(f(x))=1$ for $x \in \X_+$ and $g(f(x))=-1$ for $x \in \X_-$. This implies that $\risk_{\probMeasure}(g \circ f)= \risk_{\probMeasure}(\riskMinimiserProb)=0$. Thus, by Lemma \ref{losslessIFFARiskMinimiserUpToInvertibleTransformationLemma}, $f$ is lossless. Moreover, $g$ is invertible this is equivalent to $f(x)=g^{-1}(1)$ for $x \in \X_+$ and $f(x)=g^{-1}(-1)$ for $x \in \X_-$. Hence, $\probMeasure(f(X)=s|Y=y)=1$ when $s=g^{-1}(y)$ and $\probMeasure(f(X)=s|Y=y)=0$ otherwise, so in general $\entropyPhi(\probMeasure(f(X)=s|Y=y))=0$. Thus,
\begin{align*}
\entropyProbMeasure\left(F|Y\right) &= \sum_{y \in \Y}\probMeasure(Y=y) \sum_{ s\in f(\XSupp)} \entropyPhi(\probMeasure(f(X)=s|Y=y)) = 0 = \entropyProbMeasure(F|X),
\end{align*}
where the final inequality follows from the fact that $\probMeasure\left(F=f(x)|X=x\right)=1$. Hence, we have $\mutualInformationProbMeasure(F;Y)= \mutualInformationProbMeasure(F;X)$. It follows that $f$ is a maximal compressor and we have already shown that $f$ is lossless.

Conversely, let's suppose that $f$ is a lossless maximal compressor. Since $f$ is lossless we infer from Lemma \ref{losslessIFFARiskMinimiserUpToInvertibleTransformationLemma} that there is some transformation $\tilde{g}:[-1,1]\rightarrow \R$ such  $\risk_{\probMeasure}(\tilde{g} \circ f) =\risk_{\probMeasure}(\riskMinimiserProb)=0$, which in turn implies that if $x_+ \in \X_+$ and $x_- \in \X_-$ then $f(x_+)\neq f(x_-)$. Moreover, since $f$ is a maximal compressor we must have  $\mutualInformationProbMeasure(F;Y)= \mutualInformationProbMeasure(F;X)$ which implies 
\begin{align*}
\sum_{y \in \Y}\probMeasure(Y=y) \sum_{ s\in f(\XSupp)} \entropyPhi(\probMeasure(f(X)=s|Y=y))= \entropyProbMeasure\left(F|Y\right) = \entropyProbMeasure(F|X)=0.
\end{align*}
Thus, for each $s \in f(\XSupp)$ and $y \in \Y$ we have $\probMeasure(f(X)=s|Y=y) \in \{0,1\}$, where once again we use the fact that $\entropyPhi$ is non-negative with zero attained at $\{0,1\}$. Hence, there exists some $s_+ \in [-1,1]$ such that for all $x_+ \in \X_+$, $f(x_+)=s_+$ and some $s_- \in [-1,1]$ with $s_-\neq s_+$ such that for all $x_- \in \X_-$, $f(x_-)=s_-$. Thus, if we choose $g:[-1,1]\rightarrow \R$ to be any invertible map with $g(s_+)=1$ and $g(s_-)=-1$ we see that $g\circ f$ is a margin maximiser. This completes the proof of the theorem. 
\end{proof}

Finally, we prove Lemma \ref{losslessIFFCapturesCCProbLemmaP} which generalises Lemma \ref{losslessIFFCapturesCCProbLemma}.

\begin{lemma} \label{losslessIFFCapturesCCProbLemmaP} Suppose that $X$ and $Y$ are discrete random variables taking values in $\X$ and $\Y=\{-1,+1\}$, respectively. Suppose that $f:\X \rightarrow \Z$ and let $F=f(X)$. Then $I(X;Y)=I(F;Y)$ if and only if the map $x\mapsto \Prob(Y=1|X=x)$ is constant on all sets of the form $f^{-1}(z)\subseteq \X$ for some $z \in \Z$.
\end{lemma}
\begin{proof}
The proof uses the entropy functional $\entropyFunction:[0,1] \rightarrow \R$ by $\entropyFunction(p) = -p\cdot \log(p)-(1-p)\cdot \log(1-p)$. Note that $\entropyFunction$ is strictly concave. Now observe that
\begin{align*}
H(Y|X) &= \sum_{x \in \X}\Prob(X=x) \cdot \entropyFunction\left(\Prob(Y=1|X=x)\right)\\
&= \sum_{z \in \Z} \Prob(F=z) \cdot \sum_{x \in \X}\Prob(X=x|F=z) \cdot \entropyFunction\left(\Prob(Y=1|X=x)\right),
\end{align*}
where we have used the fact that $\Prob(X=x|F=z)= \Prob(X=x)/\Prob(F=z)$ if $z=f(x)$ and $\Prob(X=x|F=z)= 0$ otherwise. We also have
\begin{align*}
H(Y|F) &= \sum_{z \in \Z}\Prob(F=z) \cdot \entropyFunction\left(\Prob(Y=1|F=z)\right).
\end{align*}
By the strict concavity of $\entropyFunction$ for each $z \in \Z$ we have
\begin{align*}
&\sum_{x \in \X}\Prob(X=x|F=z) \cdot \entropyFunction\left(\Prob(Y=1|X=x)\right) \\ &\leq  \entropyFunction\left(\sum_{x \in \X}\Prob(X=x|F=z) \cdot \Prob(Y=1|X=x)\right)  = \entropyFunction\left(\Prob(Y=1|F=z)\right),
\end{align*}
with equality if and only if $\Prob(Y=1|X=x)$ is constant for all $x \in \X$ with $\Prob(X=x|F=z)>0$, so constant for all $x \in f^{-1}(z)$. Hence, we have $H(Y|X) \leq H(Y|F)$ with equality if and only if for each $z \in \Z$, $x\mapsto \Prob(Y=1|X=x)$ is constant on $f^{-1}(z)\subseteq \X$. To conclude note that $I(Y;X)=H(Y)-H(Y|X)$ and $I(Y;F)=H(Y)-H(Y|F)$, so $I(Y;X)\geq I(Y;F)$ with equality if and only $H(Y|X)=H(Y|F)$ which holds if and only if for each $z \in \Z$, $x\mapsto \Prob(Y=1|X=x)$ is constant on $f^{-1}(z)\subseteq \X$. 
\end{proof}

\newpage
\subsection*{B. Details of datasets used}
\subsubsection*{B1. Artificial Data}
\label{sssec:artificial_data}
The artificial dataset was generated by \emph{scikit-learn}'s \emph{make\_classification()} function. We generated $2000$ examples, each consisting of $20$ features, only $2$ of which were relevant for predicting the class. The examples belonged to $2$ different clusters for each of the $2$ classes, each cluster's points normally distributed (with unit standard deviation) about vertices of a $2$-sided hypercube. Some label noise was added by randomly flipping the label of each point with probability $0.01$. For more information see the function's documentation at \href{http://scikit-learn.org/stable/modules/generated/sklearn.datasets.make\_classification.html}{http://scikit-learn.org/stable/modules/generated/sklearn.datasets.make\_classification.html}.

\subsubsection*{B2. UCI Datasets}
\label{ssec:datasets}
Table \ref{tab:datasets} shows the characteristics of the real-world datasets used in our experiments. The original datasets are all from the \emph{UCI repository}. Examples with missing values were discarded. The multiclass datasets were converted to balanced binary ones by setting the minority class as the `positive' one and uniformly sampling examples from the remaining classes to form the `negative' class. A link to the final datasets will be provided along with the code used to generate the results.
\begin{table}[ht]
\footnotesize
\center
\caption{Characteristics of the UCI datasets used in our experiments; number of instances used, number of features and number of classes before binarization.}\label{tab:datasets}
    \begin{tabular}{|c|c|c|c|c|}
    \hline
 	\multirow{2}{*}{Dataset} & \# & \# & \# \\
 	 & Instances & Features & Classes \\ \hline\hline
    \emph{parkinsons} & $96$ & $22$ & $2$  \\ \hline  
    \emph{sonar} & $194$ & $60$ & $2$ \\ \hline
    \emph{heart} & $240$ & $13$ & $2$ \\ \hline 
    \emph{ionosphere} & $252$ & $34$ & $2$ \\ \hline    
    \emph{semeion} & $322$ & $256$ & $10$ \\ \hline
    \emph{congress} & $336$ & $16$ & $2$  \\ \hline     
    \emph{wdbc} & $424$  & $31$ & $2$ \\ \hline
    \emph{credit} & $600$ & $24$ & $2$ \\ \hline     
    \emph{landsat} & $1252$ & $36$ & $6$ \\ \hline  
    \emph{splice} & $1524$ & $60$ & $3$ \\ \hline    
    \emph{musk2} & $2034$ & $166$ & $2$ \\ \hline 
    \emph{krvskp} & $3054$ & $36$ & $2$ \\ \hline
    \emph{waveform} & $3306$ & $40$ & $3$ \\ \hline   
    \emph{mushroom} & $7832$ & $21$ & $2$ \\ \hline
    \end{tabular}
\end{table}

\subsection{C. Additional experimental results}
This section contains additional experimental results that further showcase the consistency of the trajectory of the boosting ensemble on the information plane towards the lossless maximally compressing (LMC) model point across different datasets and hyperparameter settings.


Figure \ref{fig:trajectories_mushroom_tree_depth} shows only the average trajectories of the boosting ensemble trained on the \emph{mushroom} dataset using base learners of varying capacity, demonstrating that the trajectories are again qualitatively similar. Naturally, the lower the capacity of an individual learner, the more boosting rounds are required to reach the LMC point.

Figure \ref{fig:trajectories_artificial_hyperparam} shows the same results as Figures~\ref{fig:trajectories_real}--\ref{fig:trajectories_real4} on artificial data generated as described in Section B1 of this Supplementary Material. It also demonstrates that changing the loss function, using subsampling of the examples for the purposes of the updates or shrinkage of the updates does not change the trajectory qualitatively.

\begin{figure}
\centering
\subfigure{\includegraphics[width=0.49\textwidth]{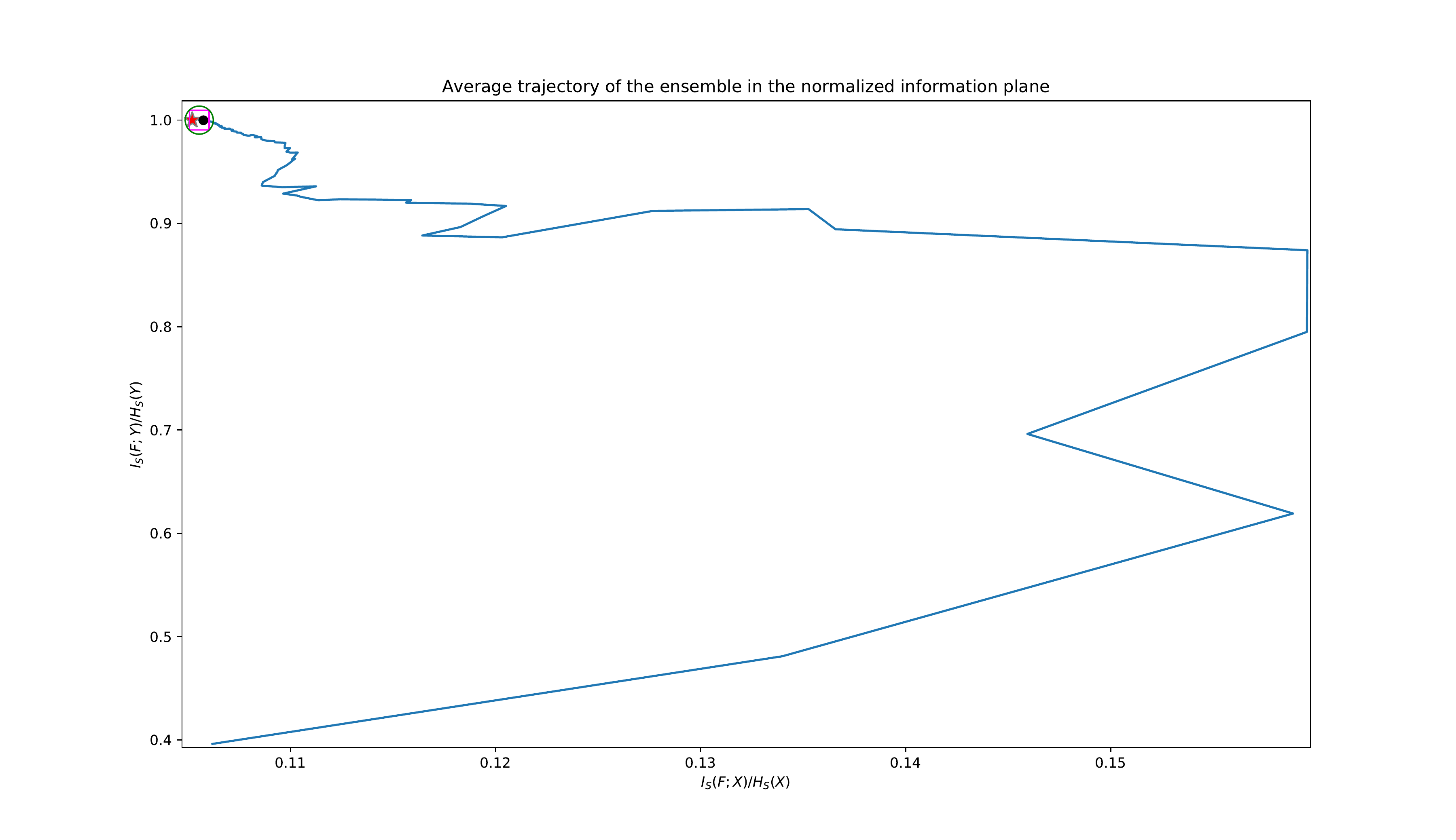}}
\subfigure{\includegraphics[width=0.49\textwidth]{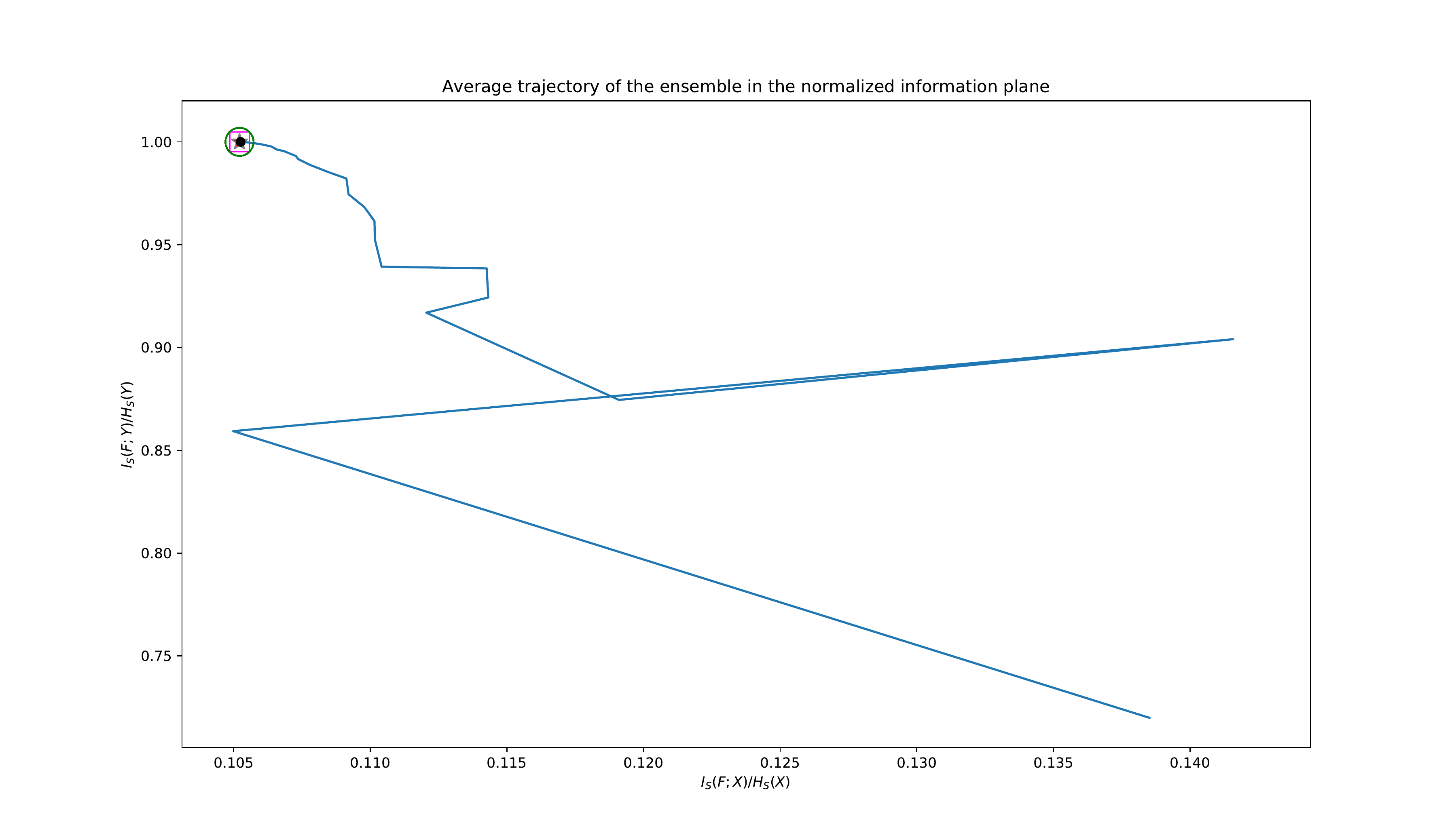}}
\subfigure{\includegraphics[width=0.49\textwidth]{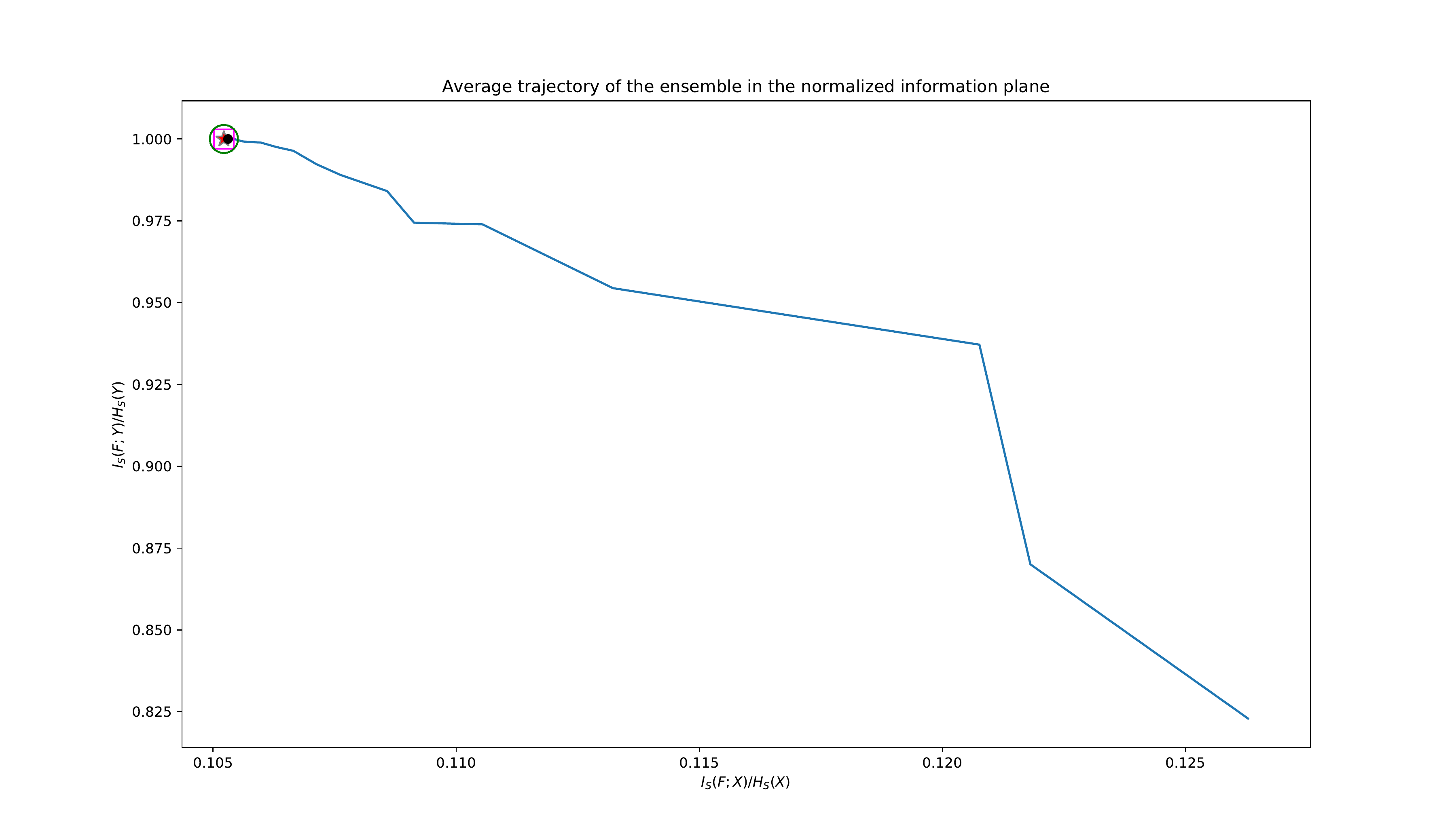}}
\subfigure{\includegraphics[width=0.49\textwidth]{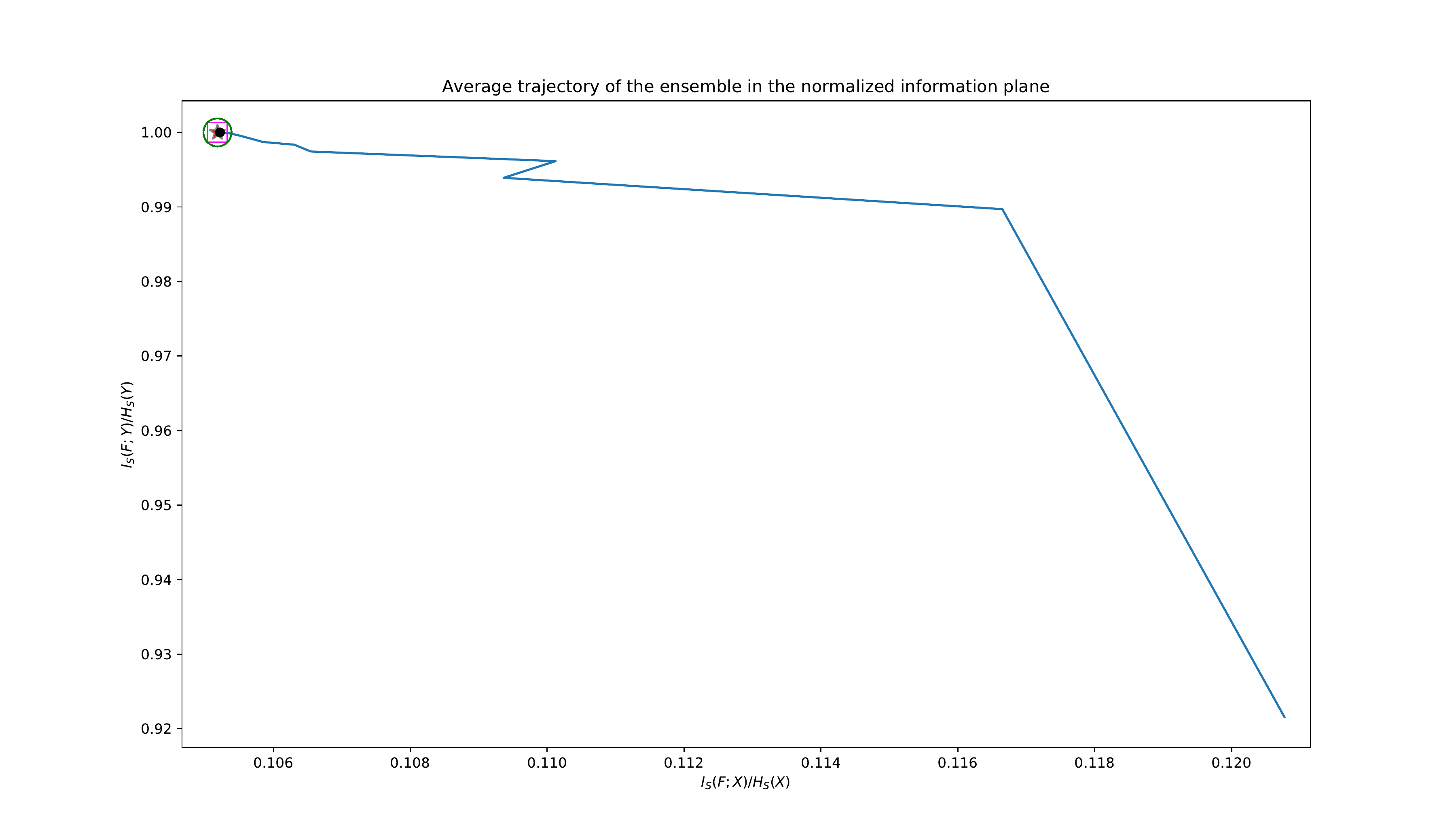}}
\subfigure{\includegraphics[width=0.49\textwidth]{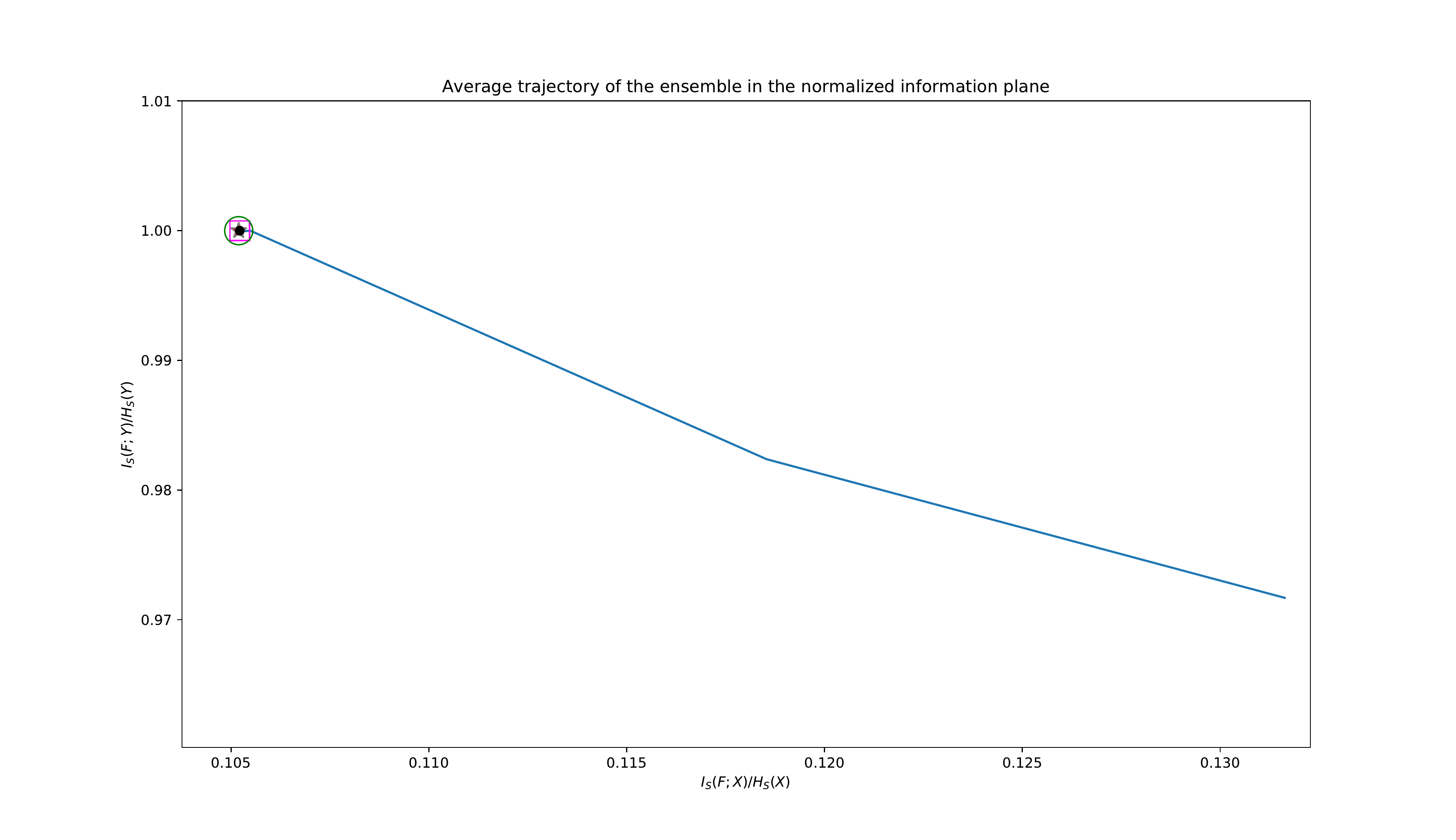}}
\subfigure{\includegraphics[width=0.49\textwidth]{./Figures/mushroom_average_traj}}
\caption{Effect of the capacity of the weak learner (max. tree depth $d$) on the average trajectory of the boosting ensemble on the normalized information plane as the rounds of boosting progress. All plots are for the \emph{mushroom} dataset. The same points of interest as in the previous figures are shown.  TOP LEFT: $d=1$; TOP RIGHT: $d=2$; MID LEFT: $d=3$; MID RIGHT: $d=4$; BOTTOM LEFT: $d=5$; BOTTOM RIGHT: $d=6$. }\label{fig:trajectories_mushroom_tree_depth}
\end{figure}

\begin{figure}
\centering
\subfigure{\includegraphics[width=0.49\textwidth]{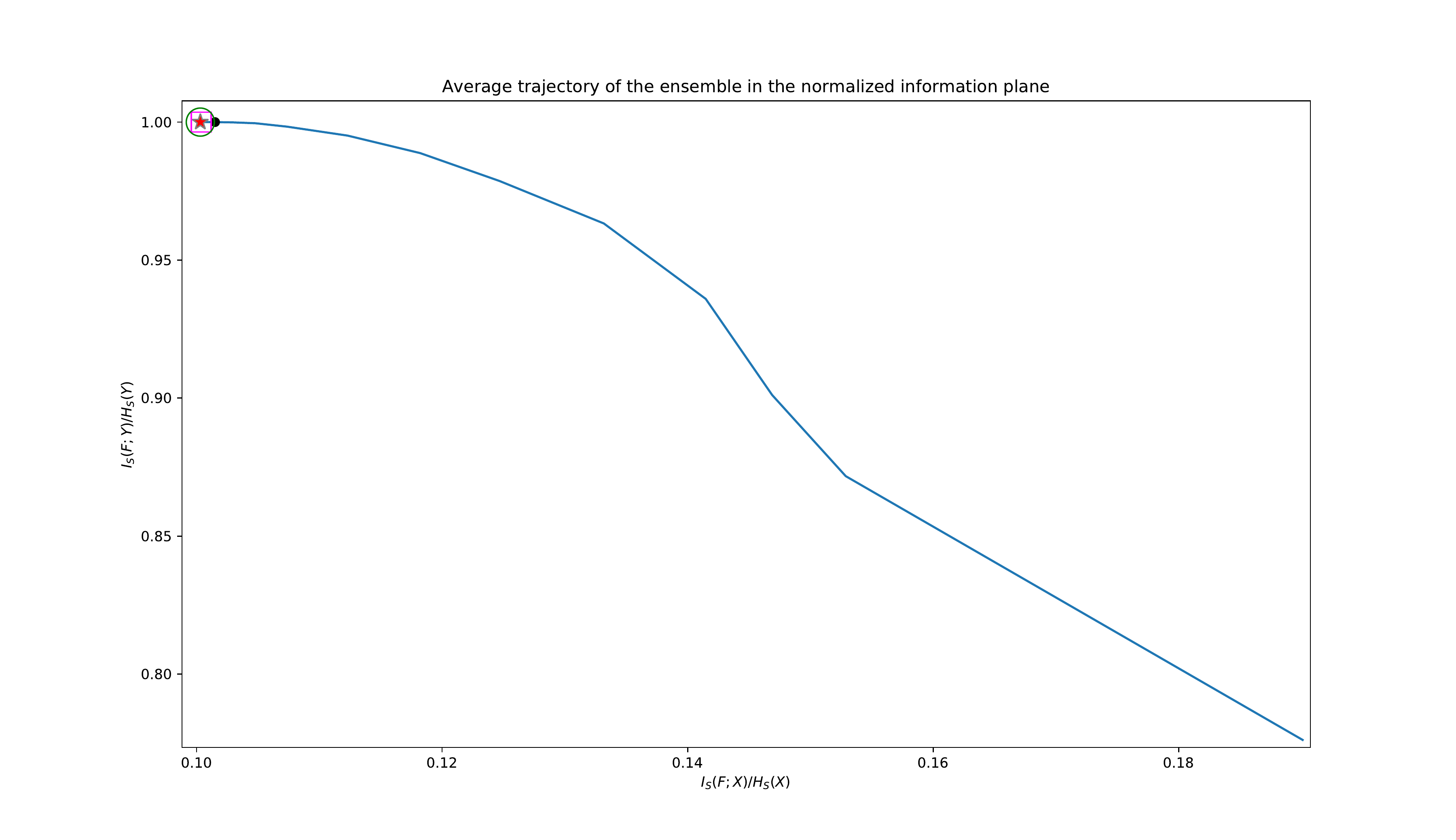}}
\subfigure{\includegraphics[width=0.49\textwidth]{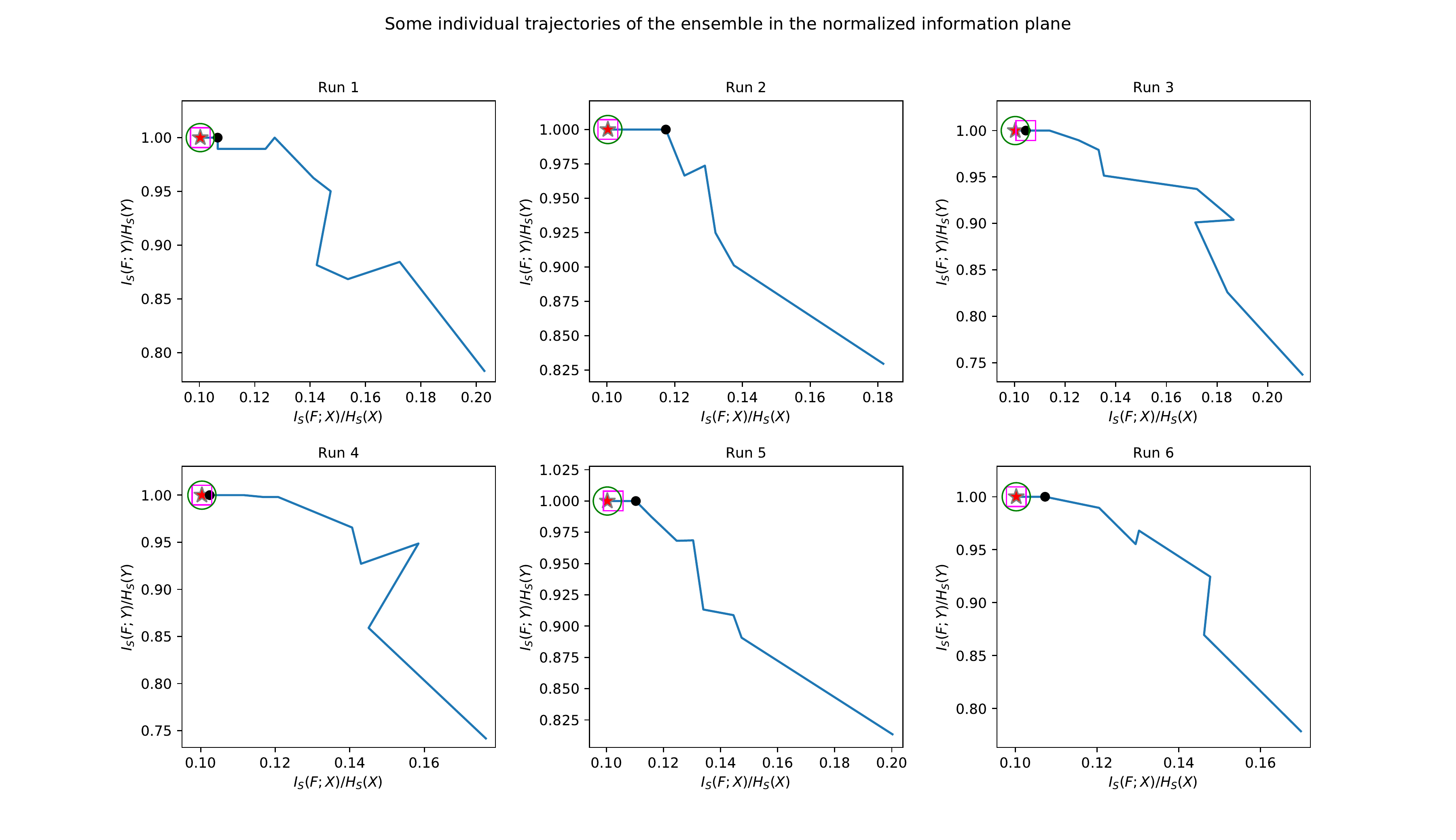}}
\subfigure{\includegraphics[width=0.49\textwidth]{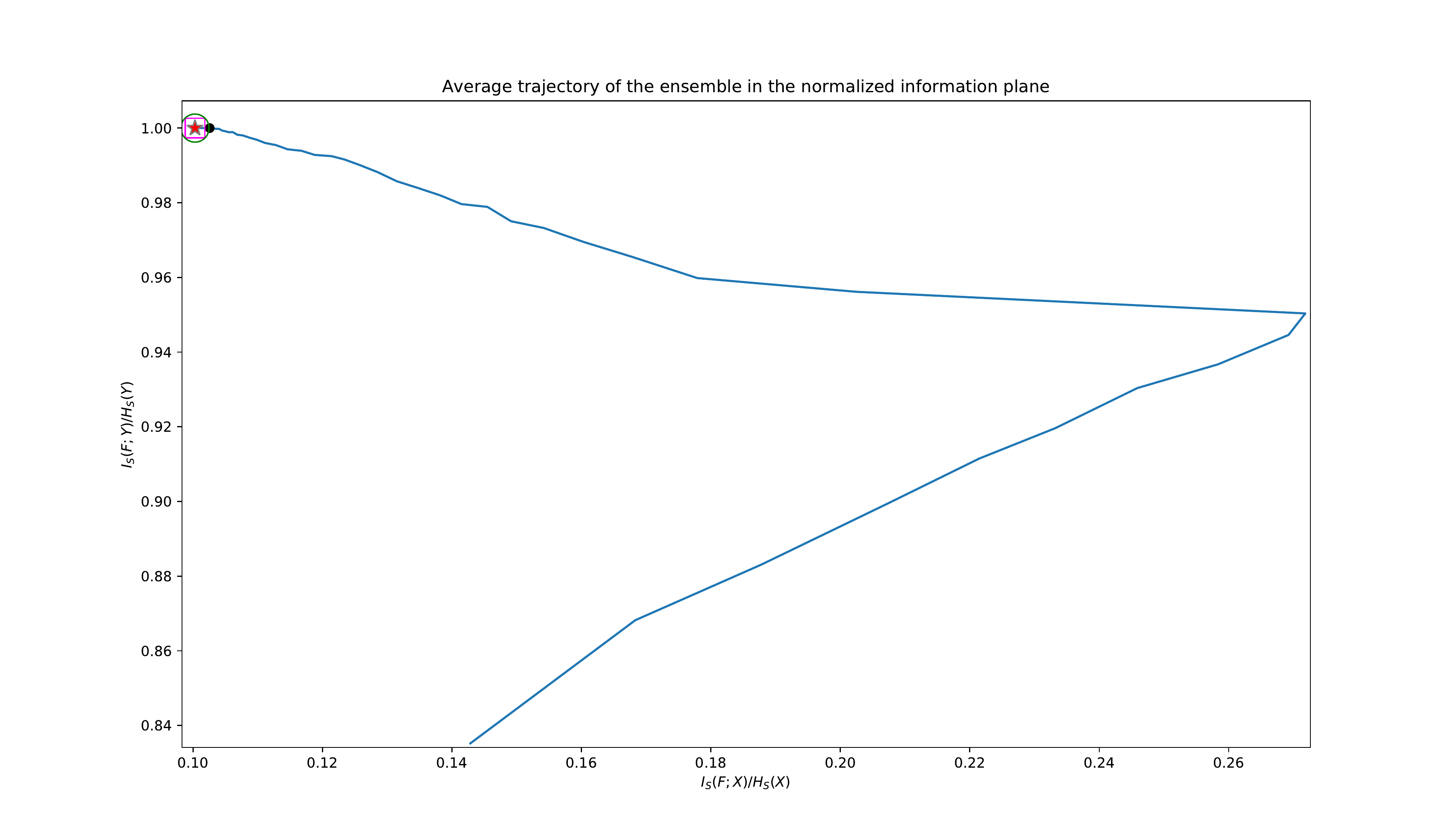}}
\subfigure{\includegraphics[width=0.49\textwidth]{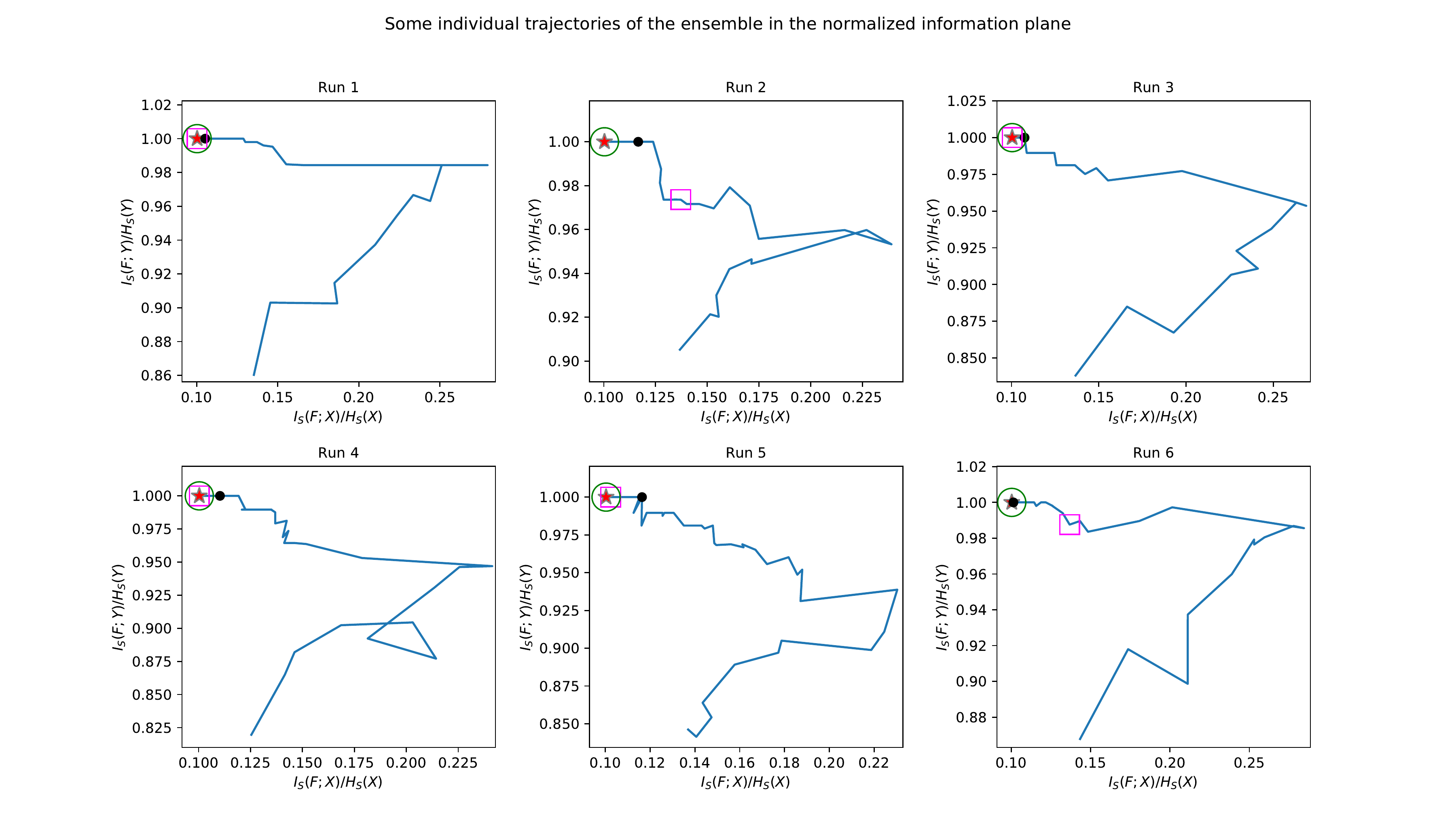}}
\subfigure{\includegraphics[width=0.49\textwidth]{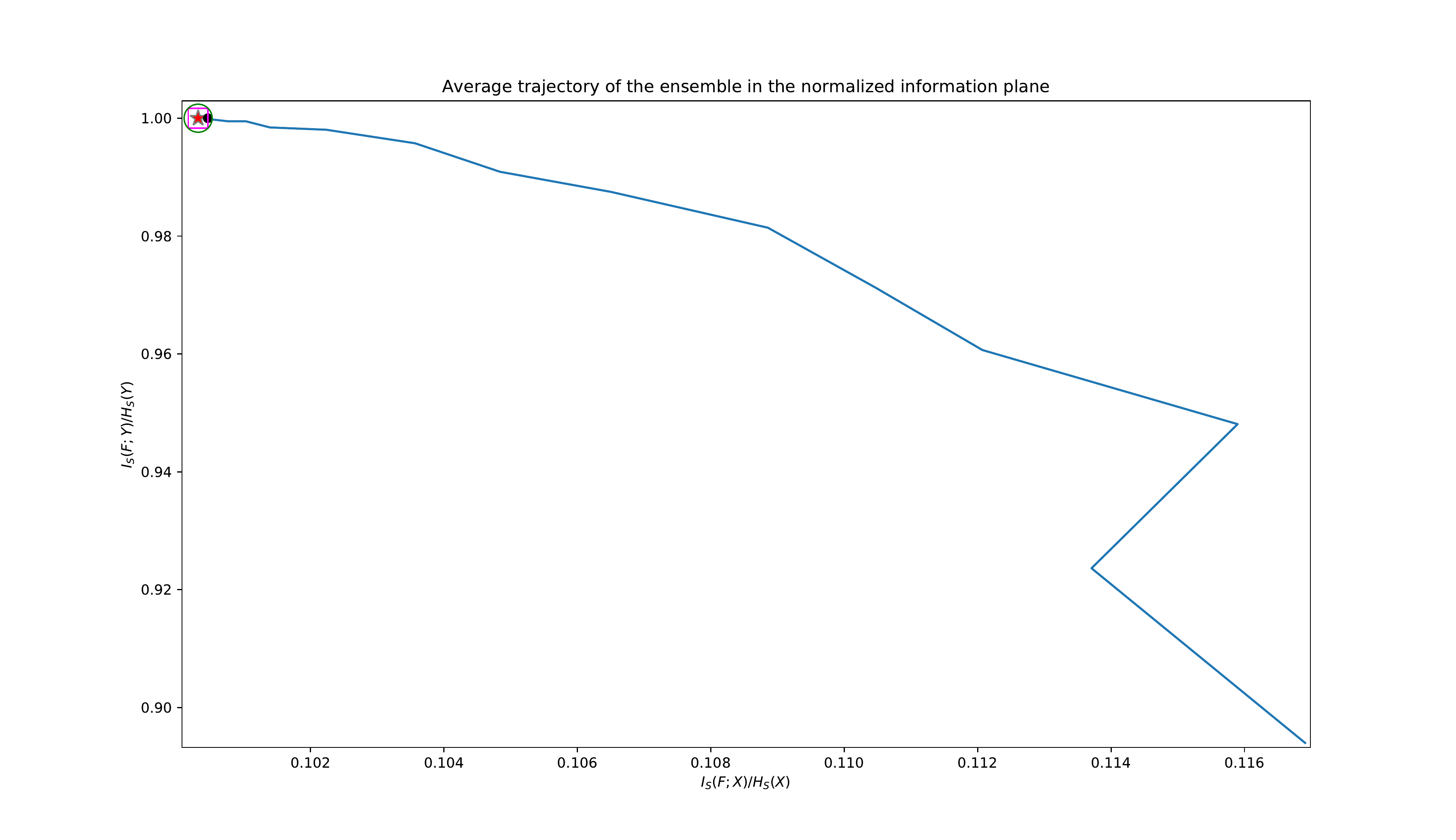}}
\subfigure{\includegraphics[width=0.49\textwidth]{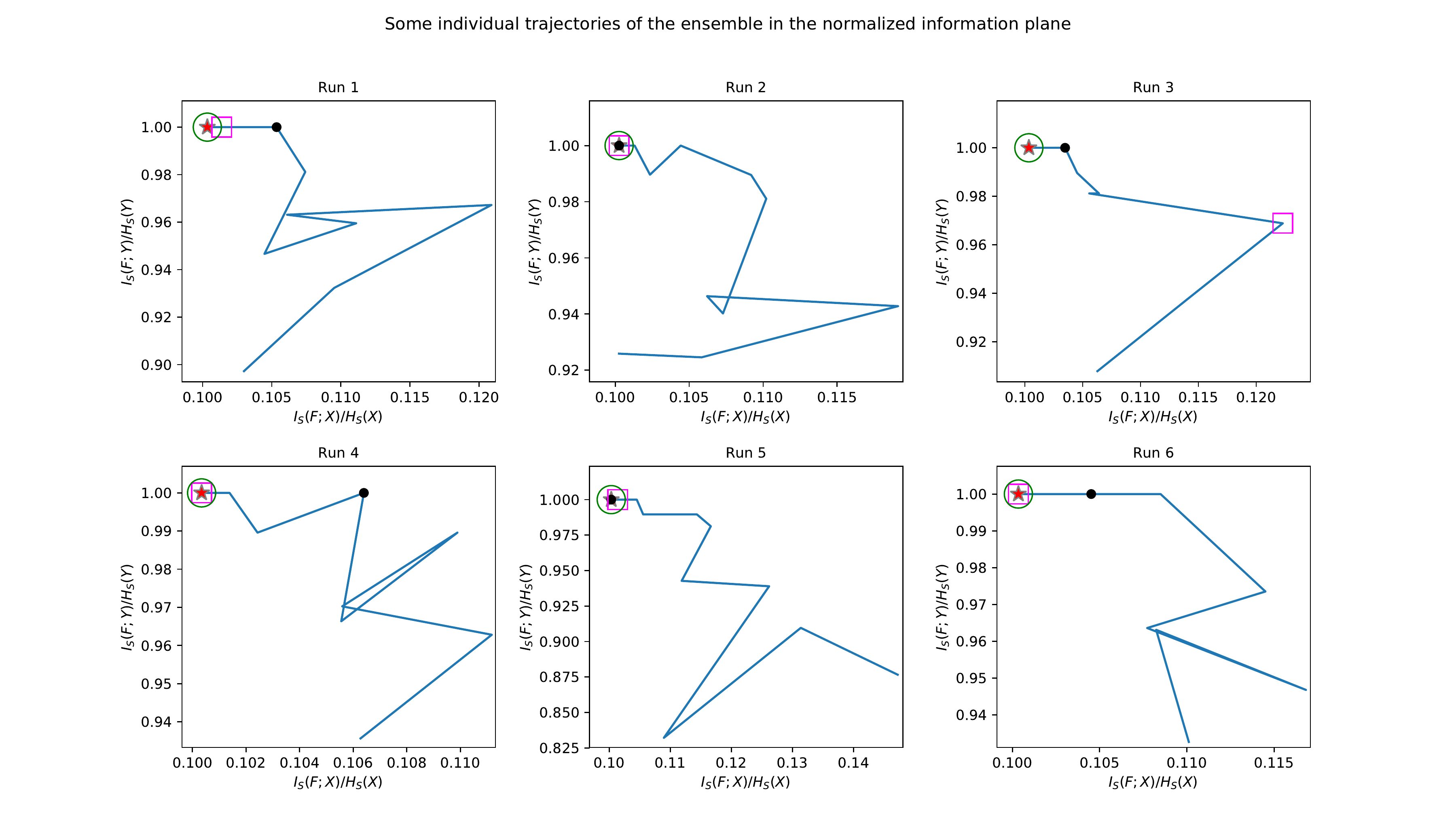}}
\subfigure{\includegraphics[width=0.49\textwidth]{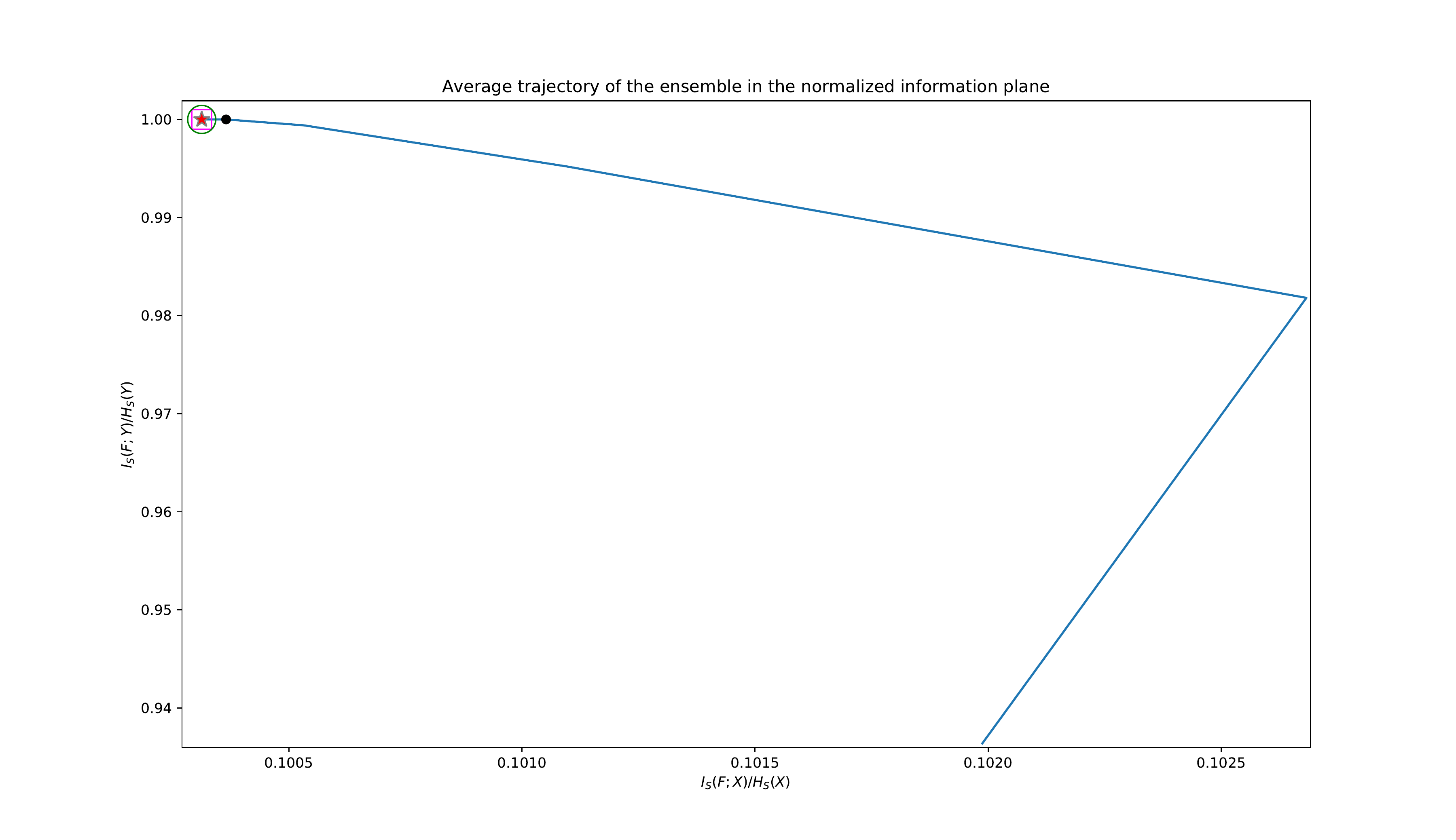}}
\subfigure{\includegraphics[width=0.49\textwidth]{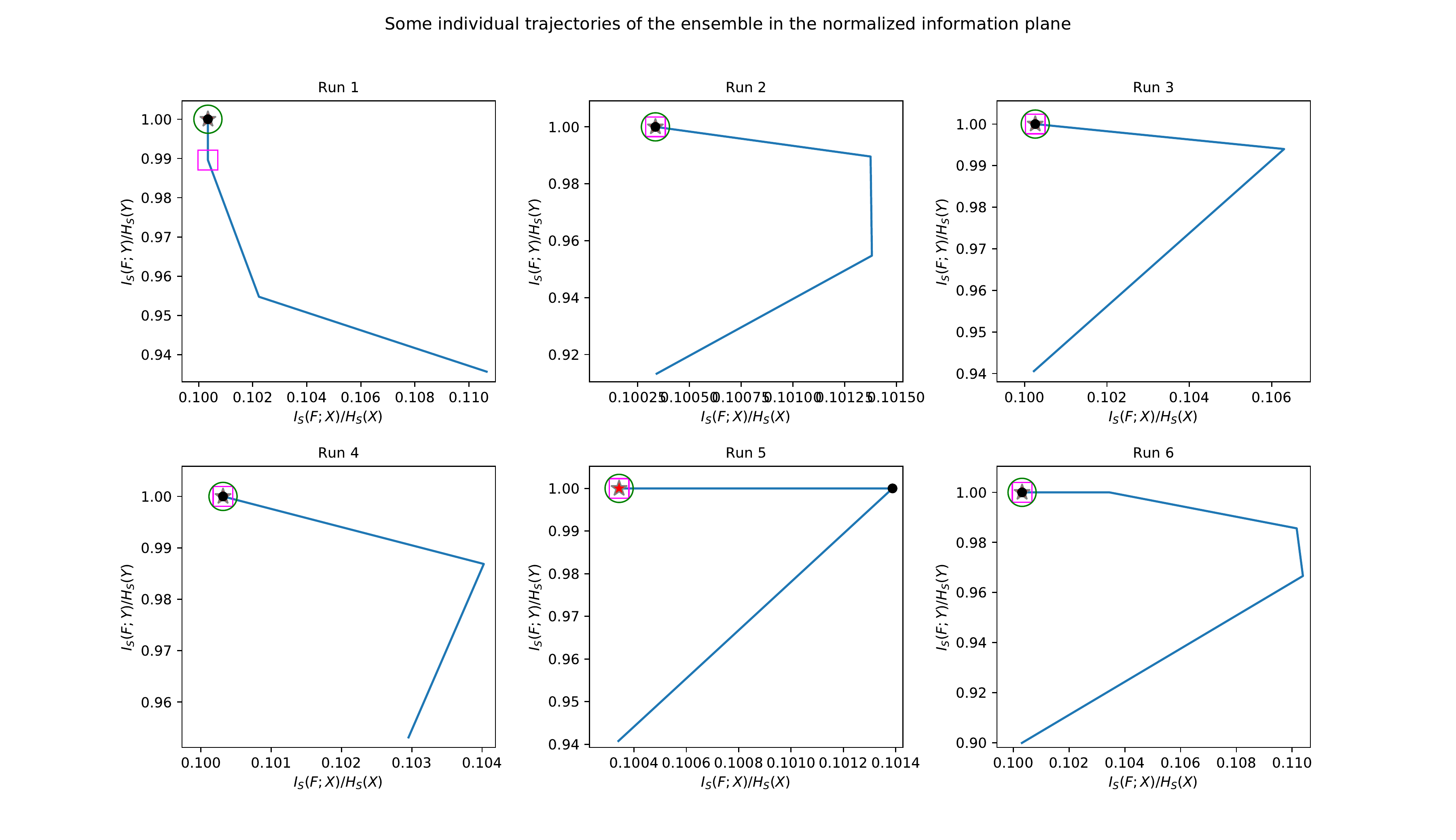}} 
\caption{Effect of hyperparameters on the trajectory of the boosting ensemble on the normalized information plane as the rounds of boosting progress. All results are on artificial data generated as described in Section B1 of this Supplementary Material. All points of interest from the previous figures are shown. ROW 1: Exponential loss, no shrinkage, no subsampling; ROW 2: Exponential loss, shrinkage with $\lambda=0.1$, no subsampling; ROW 3: Exponential loss, no shrinkage, subsampling set to $0.8$; ROW 4: Binomial deviance loss, no shrinkage, no subsampling; LEFT: Average trajectory across 100 runs; RIGHT: Some random individual trajectories.}\label{fig:trajectories_artificial_hyperparam}
\end{figure}

\end{document}